 \documentclass[final,3p,times]{elsarticle}

\usepackage{booktabs}
\usepackage[dvipsnames]{xcolor}
\usepackage{amssymb}
\usepackage{array}
\usepackage{placeins}
\usepackage{float}
\usepackage{stfloats}
\usepackage{tabularx}

\usepackage{soul}

\definecolor{mygreen}{RGB}{0,120,0}
\usepackage{amsmath}
\definecolor{mypurple}{RGB}{128,0,128}
\usepackage{microtype}  
\usepackage[linesnumbered,ruled,vlined]{algorithm2e}
\renewcommand{\hl}[1]{#1}
 \usepackage{amsthm}
    \newtheorem{assumption}{Assumption}
    
    \newtheorem{theorem}{Theorem}
    
    \newtheorem{definition}{Definition}
    \newtheorem{invariant}{Invariant}
    
\usepackage{listings}
\usepackage{multirow}
\usepackage{tikz}
\usepackage{url}
\definecolor{lightyellow}{RGB}{255,250,205}
\sethlcolor{lightyellow}
\usetikzlibrary{arrows.meta, positioning, shapes.symbols,calc}

\usepackage[framemethod=TikZ]{mdframed}
\usepackage{graphicx}
\usepackage{lineno}
\usepackage{hyperref}
\lstdefinelanguage{MontiCore}{
  keywords={grammar, extends, symbol, abstract, enum, interface, implements},
  morekeywords=[2]{Name},
  keywordstyle=\color{blue}\bfseries,
  keywordstyle=[2]\color{teal},
  commentstyle=\color{ForestGreen}, 
  morecomment=[l]{//},
  morecomment=[s]{/*}{*/},
  stringstyle=\color{orange},
  basicstyle=\scriptsize\ttfamily,
  breaklines=true,
  tabsize=2,
  showstringspaces=false,
  numbers=left,
numbersep=6pt,
xleftmargin=2em,
framexleftmargin=2em
}
\definecolor{atcolor}{RGB}{128,0,128} 
\newcommand{\repofile}[1]{\path{#1}}

\lstdefinelanguage{XML}{
  morestring=[b]",
  morestring=[s]{>}{<},
  morecomment=[s]{<?}{?>},
  stringstyle=\color{black},
  identifierstyle=\color{blue},
  keywordstyle=\color{teal},
   numbers=left,
numbersep=6pt,
xleftmargin=2em,
framexleftmargin=2em,
  morekeywords={name,ID,BTCPP_format,main_tree_to_execute,beam,location} 
}
\journal{Journal of Systems and Software}

\begin{document}

\begin{frontmatter}



\title{APTree: Extended Behavior Trees for Adaptive and Scalable Automated Robotic Task Planning in Construction\tnoteref{t1}}

\tnotetext[t1]{Accepted for publication in The Journal of Systems and Software.
\copyright\ 2026. This manuscript version is made available under the CC-BY-NC-ND 4.0 license \url{https://creativecommons.org/licenses/by-nc-nd/4.0/}}


\author[1,2]{Shermin Sherkat\corref{cor1}} 
\ead{shermin.sherkat@icd.uni-stuttgart.de}
\author[1]{Thomas Wortmann}
\author[2]{Andreas Wortmann}
\cortext[cor1]{Corresponding author.}

\affiliation[1]{organization={Institute for Computational Design and Construction, Department for Computing in Architecture (ICD/CA), University of Stuttgart},
    addressline={Keplerstraße 11}, 
    city={Stuttgart},
    postcode={70174}, 
    country={Germany}}

   \affiliation[2]{organization={Institute for Control Engineering of Machine Tools and Manufacturing Units (ISW), University of Stuttgart},
    addressline={Seidenstraße 36}, 
    city={Stuttgart},
    postcode={70174}, 
    country={Germany}}

\begin{abstract}
Task planning for robots in off-site construction prefabrication is primarily manual, leading to planning inefficiencies and reluctance to adopt robotic solutions. Existing approaches provide limited support for scalable, modular automated task planning and localized runtime replanning. This paper proposes APTree, a model-driven framework that integrates automated task planning with modular task execution using behavior trees as the execution model. APTree repeatedly applies and coordinates heterogeneous symbolic planners to automatically generate and compose modular task plans at runtime. To support this approach, we develop a domain-specific language for the formal specification and validation of task plans, together with a compatible execution engine. We evaluate APTree on the real-world livMatS Biomimetic Shell pavilion, where it models the construction robotic fabrication domain, detects injected specification and planning errors at design time and runtime, and automatically generates large-scale task plans by coordinating three symbolic planners across numerous planning modules. A machine-checked formal analysis further shows that composed execution and local replanning are safe under explicitly stated assumptions, including planner soundness and non-interference of generated plans, providing practitioners with explicit conditions for dependable plan composition. APTree thus enables construction practitioners to leverage automated task planning without requiring detailed knowledge of symbolic planning techniques or APTree's internal implementation.
\end{abstract}



\begin{keyword}
domain-specific language~\sep behavior tree~\sep AI planner~\sep task planning~\sep model-driven engineering~\sep symbolic planner~\sep automated planning~\sep robot programming~\sep software language engineering



\end{keyword}

\end{frontmatter}



\section{Introduction}
\label{intro}

The construction industry faces substantial automation challenges that require robotic solutions~\cite{davila_delgado_robotics_2019}, particularly in off-site assembly and prefabrication, where components are fabricated in controlled, factory-like environments for on-site assembly~\cite{chea_integrated_2020}; we refer to this domain as construction robotic fabrication (CRF). Robotic solutions depend on both task planning and task execution~\cite{ghallab_automated_2016}. Task planning generates plans that specify the tasks and their order to achieve a goal~\cite{Rovida17}. Task execution realizes these task plans through execution models that control runtime coordination, including timing and conditional switching~\cite{Colledanchise18}. CRF problems are inherently complex, involving hundreds of heterogeneous tasks, objects, resources, and constraints~\cite{Sakaida06}, making monolithic task planning approaches insufficient and driving the need for modular task plans composed of independently developed modules, potentially generated by different planning methods~\cite{iovino_survey_2022}. Currently, task planning for CRF is typically manual~\cite{sherkat2025}, contributing to planning inefficiency and reluctance to adopt robotic solutions~\cite{davila_delgado_robotics_2019}. Achieving scalable automated planning and execution of modular task plans requires addressing three key challenges (C):

(C1) hidden data dependencies arising from composing independently developed execution models without explicit interfaces, leading to runtime crashes due to data type errors or unpredictable behavior;

(C2) limited support for localized runtime replanning, where existing execution models typically require restructuring or regenerating the entire model rather than replacing only the affected task plan module, making execution structures rigid, limiting runtime changes, and preventing dynamic adaptation to changes in the robot's environment;

and (C3) limited scalability of automated task plan generation, where existing task planning approaches often struggle as planning problems increase in size and complexity, resulting in brittle execution models of limited size.

To address these challenges, this paper proposes APTree, a model-driven approach for scalable and modular automated task planning and execution based on behavior trees (BTs). BTs are tree-like execution models that structure tasks as nodes and control the switching between them to describe the behavior of an autonomous agent, such as a robot~\cite{Colledanchise18}. Among existing execution modeling languages, BTs are selected for their modular structure, in contrast to modeling languages such as finite state machines (FSMs)~\cite{Matteo24}. A BT consists of two main node types: flow nodes, which have one or more child nodes and control the order and conditions under which they are executed, and action nodes, which perform tasks~\cite{iovino_survey_2022}.

Specifically, we introduce a novel method for modeling task plans (tasks and their ordering) together with the symbolic planners that generate them (i.e., planners that derive task sequences through logical reasoning). These task plans are then coordinated and composed by execution models (BTs in this work), with a domain-specific language (DSL) serving as their common interface. Rather than coupling a specific task planner to a fixed execution model, our approach provides a planner-independent abstraction that enables symbolic planners to be repeatedly applied and coordinated at different levels of granularity to generate task plan modules (C3). These modules are validated for consistency, transformed into the APTree DSL, and composed into a unified execution model (C1) that supports scalable, localized runtime replanning (C2).

To realize this approach, we leverage model-driven engineering (MDE), which is a software development approach that focuses on creating and exploiting domain models and metamodels (e.g., BTs)~\cite{cheng16}. Metamodels define the rules, concepts, and relationships that constitute a model~\cite{Mustafa2023}. APTree extends the original BT metamodel with packages for domain types, graph-based plan representation, and planning services. Notably, the concepts captured by these packages are not inherently tied to BTs: the domain types and planning services contain no BT-specific constructs, and the graph-based plan representation is separated from the execution semantics; only their composition into the executable model and the execution engine are BT-specific.  Based on the extended metamodel, a text-based DSL enables the specification and validation of APTree models, and an execution engine loads and executes these models (see Figure~\ref{fig:intro}). Using MDE and software language engineering (SLE), we reduce the complexity of large CRF task planning problems, allowing domain users to focus on the problem rather than on the low-level implementation details of individual task planning methods and BT execution.
Consequently, this paper makes the following contributions:
\begin{enumerate}

\item A metamodel package for representing domain-specific types, predicates, and tasks that automatically generates a complementary metamodel for design-time validation of task interfaces and runtime validation of task preconditions and effects in execution models. The approach is realized by extending the BT metamodel to describe BT action nodes using the proposed task metamodel (C1, shown in green in Figure~\ref{fig:intro}).

\item A metamodel package for representing task plans as directed graphs with explicit temporal relations and runtime configurable execution logic, enabling planner-independent task execution and localized runtime replanning. The approach is realized by extending the BT metamodel to enable a BT flow node to describe a task plan (C2, shown in blue in Figure~\ref{fig:intro}).

\item A metamodel package for representing symbolic planners and their configurations that enables multiple symbolic planners operating at different abstraction levels to be explicitly specified, configured, and invoked within a unified task execution model. The approach is realized by extending the BT metamodel to enable a BT service node to describe a symbolic planner (C3, shown in orange in Figure~\ref{fig:intro}).

\item A text-based DSL built upon the extended BT metamodel packages, ensuring input consistency and validation of planner-generated execution models.

\item A BT-based execution engine compatible with the DSL, supporting model loading, symbolic planner invocation, and modular task planning.

\end{enumerate}

 We use MontiCore (a software language workbench for developing and processing modular and text-based DSLs~\cite{Monticore18}) to define and support APTree's DSL (Figure~\ref{fig:intro}). The MontiCore language workbench enables user input validation, prevents mismatches and errors during planning, and supports automatic code generation for the execution engine. Additionally, MontiCore generates an abstract syntax tree (AST) from the user input model, which represents the hierarchical structure of the source code. This AST is then used to enable context-aware validation of sub-execution models~\cite{Monticore18} modeled by the users or generated by symbolic planners before they are injected into the overall execution model, ensuring type correctness and structural consistency.

The methodology of this research is based on the design science research paradigm (DSR)~\cite{hevner2010design}, focusing on the development and evaluation of artifacts to solve complex challenges such as automation in CRF task planning. Key artifacts include the BT DSL package, the four DSL packages extending it, and the execution engine. 

The APTree method was evaluated using a case study on real-world CRF task planning of the livMatS Biomimetic Shell pavilion~\cite{Wagner20}, whose shell segments were robotically fabricated off-site prior to on-site assembly. The livMatS pavilion is a research demonstrator constructed at the FIT Freiburg Center, Germany, and features a digitally designed segmented shell structure~\cite{Wagner20}. The APTree DSL enabled modeling and validating the livMatS construction robotic fabrication domain, and the APTree execution engine automatically generated the task plan for assembling the segments using three different symbolic automated task planning tools: ENHSP~\cite{scala_subgoaling_2020}, Metric-FF~\cite{hoffmann_metric-ff_2003}, and LAMA-First~\cite{richter_lama_2010}. 

\begin{figure}[t]
    \centering
    \includegraphics[scale=0.7]{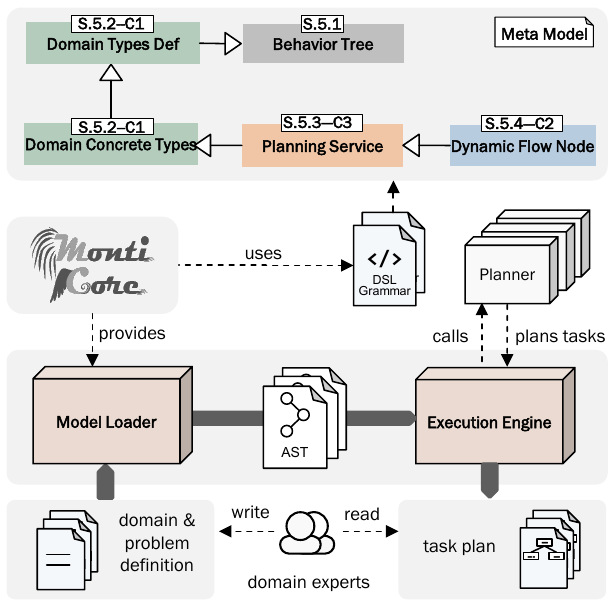}
\caption{Overview of the APTree method discussed in the paper. The APTree method extends the existing BT metamodel with four packages, each marked by its corresponding section (S) number in the paper and the challenge it addresses (C). }
    \label{fig:intro}
\end{figure}
In the remainder of this paper, Section \ref{backg} provides background information on the terminology used. Section \ref{relatedworks} reviews the related work and identifies the challenges, while Section \ref{reqs} presents the research questions and evaluation metrics to evaluate APTree based on the challenges. Section \ref{packages} describes the APTree DSL packages, and Section \ref{sec:casestudy} presents our case study and how the user can use APTree. Finally, Section~\ref{eval} presents the evaluation of APTree based on our case study results. Section~\ref{disc} discusses the findings, and Section~\ref{conc} concludes the paper.  \ref{APPA} and \ref{APPformal} contain supplementary evaluation material, and~\ref{APPB} provides information to facilitate replication.

\section{Background and Terminology}
\label{backg}

This section introduces concepts that are essential for understanding this paper. It begins with BTs, illustrated with a running example, and then provides an overview of MDE and key terminology related to software languages. Since our method uses MontiCore to develop the structure and representation of models, we also give a brief introduction to how grammars are defined in the MontiCore language workbench. Finally, we briefly introduce AI task planning paradigms and symbolic task planning used in our framework.

\subsection{Behavior Trees and The Running Example}
BTs are directed trees originally developed to design and control non-player character behavior in video games~\cite{iovino_survey_2022}. In graph theory~\cite{russell_artificial_2010}, a directed tree is a hierarchical structure with a single root node and directed edges through which all other nodes are reachable. A node with an outgoing edge is a parent, the connected node is its child, nodes with no outgoing edges are leaves, and every node except the root has exactly one parent~\cite{russell_artificial_2010}.
In BTs, leaf nodes are action nodes, while non-leaf nodes are control flow nodes~\cite{iovino_survey_2022}. Graphically, a BT can be drawn with the root at the top and children below~\cite{iovino_survey_2022} (see Figure~\ref{fig:runningexample}).

Execution of a BT begins at the root, which periodically generates tick signals. A tick triggers node execution: action nodes perform tasks, while control flow nodes propagate the tick to one or more children according to their logic. A node executes only when ticked and returns one of three statuses to its parent: \emph{InProgress} if execution is ongoing, \emph{Success} if the goal is achieved, or \emph{Failure} otherwise. In the classical BT formulation, the main control flow nodes are Sequence, Fallback, and Parallel~\cite{iovino_survey_2022}.

A Sequence node ticks its children from left to right and returns \emph{Success} only if all children succeed. A Fallback node also ticks its children from left to right but returns \emph{Success} as soon as one child succeeds, returning \emph{Failure} only if all children fail~\cite{iovino_survey_2022}. A Parallel node ticks all children simultaneously and returns \emph{Success} if at least $M$ of $N$ children succeed, \emph{Failure} if success becomes impossible, and \emph{InProgress} otherwise~\cite{iovino_survey_2022}. Action nodes execute a command when ticked and return \emph{InProgress} during execution, \emph{Success} upon completion, or \emph{Failure} if execution fails. During execution, nodes may read from or write to a shared database typically called the~\emph{blackboard}, which stores information about the BT and allows nodes to communicate and coordinate ~\cite{iovino_survey_2022}.

BT nodes may also include decorators and services, which modify the behavior of attached nodes~\cite{iovino_survey_2022}. Conditions can be modeled as decorators applied to action or control flow nodes (\texttt{GripperEmpty} in Figure \ref{fig:runningexample}). Figure \ref{fig:runningexample} shows an example BT which will be used as a running example in this paper: execution first performs \emph{pick up plate~1} and \emph{place plate~1} sequentially, then executes the \emph{pick up and place beam~1} and \emph{pick up and place beam~2} sequences in parallel (\texttt{GripperEmpty} decorator executes before \texttt{PickUp Beam1}). 

\begin{figure}[t]
    \centering
    \includegraphics[scale=0.5]{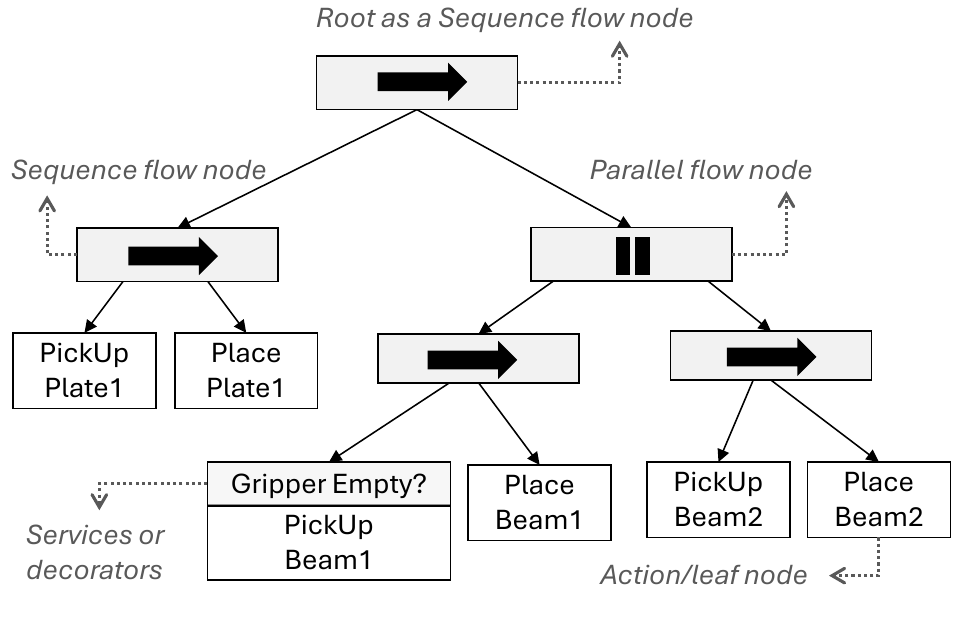}
\caption{Running example of a BT for picking and placing a plate and two beams. Sequence nodes are shown with an arrow icon and parallel nodes with two parallel lines. Decorators and services, when present, are attached to their associated node. Decorators can model node-specific preconditions evaluated before execution (e.g., \texttt{Gripper Empty?}).}

    \label{fig:runningexample}
\end{figure}

\subsection{Model-Driven Engineering and Software Languages}
In MDE, models describe different aspects of the system (e.g., structure, behavior) to reduce complexity and support maintainability~\cite{cheng16}.

\begin{definition} (Model): A model is a purpose-driven abstraction of an original real or intended system ~\cite{saam_2009,Sta73}.
\end{definition}

For example, the BT model in Figure~\ref{fig:runningexample} describes the task execution of one or more agents built to support interaction, monitoring, planning, and evaluation. A model must have an explicitly defined software language to be machine readable~\cite{combemale_2017}. A software language is an artificial language used in software development, including general-purpose languages (e.g., Python) and DSLs~\cite{combemale_2017}, such as BTs. The BT language specifies the syntax and semantics that enable the design and execution of BT models.

Defining a software language begins with its abstract syntax (AS), which captures its essential concepts and relationships (e.g., concept of a flow node or the child-parent relationship) ~\cite{combemale_2017}. AS can be specified via metamodels or grammars (Figure~\ref{fig:mm-bt} and Listing~\ref{lst:bt_revised}). The concrete syntax (CS) defines how users realize and use the AS concepts, either graphically (e.g., Figure \ref{fig:runningexample}), or textually~\cite{combemale_2017}. 
However, abstract and concrete syntax do not define model behavior. For example, tick propagation in a BT is not specified in AS or CS since users do not directly interact with it. Therefore, an additional model, which may extend the abstract syntax, specifies the language's runtime behavior.

Finally, language constraints in this paper are defined as structural well-formedness rules on the abstract syntax (e.g., nodes must have unique names) and behavioral well-formedness rules on the behavior model (e.g., only one database is allowed to store BT data such as beam and plate instances in the running example).

This paper uses the MontiCore language workbench that supports grammar-based and modular language definition of both AS and CS along with the implementation of well-formedness rules.

\subsection{MontiCore}
MontiCore is a language workbench for defining, processing, and integrating text-based DSLs~\cite{Monticore18}. It allows developers to specify concrete and abstract syntax in a single grammar, from which MontiCore generates core infrastructure such as parsers and ASTs~\cite{Monticore18}. An AST is a concrete instance of a model that conforms to the language’s abstract syntax and is used for model analysis and processing~\cite{Monticore18}. MontiCore supports modular language design through composition, inheritance, and embedding of language packages~\cite{Monticore18}. 
Each MontiCore grammar has a head and a body~\cite{Monticore18}. The head contains the package declaration, the grammar’s name, and any grammars it extends (Listing \ref{lst:bt_root}). 

\begin{nolinenumbers}

\begin{lstlisting}[
    language=MontiCore, 
    basicstyle=\scriptsize\ttfamily, 
    basewidth=0.5em,
    lineskip=-1pt,
    caption={Grammar head definition defining \texttt{BehaviorTree} grammar that extends \texttt{MCBasics} grammar}, 
    label={lst:bt_root}
]
grammar BehaviorTree extends de.monticore.MCBasics {   
  // Grammar body containing symbols and productions 
}
\end{lstlisting}
\end{nolinenumbers}
The body defines the language itself. It contains:
\begin{itemize}

    \item Non-terminals, which define the abstract syntax of the language.
    \item Terminals, shown in quotes (e.g., "Action"), which define the concrete syntax.
    \item Productions, which are the main rules that define valid language concepts. Each production corresponds to a non-terminal and specifies how it can be constructed from terminals, properties, and other non-terminals.
    \item Programming concepts, e.g., interface, or enumeration non-terminals, which allow language extension and reuse.
    \item Tokens, such as \texttt{Name} or \texttt{Integer}, can be defined or imported from standard MontiCore grammars like MCBasics. 
\end{itemize}

MontiCore also supports symbols, which identify and reference language entities by name. Declaring a non-terminal as a symbol ensures that the value assigned to this non-terminal's name property corresponds to an existing entity of the specified type. For example, if we have a non-terminal defined as \texttt{symbol Beam}, then another production rule can use \texttt{Name@Beam} to ensure that the name provided in the model refers to a valid \texttt{Beam} instance. After the model is parsed by MontiCore, symbols can be stored in a symbol table and used for enforcing well-formedness rules.
The following example (Figure \ref{fig:grammar_anatomy_dotted}) shows the body of a portion of the BT grammar, and Listing \ref{lst:example} provides a valid and invalid example model based on it.

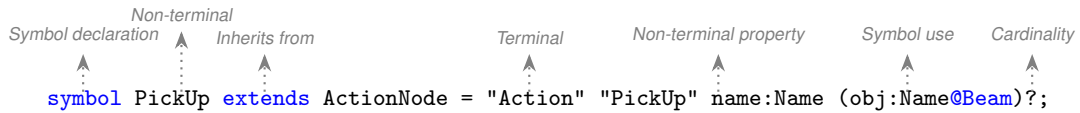
\begin{figure}[ht]
\centering
\begin{tikzpicture}[
    annot_style/.style={-Stealth, font=\scriptsize\sffamily\itshape, gray, thick, dotted},
    node distance=0.5cm
]

    \node (code) [font=\small\ttfamily] {
        \textcolor{blue}{symbol} PickUp \textcolor{blue}{extends} ActionNode = "Action" "PickUp" name:Name (obj:Name\textcolor{blue}{@Beam})?;
    };

    \coordinate (sym)  at ($(code.west) + (0.6, 0.1)$);
    \coordinate (ext)  at ($(code.west) + (3, 0.1)$);
    \coordinate (nter) at ($(code.west) + (1.9, 0.1)$);
    \coordinate (term) at ($(code.west) + (6.5, 0.1)$);
    \coordinate (prop) at ($(code.west) + (9, 0.1)$);
    \coordinate (suse) at ($(code.west) + (11.5, 0.1)$);
    \coordinate (card) at ($(code.west) + (13.15, 0.1)$);

    \draw [annot_style] (sym)  -- ++(0, 0.5) node[above] {Symbol declaration};
    \draw [annot_style] (nter)  -- ++(0, 0.8) node[above] {Non-terminal};
    \draw [annot_style] (ext)  -- ++(0, 0.5) node[above] {Inherits from};
    \draw [annot_style] (term) -- ++(0, 0.5) node[above] {Terminal};
    \draw [annot_style] (prop) -- ++(0, 0.5) node[above] {Non-terminal property};
    \draw [annot_style] (suse) -- ++(0, 0.5) node[above] {Symbol use};
    \draw [annot_style] (card) -- ++(0, 0.5) node[above] {Cardinality};

\end{tikzpicture}
\caption{A MontiCore grammar production rule with annotated components.}
\label{fig:grammar_anatomy_dotted}
\end{figure}
Figure \ref{fig:grammar_anatomy_dotted} shows an example of a production where \texttt{PickUp} is a non-terminal that extends the \texttt{ActionNode} non-terminal. The terminals \texttt{"Action"} and \texttt{"PickUp"} require that valid model instances begin with these keywords. The \texttt{name:Name} property defines a named attribute of \texttt{PickUp}, while \texttt{obj:Name@Beam} is a symbol reference, ensuring that the provided name corresponds to an existing \texttt{Beam} instance. Cardinalities can be specified using \texttt{?} for zero or one (optional), \texttt{*} for zero or more, and \texttt{+} for one or more occurrences. Based on this grammar and the defined well-formedness rules, a model can be validated as correct or incorrect (see Listing~\ref{lst:example}).

\begin{nolinenumbers}
\begin{lstlisting}[
    language=MontiCore, 
    basicstyle=\scriptsize\ttfamily, 
    basewidth=0.5em,
    lineskip=2pt,
    caption={Syntactic validation examples for PickUp action nodes.}, 
    label={lst:example},
    escapechar=| % Use | to escape to LaTeX for colors
]
// Correct model: pickup action with parameters
"Action" "PickUp" PickUpBeam1 (beam1)
// Correct model: pickup action without parameters (optional in grammar)
"Action" "PickUp" PickUpBeam1 
// Incorrect model: Missing "Action" terminal and incorrect parameter type
 |\color{red}\bfseries"PickUp"| PickUpBeam1 (|\color{red}\bfseries gripper1|)
\end{lstlisting}
\end{nolinenumbers}

\subsection{AI Task Sequence Planning Paradigms: Symbolic and Sub-symbolic}

Several AI methods exist for planning task sequences, varying in their underlying principles. 
Symbolic AI planners generate task sequences from an initial state to a goal state by manipulating symbolic representations of the problem and applying logical reasoning~\cite{Ghallab04}. Unlike data-driven AI, they rely on explicitly defined knowledge and rules, allowing plans to be deduced through inference. Symbolic planners model the world as a directed graph in which nodes represent world states and edges represent actions. World states describe configurations of the environment using predicates, which are logical statements about properties of the world (e.g., \texttt{At(robot, roomA)}). Actions define transitions between states, such as moving a robot from one location to another~\cite{Ghallab04}.

Each action (e.g., \texttt{(travel A B)}) is defined by its parameters (e.g., locations \texttt{A} and \texttt{B}), preconditions that must hold before execution (e.g., the robot is at location \texttt{A}), and effects that describe how the world state changes (e.g., the robot is at location \texttt{B}). At runtime, the planner starts from the initial state and reasons over actions’ preconditions and effects to find a sequence of actions that transforms the world into the goal state.  
Symbolic planners typically use declarative languages such as the Planning Domain Definition Language (PDDL)~\cite{Haslum19} to specify states, actions, predicates, and goals. In PDDL-based planning, the planning problem formulation is typically divided into a \emph{domain file} and a \emph{problem file}~\cite{Ghallab04}. The domain file defines the general planning model, including predicates and action types (with their properties, preconditions, and effects), while the problem file specifies a concrete planning instance by defining the initial state, object instances, and goal state~\cite{Ghallab04}.

Sub-symbolic planners operate on data-driven representations rather than explicitly defined symbolic models~\cite{hitzler2022neurosymbolic}. They use techniques such as neural networks or reinforcement learning to learn action-selection policies from data instead of reasoning over symbolic predicates and logical rules~\cite{hitzler2022neurosymbolic}. Given a numerical representation of the environment state, the learned policy predicts the next action at each step, producing behavior through sequential action selection rather than constructing an explicit symbolic plan~\cite{hitzler2022neurosymbolic}.

\section{Research Challenges and Related Work} \label{relatedworks}

The following subsections review the state of the art in robotic task planning and execution to identify three key challenges motivating this work. Specifically, they examine (C1) managing data dependencies between independently developed task plan modules, (C2) support for localized runtime adaptation and replacement of task plan modules, and (C3) scalable, robust automated generation of executable task plans from heterogeneous planning methods. Due to their modularity and composability, BTs are adopted in this work as the target task execution modeling language, and challenges are particularly discussed for BTs.

\subsection{Type Mismatches and Hidden Dependencies (C1)} \label{btchallenge}

Different execution models, including FSMs, hierarchical state machines, and BTs, can be used to represent robotic task plans~\cite{Matteo24}. State machines explicitly model execution states and state-specific transitions~\cite{Matteo24}. While effective for many robotic applications, they become increasingly difficult to maintain and scale as task complexity grows because transitions are tightly coupled to individual states, and modifying one part of the execution model often requires updating related transitions~\cite{Matteo24}. BTs alleviate this limitation by providing a hierarchical and modular execution model in which every node follows a common execution interface, enabling task execution models to be composed from reusable subtrees~\cite{Matteo24}. Regardless of the execution modeling language, composing independently developed task execution models requires explicit mechanisms for communicating data between tasks. Without explicitly modeled data exchange, type inconsistencies, unintended side effects, and execution failures may occur, even when individual task execution models function correctly in isolation~\cite{iovino_survey_2022,voinov_model_2010}.

In BTs, communication between nodes and subtrees typically relies on a centralized blackboard. When these data exchanges are not explicitly modeled, hidden dependencies arise from implicit data sharing and hardware constraints (e.g., the limited range of motion of a robot arm), leading to type inconsistencies, unintended side effects, execution failures, and unintended robotic behavior~\cite{iovino_survey_2022,Oliver_Biggar_2024,Alexander_SchulzRosengarten_2024}.

To prevent hidden dependencies, mechanisms that decentralize node communication and explicitly model the inputs and outputs of BT nodes have been suggested~\cite{Oliver_Biggar_2024, Alexander_SchulzRosengarten_2024}. For instance, a BT-inspired programming language based on Haskell~\cite{Oliver_Biggar_2024} facilitates modeling tasks, particularly "chattering tasks" that rely on changing conditions, with explicit data transfer details. This method focuses on the reactive selection and monitoring of tasks in autonomous agents and recommends approaching tasks like function calls with distinct inputs and outputs~\cite{Oliver_Biggar_2024}. In a similar approach, a data flow extension to BTs was proposed to explicitly represent data relationships and communication between nodes~\cite{Alexander_SchulzRosengarten_2024}. Each action node has explicit inputs and outputs and points to the subsequent action node to execute upon every tick, ensuring accurate data flow. This approach allows nodes to exchange data through predefined channels~\cite{Alexander_SchulzRosengarten_2024}. 

However, these methods alter the core BT logic by modeling pointers to the next actions instead of using flow nodes for action transition. They also model direct data flows between nodes. This approach makes task interfaces more explicit, but reduces modularity, as in state machines, thereby hindering the modeling of complex execution models. 

Other approaches, such as BehaviorTree.CPP~\cite{Ghzouli20}, allow modeling inputs and outputs for each BT node but lack formal verification methods to automatically check whether the BT behaves as intended under different scenarios~\cite{Uwe_Koeckemann_2023}. Consequently, a method was developed to automatically identify data and sequence constraint violations to facilitate debugging~\cite{Uwe_Koeckemann_2023}. The authors formulate data dependency verification as a symbolic AI planning problem, seeking event sequences that lead to nodes lacking the required data and thereby identifying potential failures.

However, these verification methods are primarily applicable to small BTs or individual subtrees. They do not provide mechanisms for automatically extracting and analyzing data dependencies across large-scale BTs with numerous subtrees. Consequently, the following research challenge arises:

\begin{mdframed}[style=summary, innertopmargin=5pt]
\textbf{C1: Hidden data dependencies and type mismatches hinder modular composition of task execution models.}
As task hierarchies and independently developed task execution models are composed, implicit data dependencies and mismatches arising from shared data and undeclared interfaces may lead to type conflicts, unintended side effects, and non-deterministic runtime behavior~\cite{Matteo24}.

\end{mdframed}

\subsection{Limited Support for Localized Task Plan Replacement (C2)}

CRF task plans are multifaceted and often consist of multiple independently developed plan modules~\cite{sherkat_potentials_2022}. During execution, changes in the environment, resource availability, or task failures may invalidate only a subset of the overall task execution plan~\cite{wang_task_2025}. Ideally, only the affected execution model should be re-planned and replaced, while the remaining execution structure is preserved. Existing task execution modeling languages provide limited support for such localized replanning~\cite{wang_task_2025}. 

State-machine-based approaches require modifying the corresponding transition structures when replacing parts of an execution model, as adding or removing a state requires updating all related transitions across the machine~\cite{Matteo24, ghzouli_behavior_2023}. While some state-machine frameworks, such as FlexBE, do support modifying the system while it is running, this works by reloading and swapping in the entire execution model, rather than a small, targeted fix to just the affected structure~\cite{ghzouli_behavior_2023}. Furthermore, when a plan becomes infeasible during execution, FSM-based systems typically discard the running plan entirely and replan from scratch, as seen in systems where failure outcomes are fed back to an external planner that generates a new state machine for execution~\cite{murata_multi-robot_2026}. Critically, the runtime modification mechanisms that exist in FSM frameworks often require human intervention to initiate structural changes during execution~\cite{schillinger_human-robot_2016}.

BTs typically support flow node types of sequence, fallback, parallel, and utility~\cite{Colledanchise18}. The execution logic for these flow node types is predefined and mixed with ordering semantics. Consequently, in more complex task planning scenarios, it is often necessary to develop multiple flow node types tailored to specific problems~\cite{Matteo24}. This rigid structure of BT flow nodes can lead to many layers of hierarchy, for example, when adjusting the BT nodes dynamically at runtime~\cite{Matteo24}. Hence, while the simplicity of flow node types is important and enhances modularity, it can also result in complex and dysfunctional structures with many layers of hierarchy, making it difficult to scale up and maintain BTs~\cite{Matteo24}. 
Segmenting BTs and using LLMs to explain each subtree has been proposed to make complex BT structures more understandable~\cite{Hang_Su_2023}. This segmentation makes it easier for non-expert users to grasp BT structure and functionality.  This approach, while increasing understandability, does not address the challenge of rigid BT structures. 

The modularity of BTs allows for the programming of problem-specific flow nodes and subtrees, but this can also lead to behavior models that require defining and maintaining many problem-specific flow nodes (e.g., constraints between actions such as temporal constraints), thus limiting flexibility~\cite{Marzinotto_ogren}. Specifically, modeling reactive tasks using BTs can become too complicated, which is why "stateful behavior trees (SBTs)" were introduced, a method that borrows from BTs and state charts~\cite{serbinowska2023formalizing}. SBTs differ from BTs as they execute all nodes in one tick and reduce the reliance on predefined  flow nodes. Each node in this method determines the next node based on the conditions and transitions modeled within that node. On the other hand, unlike state charts, which focus on system states and event-triggered transitions without managing memory, SBTs incorporate both memory (the blackboard) and environmental conditions. This allows SBTs to adapt execution models based on internal and external factors, making them more suitable for dynamic environments, but also limits them to very specific and simple planning problems~\cite{serbinowska2023formalizing}. Another state machine-inspired approach models nodes as states with preconditions and effects and uses these conditions for state transitions, but the method abstracts away the internal structure of BTs, meaning that the relation between actions and states is not specified explicitly as in BTs~\cite{Paxton19}. The authors argue that this method allows for easier programming of complex behaviors compared to hard-coding multiple flow nodes and can quickly react to changing circumstances without the need for extensive manual engineering of tree structure. While these two methods enhance efficiency for simple and specific planning problems, they share scalability issues with state charts, as transitions must be modeled for each state, reducing modularity and scalability. Consequently, the following research challenge arises:

\begin{mdframed}[style=summary, innertopmargin=5pt]
\textbf{C2: Existing execution modeling languages provide limited support for locally replacing or replanning task plans at runtime.} 
 Adapting to execution failures or environmental changes often requires restructuring or regenerating the whole module, limiting flexibility and leading to state loss.

\end{mdframed}

\subsection{Limited Automated Generation of Execution Models (C3)}

Execution modeling languages specify task execution and do not perform task planning. Consequently, execution models must either be constructed manually or generated automatically from planning algorithms~\cite{Matteo24}. Automated generation of task execution models from symbolic planners has been explored for several modeling languages. For example, ROSPlan~\cite{cashmore_rosplan_2015} links PDDL planning to ROS execution by dispatching plan actions sequentially through a centralized dispatcher. Another study~\cite{ridge_rapid_2017} generates state machine code automatically, but from human-authored meta-scripts. FSM synthesis from formal specifications such as linear temporal logic is theoretically well-founded but has proven computationally intractable at scale~\cite{kress-gazit_synthesis_2018}. 

BTs are typically modeled manually. Manual BT design is often effective for simple planning problems due to its efficiency and ease of design~\cite{Matteo24}. For example, modeling task and motion of a robot pushing objects~\cite{Filippo_Bertoncelli_2024}, or modeling multiple behavior strategies as separate subtrees for robots to choose from~\cite{JH_Cloete_2024}.  
Although manual BT design, including designing the initial nodes and BT structure, demonstrates success in simple planning problems, it becomes time-consuming and unmanageable as planning complexity increases, making maintenance, debugging, and verification increasingly challenging~\cite{iovino_survey_2022}. Automated symbolic AI task planners and sub-symbolic task planners have been used to automate BT generation~\cite{iovino_survey_2022}. The following sections analyze the challenges associated with each category.

\subsubsection{Symbolic AI Planners}

Symbolic AI planners offer structured approaches to BT generation and validation, but they face challenges in scalability and flexibility~\cite{iovino_survey_2022}. For example, hierarchical task networks (HTNs) have been used to generate BTs~\cite{ Rovida17,Neufeld18}. HTN’s high-level decomposable nodes (called methods) become sequence nodes in BT, and fallback nodes represent decomposition options~\cite{Neufeld18}. Another approach generates a plan using HTN in a similar way, and then, based on preconditions and effects that are modeled as decorators in BTs, the parallel nodes will be identified to optimize execution~\cite{Rovida17}. They also introduce a shared semantic model where nodes are the objects and edges are the semantic relations between them to be used for communication between the tree and the planner. Our method is inspired by this work, as we also model preconditions and effects, though not as decorators. However, this method for BT generation leads to brittle BTs as the structure of the tree is determined by the HTN planner, and no further information is mentioned on how the planner is synchronized with BT architecture and the semantic world model. Similarly, PlanSys2~\cite{martin_plansys2_2021} generates executable BTs from the output of PDDL-based symbolic planners using BehaviorTree.CPP. However, the generation process is tightly coupled to the underlying planner representation, making it difficult to generalize across different symbolic planning paradigms. 

The back-chaining algorithm, used in AI planners, has also been employed to generate BTs~\cite{Colledanchise23, Tadewos19}. Starting from the goal condition, it finds actions that meet the conditions. Then it looks at the preconditions of those actions and tries to find actions that satisfy them in the next layer of BT hierarchy, and so on~\cite{Colledanchise23}. Another study uses the same algorithm but focuses on multi-agent synchronization using an auction-based approach~\cite{Tadewos19}. 

Other approaches include manually generating the initial BT and automatically generating BT nodes at runtime upon the failure of a node by evaluating a list of actions to identify those whose postconditions contain the preconditions of the failed node~\cite{Ning_Li_2022}. Additionally, a conceptual multi-layered architecture for robot control has been proposed where the top layer handles task planning using symbolic AI planners, the middle layer employs BTs, and the lowest layer directly interacts with the robot hardware~\cite{Marzinotto_ogren}. However, all of these symbolic planning methods are inherently limited to simple task planning problems as they face challenges such as combinatorial explosion, and considering the ease of designing BTs manually, they do not offer any practical use~\cite{Matteo24}.

\subsubsection{Sub-symbolic AI Planners}

Several sub-symbolic AI methods have been used to automatically generate BTs, such as LLMs, learning BTs, genetic algorithms, and case-based reasoning.

Two recent approaches use LLMs to automate the generation of BTs from natural language instructions, with the aim of simplifying robot programming~\cite{J_Liang_2024,Izzo24}.  The DiaGBT method~\cite{J_Liang_2024} receives instructions in natural language, analyzes them using LLMs to identify key components, generates individual BT nodes based on these components, and then assembles them into a coherent BT, again using LLM for the assembly process. The second approach develops a dataset of 600 BTs paired with natural language task descriptions to generate BT models~\cite{Izzo24}. The generated BTs are evaluated across nine robotic tasks, including navigation and manipulation~\cite{Izzo24}.

Another challenge is for the AI agent to effectively learn to represent and manipulate BTs using LLMs, which has been addressed by transforming BTs into vector spaces. For example, a method has been proposed to encode BTs into a compact numerical form, aiming to enhance the learning capabilities and adaptability of robotic tasks~\cite{Y_Chao_2022}. This approach leverages pre-trained word embeddings and a node-aggregation mechanism to capture both the semantic information of node descriptions and the structural hierarchy of BTs~\cite{Y_Chao_2022}.  Similarly, the WE-BTR method also relies on vector embeddings to represent BTs~\cite{Hang_Su_2023}. However, this method focuses on recommending BTs by segmenting them into subtrees, aggregating nodes into vectors, and then evaluating the similarity between BTs based on these vectors.

Other AI methods used for the automatic generation of BTs include learning from demonstration, complemented by genetic algorithms, which enable the robot to understand tasks based on human input and optimize performance metrics~\cite{Matteo_Iovino_2023}. Case-based reasoning is another method in which the system maintains a "memory" of past BT states. When a new situation arises, it queries this memory, compares the current state to previous cases using a similarity function, and retrieves the most similar state~\cite{Florez21, Florez10}. Finally, genetic algorithms are used to evolve BTs by first creating an initial population of BTs with random structures and parameters. Each BT is then evaluated in the population by running simulations or real-world tests~\cite{Scheper16}.

However, sub-symbolic AI-based approaches have no method to accurately verify the task plan, and their performance reduces as the number of tasks increases. Additionally, the performance of these data-based methods is largely based on the quality of the training data. This dependency can compromise consistency and flexibility, particularly in unpredicted situations where robot tasks may change unexpectedly. Consequently, the remaining research challenge is:

\begin{mdframed}[style=summary, innertopmargin=5pt]
\textbf{C3: Limited scalability and robustness of automated task execution model generation.}
Automated generation of execution models with sub-symbolic AI planners is unreliable, and symbolic AI planners have limitations: they focus on niche problem types, struggle with complex planning due to combinatorial explosion, and produce brittle execution models that fail with slight input variations. As a result, scalable and modular generation of task plans that can be reliably transformed into execution models remains an open challenge.
\end{mdframed}

\begin{table*}[t]
\centering
\footnotesize
\caption{Summary of research questions and corresponding metrics.}
\label{tab:rq_metric_summary}
\begin{tabularx}{\textwidth}{p{0.8cm} X p{1.6cm} X}
\toprule
\textbf{Q} & \textbf{Evaluation Question} & \textbf{Metric} & \textbf{Metric Description } \\
\midrule

Q1--Q4 & Detection or prevention of structural, syntactic, type, and logical errors & M1 ($TPR$) & Ratio of injected errors correctly detected \\
 &  & M2 ($PR$) & Ratio of errors prevented by design \\

Q5 & Preservation of structure during local replanning & M3 ($\overline{G}_{\text{tree}}$) & Average structural change after replanning \\

Q6 & Scalability of repeated local replanning & M4 (RF) & Number of replanning events per minute \\
 &  & \hl{M5 ($NRR$)} & \hl{Ratio of nodes retained in the BT to generated nodes} \\

Q7 & Composition of plans from multiple symbolic planners & M6 & Number of distinct symbolic planners applied \\
 &  & M7  & Number of planning and replanning applications \\

Q8 & Decomposition and recomposition of planning problems & M8 ($PL$) & Cumulative and normalized planning latency \\
 &  & M9 ($DF$) & Average executed planning steps per planning module \\

Q9 & Plan correctness & M10 ($PPR$) & Ratio of logically correct task steps \\
 &  & M11 ($SSR$) & Ratio of superfluous task steps \\

 Q10 & Runtime performance 
 &   M12 ($\overline{CAL}$) & Average latency between consecutive action executions \\

\bottomrule
\end{tabularx}
\end{table*}

\section{Research Questions and Evaluation Goals} \label{reqs}
Based on the three challenges identified in the related work section, we briefly describe how each challenge can be addressed, followed by the corresponding evaluation goals, research questions (Q), and evaluation metrics (M) defined using the Goal–Question–Metric (GQM) method. The GQM method provides a structured framework for deriving measurable evaluation criteria from high-level goals~\cite{basili_methodology_1984}. A summary of metrics is provided in Table~\ref{tab:rq_metric_summary}.

\subsection{Error Detection of Task Execution Models}

To address the implicit data dependencies that hinder modular composition and reuse of task execution models (C1), modeling inconsistencies should be detected or prevented before they result in runtime crashes or silent failures. Hence, the \textbf{goal} is to analyze APTree task execution models for the purpose of evaluating their ability to detect or prevent modeling and data inconsistencies before or during execution, from the viewpoint of CRF developers, in the context of symbolic task planner integration.
 \textbf{Questions:}
 \begin{itemize}

 \item \textbf {Q1:}
To what extent does the task execution model detect or prevent invalid structural configurations in the planning and execution models (M1, M2)? For example, can it detect a missing planning domain or problem definition?

\item \textbf {Q2:} To what extent does the task execution model detect or prevent undefined or inconsistent types (M1, M2)? For example, can it detect undefined node types (e.g., \texttt{PickUp} action that is not registered) or interface mismatches between execution model tasks?

\item \textbf {Q3:}
To what extent does the task execution model detect or prevent syntactic inconsistencies in planning and execution models (M1, M2)? For example, can it detect missing parentheses or invalid node declarations?

 \item \textbf {Q4:}
To what extent does the task execution model detect or prevent logical inconsistencies arising from independently generated task execution models (M1, M2)?
For example, can it detect that the execution of plan A invalidates plan B?

 \end{itemize}
To operationalize Q1–Q4, we define two metrics: M1, the true positive rate (TPR), which measures the proportion of injected structural, typing, syntactic, and logical errors detected at design time before causing a runtime crash or silent failure, and M2, the prevention rate (PR), which measures the proportion of injected errors prevented by design.

\textbf{M1, true positive rate (TPR):}
The error detection rate is computed as the true positive rate (recall):

\begin{equation}
\textbf{M1:}\quad
\mathrm{TPR} =
\frac{E_{DT}+E_{PT}}
{E_{DT}+E_{PT}+E_{RT}+E_{SF}+E_{RTC}}
\times100
\end{equation}

where $E_{DT}$, $E_{PT}$, and $E_{RT}$ denote the numbers of injected errors detected at design time, planning time (detected by the planner's parser), and runtime, respectively; $E_{SF}$ and $E_{RTC}$ denote the numbers resulting in silent failures and runtime crashes or failure. Errors classified as not applicable (NA) and prevented by design (PD) are excluded. A higher TPR indicates stronger design time error-detection capability. Because the evaluation is based on injected violations of formally specified well-formedness constraints, false positives are not applicable. Consequently, precision is 100\%, and the F1 score equals $\frac{2\cdot\mathrm{TPR}}{1+\mathrm{TPR}}$.

\textbf{M2, prevention rate (PR):}
In addition to TPR, error classes prevented by design are computed as:

\begin{equation}
\textbf{M2:}\quad
\mathrm{PR} =
\frac{E_{PD}}
{E_{PD}+E_{DT}+E_{PT}+E_{RT}+E_{SF}+E_{RTC}}
\times100
\end{equation}

where $E_{PD}$ denotes the number of errors prevented by design. As with M1, errors classified as not applicable (NA) are excluded. A higher PR indicates that more errors are eliminated by method design.

\subsection{Support for Local Task Replanning}

To address the challenge of inflexible execution models for local replanning (C2), the execution model must support localized runtime modification without reconstructing the entire execution model. Hence, the \textbf{goal} is to analyze APTree for the purpose of evaluating its support for localized runtime replanning of independently developed task execution models, from the viewpoint of construction robotics developers in the context of adaptive automated task execution.
\textbf{Questions:}

\begin{itemize}

\item \textbf{Q5:} To what extent does local replanning preserve the structure of the running task execution model (M3)?
   
\item \textbf{Q6:}  To what extent does the approach sustain execution through repeated local replanning without restarting the execution model or causing execution failures (M4, M5)?

\end{itemize}

To operationalize Q5--Q6, we define three metrics: M3, the structural change ($G_{\text{tree}}$), addressing Q5; M4, replanning frequency in one run (RF); \hl{and M5, node retention rate (NRR), addressing Q6.}

\textbf{M3, structural change ($G_{\text{tree}}$):}
The structural impact of local replanning is quantified across three BT dimensions: size (number of nodes), breadth (maximum width), and depth (longest path)~\cite{GUGLIERMO24}. The relative change in each dimension is computed as:
$
g_n=\left|\frac{N_1-N_0}{N_0}\right|$, $
g_w=\left|\frac{W_1-W_0}{W_0}\right|$, $
g_d=\left|\frac{D_1-D_0}{D_0}\right|,
$
where $N$, $W$, and $D$ denote the number of nodes, maximum width, and depth of the BT before ($N_0$, $W_0$, $D_0$) and after ($N_1$, $W_1$, $D_1$) replanning. The structural change for replanning event $i$ is computed as:
\begin{equation}
\textbf{M3:}\quad
G_{\text{tree},i}=\frac{g_{n,i}+g_{w,i}+g_{d,i}}{3}\times100\qquad
\overline{G}_{\text{tree}}=\frac{1}{R}\sum_{i=1}^{R}G_{\text{tree},i}
\end{equation}
where $R$ is the total number of replanning events. Lower $\overline{G}_{\text{tree}}$ values indicate higher structural preservation.

\hl{\textbf{M4, replanning frequency (RF) and M5, node retention rate (NRR):}
To assess the scalability of repeated local replanning, replanning frequency and node retention rate are computed as:}
\begin{equation}
\textbf{M4:}\quad
\mathrm{RF}=\frac{R}{T_{\mathrm{min}}},\qquad \textbf{M5:}\quad
\mathrm{NRR}=\frac{N_{\mathrm{retained}}}{N_{\mathrm{generated}}}\times100
\end{equation}
\hl{where $R$ is the total number of replanning events triggered, $T_{\mathrm{min}}$ is the execution time in minutes, $N_{\mathrm{generated}}$ is the total number of execution nodes generated across all planning and replanning events, and $N_{\mathrm{retained}}$ is the number of execution nodes that remain in the final BT and are executed to completion (success or failure). All metrics are measured over a single execution run. RF measures how often the execution model is modified at runtime, and NRR the proportion of generated execution nodes retained in the final BT. Together they quantify the degree of runtime plan change the framework sustains without restarting execution; they do not measure task outcomes, failure severity, or recovery performance.}

\subsection{Scalable and Robust Automated Execution Model Generation}

To address the challenge of scalable and robust automated execution model generation (C3), execution models must support the integration of task plan modules generated by multiple heterogeneous symbolic planners while maintaining the logical correctness of the overall execution model at runtime. Hence, the \textbf{goal} is to analyze the APTree framework for the purpose of evaluating its support for scalable and modular automated execution model generation using symbolic planners from the viewpoint of CRF developers, in the context of modular plan composition.
\textbf{Questions:}

\begin{itemize}
   \item \textbf{Q7:} To what extent can the approach compose sub-plans generated by multiple symbolic planners, potentially invoked multiple times across planning modules, to solve a single planning problem (M6, M7)?

    \item \textbf{Q8:} To what extent can large-scale symbolic planning problems be decomposed into smaller sub-problems and their resulting task plans recomposed into a single execution model (M8, M9)?

    \item \textbf{Q9:} To what extent does the automatically generated execution model preserve the logical correctness of the composed task plan (M10, M11)?

   \item \textbf{Q10:} What is the runtime latency between consecutive action executions, and how do runtime planning and replanning affect this latency (M12)?
\end{itemize}

To operationalize Q7--Q10, we define M6, planner diversity, and M7, planning application count, for Q7; M8, planner latency (PL), and M9, decomposition factor (DF), for Q8; M10 and M11 to assess the logical correctness and efficiency of the generated task execution model for Q9; and M12 to measure the runtime latency between consecutive action executions for Q10.

\textbf{M6-M7, planner diversity and planner application count:}
Planner diversity measures the number of distinct planners successfully invoked during a single execution, demonstrating support for heterogeneous planner application. The planner application count measures the total number of planning and replanning applications at runtime. Higher values indicate stronger support for repeated autonomous planning in independently executed planning modules.

\textbf{M8-9, \hl{p}lanner latency (PL) and decomposition factor (DF):}
These metrics evaluate the efficiency and modularity of planning. Planner latency (PL) measures the overhead of applying a symbolic planner, including planner invocation, model transformation, and plan integration, and is reported as both cumulative and normalized latency during execution. The decomposition factor (DF) quantifies the average number of executed planning steps contributed by each planning submodule. 
\begin{equation}
\textbf{M8:}\quad
\mathrm{PL}_{\mathrm{cum}}=\sum_{i=1}^{N} t_i\qquad,\qquad
\mathrm{PL}_{\mathrm{norm}}=\frac{\mathrm{PL}_{\mathrm{cum}}}{T_{\mathrm{exec}}}\times100\qquad
\textbf{M9:}\quad
\mathrm{DF}=\frac{\mathrm{FPS}}{B_{\mathrm{parent}}}
\end{equation}
Here, $t_i$ is the planning latency of the $i$-th planning or replanning invocation, $N$ is the total number of planning and replanning invocations, $T_{\mathrm{exec}}$ is the total execution time expressed in the same time unit as $t_i$, $B_{\mathrm{parent}}$ is the number of planning submodules contributing at least one executed action to the final composed plan, and $\mathrm{FPS}$ is the final executed plan length. PL measures the ratio of planning time to full execution time, whereas lower DF values indicate finer decomposition of the final plan across planning submodules, reflecting greater structural modularity and support for localized planning.

\textbf{M10-11, plan parity ratio (PPR) and superfluous step ratio (SSR):}
These metrics evaluate the logical correctness of automatically generated task plans against a reference plan. Logical correctness is assessed by preserving task dependencies and execution ordering, rather than exact action names or the number of planning steps. Accordingly, the PPR measures the proportion of logically equivalent reference task steps preserved in the generated plan, while the SSR measures the proportion of generated task steps that are not required by the reference plan: 
\begin{equation}
\textbf{M10:}\quad
\mathrm{PPR}=\frac{S_{\mathrm{matched}}}{S_{\mathrm{reference}}}\times100,\qquad
\textbf{M11:}\quad
\mathrm{SSR}=\frac{S_{\mathrm{extra}}}{S_{\mathrm{reference}}}\times100
\end{equation}
Here, $S_{\mathrm{matched}}$ is the number of reference task steps whose logical dependencies and execution ordering are preserved in the generated plan, $S_{\mathrm{reference}}$ is the total number of task steps in the reference plan, and $S_{\mathrm{extra}}$ is the number of generated task steps that are not required by the reference plan. Higher PPR values indicate greater preservation of the reference plan's logical structure, whereas lower SSR values indicate fewer unnecessary planning steps.

\textbf{M12, consecutive action latency (CAL):}
This metric measures the time elapsed between the completions of consecutive actions during execution, including any planning or replanning delay incurred between them:
\begin{equation}
\textbf{M12:}\quad
\mathrm{CAL}_i=c_{i+1}-c_i,\qquad
\overline{\mathrm{CAL}}=\frac{1}{N-1}\sum_{i=1}^{N-1}\mathrm{CAL}_i
\end{equation}

\section{APTree Domain Specific Language} \label{packages}
This section defines our DSL in separate packages. Each package begins with the definition of the abstract syntax using a metamodel, which specifies the core concepts of the DSL package and their relationships, followed by structural well-formedness rules. The concrete syntax is then defined using MontiCore grammars, together with an example of a valid model. Finally, we describe the behavior model and the behavioral well-formedness rules related to each package. This structure is applied consistently across all DSL packages.

\subsection{Behavior Tree DSL Package } \label{bt}

\begin{figure*}[t]
    \centering
    \includegraphics[width= 0.55\linewidth]{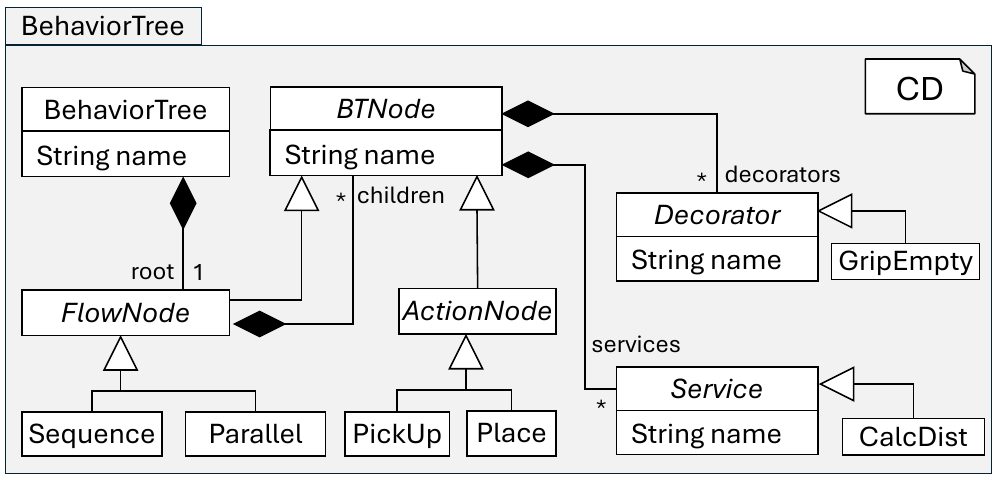}
\caption{BT metamodel used and extended in this paper for the proposed DSL, highlighting nodes, decorators, services, and their relations.}
    \label{fig:mm-bt}
\end{figure*}

The BT metamodel defines the traditional BT structure used by some existing BT implementations. We define a concrete text-based syntax for it and present the metamodel, as BT metamodels may vary slightly across tools. 

\textbf{Abstract syntax}: Figure~\ref{fig:mm-bt} presents the metamodel of the BT language. The BT language includes:
\begin{itemize}
   \item BT (\texttt{BehaviorTree}). Allows defining a BT instance model with exactly one flow node, which is always \texttt{root}.
   
    \item Nodes (\texttt{FlowNode} and \texttt{ActionNode}). These represent the two main types of nodes in a BT. Action nodes model executable actions (e.g., \texttt{PickUp}, \texttt{Place}), while flow nodes model control logic (e.g., \texttt{Sequence}, \texttt{Parallel}). 
    
\item Decorators and services (\texttt{Decorator}, \texttt{Service}). Decorators modify node execution or its outcome (e.g., \texttt{GripEmpty}), while services provide auxiliary functionality (e.g., \texttt{CalcDist}).

\end{itemize}

The BT metamodel is constrained by the following structural well-formedness rules. Names of nodes, BTs, services, and decorators must be unique within the model scope. Flow nodes may contain child nodes, whereas action nodes are leaf nodes and cannot have children. Furthermore, the first node in a BT is always designated as "Root".

\textbf{Concrete syntax}: Listing~\ref{lst:bt_revised} shows the concrete syntax defined using MontiCore grammar, and Listing~\ref{lst:robot_tree_instance} provides a valid instance of the BT from our running example (Figure~\ref{fig:runningexample}) using our textual DSL. We define BTs and nodes as symbols, enabling trees and nodes to reference each other by name when the grammar is extended in the subsequent packages. Concrete instances of actions, services, and decorators (e.g., PickUp and GripEmpty) are not restricted to the types defined in the metamodel and grammar; rather, they serve only as examples.

\begin{nolinenumbers}

\begin{lstlisting}[
    language=MontiCore, 
    basicstyle=\scriptsize\ttfamily, 
    basewidth=0.5em,
    lineskip=-1pt,
    keywordstyle=\color{blue}\bfseries,
    caption={The behavior tree MontiCore grammar defining the textual concrete syntax for behavior trees, nodes, services, and decorators.}, 
    label={lst:bt_revised}
]
 grammar BehaviorTree extends  de.monticore.types.MCBasicTypes {
      symbol BehaviorTree = "BehaviorTree" name:Name "{" "Root" root:FlowNode "}";
      // abstract node class
      abstract symbol BTNode = name:Name (Decorator | Service)*;   
      // flow nodes and action nodes
      abstract  FlowNode extends BTNode = name:Name  (Service | Decorator | children:BTNode)* ;
      Sequence extends FlowNode = "Sequence" name:Name "{" (Service | Decorator | children:BTNode)* "}";
      Parallel extends FlowNode = "Parallel" name:Name "{" (Service | Decorator | children:BTNode)* "}";
      abstract  ActionNode extends BTNode = "Action" name:Name ("{" (Decorator | Service)* "}")? ;
      PickUp extends ActionNode = "Action" name:Name ("{" (Decorator | Service)* "}")?;
      Place extends ActionNode = "Action" name:Name ("{" (Decorator | Service)* "}")?; 
      // services and decorators
      abstract Service = "Service" name:Name;
      abstract Decorator = "Decorator" name:Name;
      GripEmpty extends Decorator = "Decorator" name:Name;
      CalcDist extends Service = "Service" name:Name;
  }
\end{lstlisting}

\begin{lstlisting}[
    language=MontiCore, 
    basicstyle=\scriptsize\ttfamily, 
    basewidth=0.5em,
    lineskip=-1pt,
    keywordstyle=\color{blue}\bfseries,
    morekeywords={BehaviorTree, Root, Sequence, Parallel, Action, Decorator},
    caption={Example model describing the behavior tree from Figure~2 using our textual concrete syntax.}, 
    label={lst:robot_tree_instance}
]
BehaviorTree MyRobotTree {
  Root Sequence MainSequence {    
    Sequence HandlePlate {
      Action PickUpPlate1 
      Action PlacePlate1
    }   
    Parallel BeamPlacement {
      Sequence HandleBeam1 {
        Action PickUpBeam1 {
          Decorator GripEmpty
          }
        Action PlaceBeam1
      }
      Sequence HandleBeam2 {
        Action PickUpBeam2 
        Action PlaceBeam2 
            }
        }
    }
}
\end{lstlisting}

\end{nolinenumbers}

\begin{figure*}[h]
    \centering
    \includegraphics[width= 0.7\linewidth]{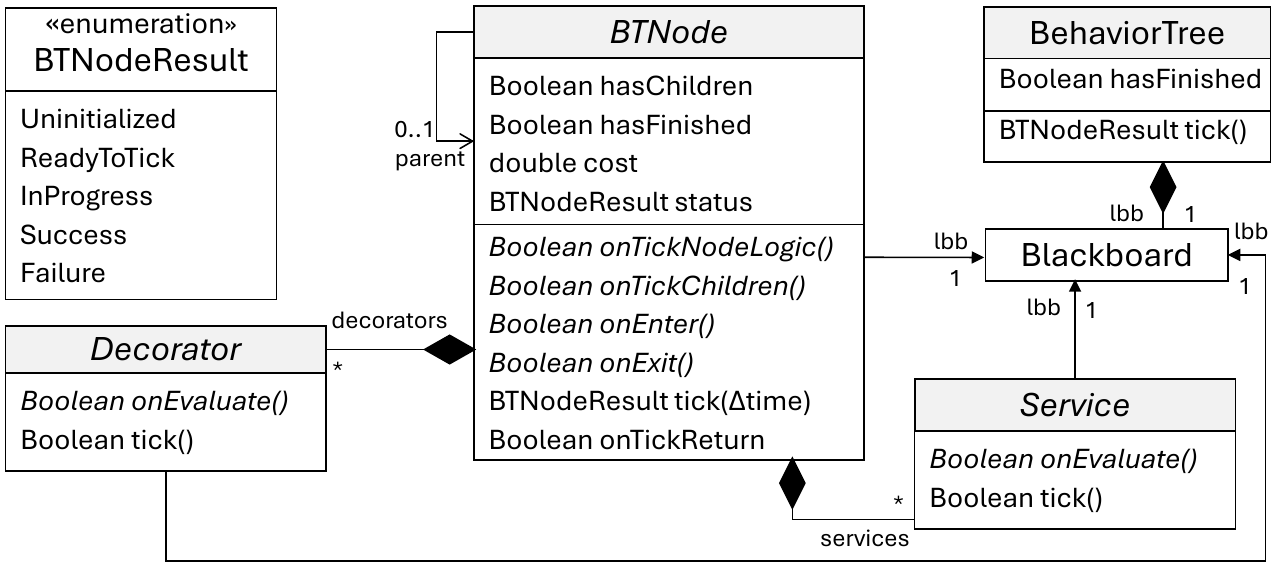}
\caption{Behavior model defining the runtime semantics of the BT metamodel. Gray boxes correspond to metamodel classes. \texttt{lbb} abbreviates \texttt{linkedBlackboard} to reduce diagram clutter.}
    \label{fig:exe-bt}
\end{figure*}

\textbf{Behavior:} Figure~\ref{fig:exe-bt} illustrates the execution model of BTs by extending the BT metamodel (Figure~\ref{fig:mm-bt}) with runtime concepts. Classes highlighted in gray correspond to the abstract syntax elements of the BT metamodel. 

When a BT node’s \texttt{tick} operation is invoked, it first ticks its services, then its decorators, and finally executes its node-specific logic through \texttt{onEnter}, \texttt{onTickNodeLogic}, and \texttt{onExit} operations (Algorithm~\ref{alg:tick_function}). The \texttt{onEnter} operation defines behavior that must occur before the main logic of the node, while \texttt{onExit} specifies behavior after the node finishes execution. If the node is a flow node with children, the tick is propagated to them using the \texttt{onTickChildren} operation. For action nodes, \texttt{onTickChildren} always returns true or delegates execution to another subtree, while for flow nodes, it defines how children are iterated. For example, a Sequence flow node ticks children sequentially from left to right, whereas a Parallel node ticks all children at the same time.

Regarding the \texttt{onTickNodeLogic}, action nodes implement their specific actions (e.g., a robot picking up a plate), while flow nodes evaluate their execution status based on the last status of their children (e.g., a Sequence node returns success only if all children succeed and returns failure as soon as one child fails). The \texttt{onTickReturn} operation handles status updates, post-processing decorators, and invokes \texttt{onExit} operation (Algorithm~\ref{alg:ontickreturn}).

Services and decorators implement a simpler tick logic (Algorithm~\ref{alg:tick_function-serv}). Their tick operation only invokes \texttt{onEvaluate} operation specific to each service or decorator type (e.g., \texttt{GripEmpty} checks the state of the gripper).

As described in Section~\ref{backg}, each time a node is ticked, it returns its last status (\texttt{BTNodeResult}). In addition to the standard statuses \emph{Success}, \emph{Failure}, and \emph{InProgress}, we include \emph{Uninitialized} and \emph{ReadyToTick}, which represent nodes that have never been ticked and nodes that are allowed to receive a tick, respectively.

The BT models used in this paper use a single shared database (\texttt{Blackboard}), which is accessible to all nodes, decorators, and services. The blackboard is modeled as a key–value data structure.
The well-formedness rules of the behavior model specify that there must be exactly one blackboard shared across the entire BT, and only nodes with the \emph{ReadyToTick} status may receive a tick. The \texttt{hasFinished} flag in \texttt{BTNode} only turns true if the node's last status is success or failure.

\begin{figure}[h]
\centering

\begin{minipage}[t]{0.48\textwidth}
\vspace{0pt} 
\begin{algorithm}[H]
\footnotesize
\DontPrintSemicolon
\caption{ operation tick() of BTNode}
\label{alg:tick_function}
\KwIn{$\Delta$ time}
\KwOut{BTNodeResult status}

\If{$status = \text{Uninitialized}$}{
    $status \gets \text{ReadyToTick}$\;
}

\tcp{1. Ticking services}
\ForEach{$svc \in node.services$}{
    \If{$\neg svc.\text{tick}()$}{
        $status \gets \text{Failure}$\;
        \Return{$\text{onTickReturn}(status)$}\;
    }
}

\tcp{2. Ticking decorators}
\ForEach{$dec \in node.decorators$}{
    \If{$\neg dec.\text{tick}()$}{
        $status \gets \text{Failure}$\;
        \Return{$\text{onTickReturn}(status)$}\;
    }
}

\tcp{Check if node finished}
\If{$\text{hasFinished}$}{
    \Return{$\text{onTickReturn}(status)$}\;
}

\tcp{3. Readiness check}
\If{$status = \text{ReadyToTick}$}{
    $\text{onEnter}()$\;
    \If{$\text{hasFinished}$}{
        \Return{$\text{onTickReturn}(status)$}\;
    }
}

\tcp{4. Node execution logic}
\If{$\neg \text{onTickNodeLogic}()$}{
    \Return{$\text{onTickReturn}(status)$}\;
}

\If{$node.hasChildren$}{
    \If{$\neg \text{onTickChildren}()$}{
        \Return{$\text{onTickReturn}(status)$}\;
    }
}

\Return{$\text{onTickReturn}(status)$}\;
\end{algorithm}
\end{minipage}
\hfill
\begin{minipage}[t]{0.48\textwidth}
\vspace{0pt}

\begin{algorithm}[H]
\footnotesize
\DontPrintSemicolon
\SetKwFunction{FExit}{onExit}
\caption{operation onTickReturn() of BTNode}
\label{alg:ontickreturn}
\KwIn{BTNodeResult $status$}
\KwOut{BTNodeResult $finalResult$}

$finalResult \gets status$\;

\tcp{Extension Point: Post-processing decorators}
\tcp*[h]{Implement decorator-based result modification here}\;

\If{$finalResult = \text{Success} \lor finalResult = \text{Failure}$}{
    $\text{hasFinished} \gets \text{true}$\;
    \FExit{}\;
}

\Return{$finalResult$}\;
\end{algorithm}
\vspace{1em} 
\begin{algorithm}[H]
\footnotesize
\DontPrintSemicolon
\SetKwFunction{FEval}{onEvaluate}
\caption{operation tick() of Decorator and Service }
\label{alg:tick_function-serv}
\KwOut{boolean $result$}

\tcp{Execute evaluation logic with time delta}
$result \gets$ \FEval{}\;

\Return{$result$}\;
\end{algorithm}
\end{minipage}
\label{fig:bt_logic_main}
\end{figure}

\subsection{Domain Types DSL Packages} \label{CRF}

\begin{figure}[h]
    \centering
    \includegraphics[width=\linewidth]{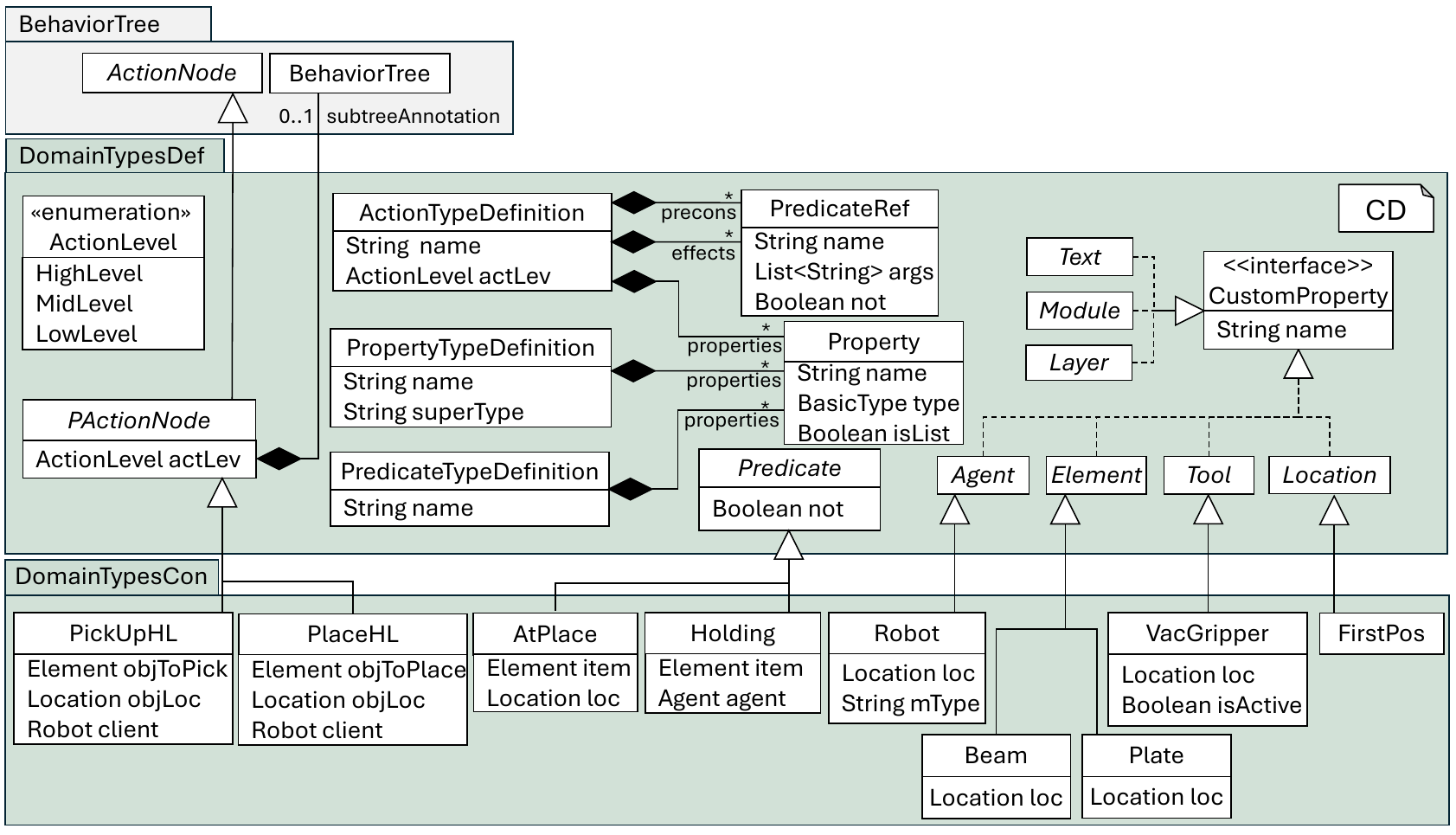}
    \caption{DomainTypesDef metamodel highlighting concepts for defining new property, predicate, and action types. The DomainTypesCon metamodel extends DomainTypesDef and is automatically generated from user-defined DomainTypesDef models. \texttt{BasicType} includes types such as String, Boolean, Integer, etc.}
    \label{fig:mm-types}
\end{figure}

The \texttt{DomainTypesDef} and \texttt{DomainTypesCon} metamodels (Figure~\ref{fig:mm-types}) support the explicit modeling of tasks within a task plan by representing task properties, planning preconditions, and effects. This representation enables design-time validation of task interfaces and runtime logical validation of task execution across dynamically composed task plans. To support different robotic fabrication domains, where robots, equipment, and fabrication setups vary considerably between applications, the metamodels define and instantiate extensible domain-specific types. Specifically, \texttt{DomainTypesDef} specifies the abstract type structure and type definitions, while \texttt{DomainTypesCon} instantiates these definitions as concrete domain types. This separation keeps the core DSL domain-independent while enabling domain-specific grammars to be generated automatically for different robotic fabrication applications. 

\textbf{Abstract syntax}: Figure~\ref{fig:mm-types} presents the metamodel of the domain types definition (\texttt{DomainTypesDef}) grammar, which extends the \texttt{BehaviorTree} grammar. The domain types definition grammar introduces the following concepts:

\begin{itemize}
    \item High-level supertypes (\texttt{CustomProperty}). High-level supertypes distinguish semantically different planning entities (e.g., agents, elements, and locations), enabling decisions on planning or replanning to depend on their roles rather than treating all task properties uniformly. For example, the unavailability of an agent to execute a task does not necessarily invalidate the current plan or require replanning, but may instead only require reassigning the task to another agent.
  
    \begin{itemize}
        \item \texttt{Element:} Objects within the environment (what), e.g., plates or beams.
        \item \texttt{Agent:} Actors performing tasks (who), e.g., robots or humans.
        \item \texttt{Location:} Spatial locations (where), e.g., stacking positions.
        \item \texttt{Tool:} Tools required to execute actions (how), e.g., robot gripper.
        \item \texttt{Module, Layer:} Organizational structures (how it is organized), e.g., groups of beams arranged in layers.
        \item \texttt{Text:} Textual types such as messages.
    \end{itemize}

    \item Type definitions (\texttt{PropertyTypeDefinition}, \texttt{PredicateTypeDefinition}, \texttt{ActionTypeDefinition}). Allow defining new domain types and extending the \texttt{DomainTypesCon} metamodel, while enforcing adherence to the grammar structure. For example, one can define new predicate types or new action types with their properties.

    \item Properties (\texttt{Property}). Allows the addition of planning-relevant properties for new action and predicate types. 

    \item Predicate references (\texttt{PredicateRef}). In symbolic task planning, actions have preconditions and effects as a list of predicates for each. Predicate references allow one to define these predicates using action properties as their variables (see lines 24 to 30 in Listing \ref{lst:type_model}). Modeling predicates separately from action properties enables validation of planning assumptions independently of task properties during composed plan execution. 

    \item Action levels (\texttt{ActionLevel}). The APTree method supports hierarchical action decomposition, where different levels of hierarchy represent varying degrees of action granularity and determinism. The three action levels separate design intent, adaptive planning, and execution control. This hierarchy localizes uncertainty to the ML level while preserving deterministic design constraints and bounded execution complexity.

    \begin{itemize}
        \item (\texttt{High-level or HL}). Deterministic actions that represent abstract design decisions and remain unchanged despite uncertainty. For example, \texttt{PlaceHL(beam1, pr2)} specifies placing \texttt{beam1} at its final location according to the planned design. A high-level action can reference one subtree for more granular mid-level planning (\texttt{PActionNode.subtreeAnnotation}).
        \item (\texttt{Mid-level or ML}). Non-deterministic actions that are more granular. Mid-level actions may require replanning and adjustments based on the changing state of the world.  For example, to place \texttt{beam1} at \texttt{pr2}, the robot might or might not need to travel to \texttt{pr2} depending on its current location. A mid-level action can reference one subtree for more granular low-level execution.
        \item (\texttt{Low-level or LL} ). Fine-grained non-deterministic actions that implement detailed execution. For example, a low-level action sequence stemming from \texttt{PickUpML} could include:
        
        \texttt{MoveEndEffectorToPosLL --> OpenGripperLL --> MoveEndEffectorBackLL}.
    \end{itemize}

    \item Planning actions (\texttt{PAction}). Allows the definition of action nodes compatible with APTree and the integration of symbolic task planners. These actions include:
    \begin{itemize}
        \item Planning level (HL, ML, LL) to distinguish between the level of granularity (\texttt{ActionLevel)}. 
        \item Optional references to subtrees to refine action granularity. For example, \texttt{PickUpHL} may reference a subtree that defines ML actions required to pick up an object.
        \item Preconditions, effects, and properties.
    \end{itemize}
\end{itemize}

The \texttt{DomainTypesCon} metamodel (Listing~\ref{lst:crfcongrammar}) is automatically generated as a grammar from the model defined using \texttt{DomainTypesDef}. For example, the types used in our running example are defined in the \texttt{DomainTypesCon} metamodel shown in Figure~\ref{fig:mm-types} and are automatically derived from the model in Listing~\ref{lst:type_model}.  
The following \textbf{structural well-formedness rules} must be satisfied when using these grammars:
\begin{itemize}

    \item The name of the supertype from which new properties are inherited (\texttt{PropertyTypeDef.superType}) must be one of the predefined supertypes: \texttt{Agent}, \texttt{Element}, \texttt{Location}, \texttt{Tool}, \texttt{Layer}, \texttt{Module}, or \texttt{Text} (see lines 2-6 in Listing \ref{lst:type_model} for valid example models defined to describe the types in our running example).

    \item Actions may only use predicate types for their preconditions and effects that are defined (e.g. in Listing \ref{lst:type_model}, action \texttt{PlaceHL} can have a predicate \texttt{Holding}, but not \texttt{Clear}).

    \item Actions may only use properties in their preconditions and effects that are defined for that action. Introducing new properties in predicates is not allowed (e.g., the predicate \texttt{Holding} in the action \texttt{PlaceHL} in Listing~\ref{lst:type_model} can use only \texttt{objToPlace}, \texttt{objLoc}, and \texttt{client} as its variables).   

    \item Subtree references must respect the action hierarchy: Low-level actions cannot reference subtrees; Mid-level actions may reference Low-level subtrees; High-level actions may reference Mid-level subtrees.

\end{itemize}

\textbf{Concrete syntax:} 
Listing~\ref{lst:crf_grammar} presents the concrete syntax for defining new property, predicate, and action types using the \texttt{DomainTypesDef} grammar. Listing~\ref{lst:crfcongrammar} shows the concrete syntax for defining instances of these types using the \texttt{DomainTypesCon} grammar, which is automatically generated from a \texttt{DomainTypesDef} model (Listing \ref{lst:type_model}).

The \texttt{DomainTypesDef} grammar extends the \texttt{BehaviorTree} grammar (Section \ref{bt}), and \texttt{DomainTypesCon} extends \texttt{DomainTypesDef}. Valid models for each grammar describing the types in our running example are shown in Listings~\ref{lst:type_model} and~\ref{lst:concrete_model}, respectively. 
For action types such as \texttt{PickUpHL} and \texttt{PlaceHL}, a small piece of Java code (e.g., lines 14 and 17)  is used to automatically set the value of the action level (\texttt{PAction.actLevel}).

\begin{nolinenumbers}

\begin{lstlisting}[language=MontiCore, caption={The DomainTypesDef MontiCore grammar defining the textual concrete syntax for specifying new types.},
label={lst:crf_grammar}]
grammar DomainTypesDef extends BehaviorTree, de.monticore.types.MCBasicTypes {
//  All the allowed custom super types
    interface CustomProperty = name:Name; 
    abstract symbol Layer    implements CustomProperty = name:Name;
    abstract symbol Module   implements CustomProperty = name:Name;
    abstract symbol Text     implements CustomProperty = name:Name; 
    abstract symbol Location implements CustomProperty = name:Name;
    abstract symbol Element  implements CustomProperty = name:Name;
    abstract symbol Tool     implements CustomProperty = name:Name;
    abstract symbol Agent    implements CustomProperty = name:Name;
// property definition - accepts any Name as type 
    Property = name:Name ":" type:BasicType isList:["+"]?;
    PropertyTypeDefinition = superType:Name  name:Name  ("{" (Property)* "}")?;
// predicate definition
    abstract Predicate = not:["!"]?;
    PredicateTypeDefinition = "Predicate" name:Name  "{" (Property)* "}";
// Action definition
    abstract symbol PActionNode extends ActionNode = 
        name:Name 
        actLevel: ActionLevel
        ("@" subtreeAnnotation:Name@BehaviorTree)? 
        ("{" (Decorator | Service)* "}")?
        ;
    enum ActionLevel = "HighLevel" | "MidLevel" | "LowLevel";
    ActionTypeDefinition = 
    "Define" "Action" name:Name "{"
        "ActLevel:" actLevel:ActionLevel
        "Properties" "{" Property* "}"
        "Preconditions" "{" precons:PredicateRef* "}"
        "Effects" "{" effects:PredicateRef* "}"
        "}";
// predicate reference 
    PredicateRef = not:["!"]? name:Name  "(" (args:Name ("," args:Name)*)? ")";
}
\end{lstlisting}

\begin{lstlisting}[language=MontiCore, 
caption={The DomainTypesCon grammar extending DomainTypesDef and including types automatically generated from DomainTypesDef models.},
  numbers=left,
  numberstyle=\tiny,
  numbersep=6pt,
   morekeywords={key, astrule, return},
label={lst:crfcongrammar}]
grammar DomainTypesCon extends DomainTypesDef {

// GENERATED PROPERTY RULES 
symbol Beam       extends Element  = "Beam" name:Name "(" loc:Name@Location ")";
symbol Plate      extends Element  = "Plate" name:Name "(" loc:Name@Location ")";
symbol Robot      extends Agent    = "Robot" name:Name "(" loc:Name@Location mType:Name ")";
symbol VacGripper extends Tool   = "VacGripper" name:Name "(" loc:Name@Location isActive:Boolean ")";
symbol FirstPos   extends Location = "FirstPos" name:Name "(" ")";
// GENERATED PREDICATE RULES 
Holding extends Predicate = not:["!"]? "Holding" "(" item:Name@Element agent:Name@Agent ")";
AtPlace extends Predicate = not:["!"]? "AtPlace" "(" item:Name@Element loc:Name@Location ")";
// GENERATED ACTION RULES 
PickUpHL extends PActionNode = "Action" "PickUpHL" name:Name "(" objToPick:Name@Element objLoc:Name@Location client:Name@Robot ")" ("{" (Decorator | Service)* "}")? ("@" subtreeAnnotation:Name)? ;
astrule PickUpHL = method public domaintypesdef._ast.ASTActionLevel getActLevel() { if(actLevel==null){ actLevel=domaintypesdef._ast.ASTActionLevel.HIGHLEVEL; } return actLevel; } ;

PlaceHL extends PActionNode = "Action" "PlaceHL" name:Name "(" objToPlace:Name@Element objLoc:Name@Location client:Name@Robot ")" ("{" (Decorator | Service)* "}")? ("@" subtreeAnnotation:Name)? ;
astrule PlaceHL = method public domaintypesdef._ast.ASTActionLevel getActLevel() { if(actLevel==null){ actLevel=domaintypesdef._ast.ASTActionLevel.HIGHLEVEL; } return actLevel; } ;
}
\end{lstlisting}

\begin{figure*}[h] 
\centering
\begin{minipage}[t]{0.48\textwidth}
\begin{lstlisting}[
    language=MontiCore, 
    caption={Example model of DomainTypesDef grammar describing the types used in our behavior tree running example from Figure 2}, 
    morekeywords={Define, as, Element, Location, Agent, Tool, boolean, Predicate, Action, ActLevel, Properties, String, Boolean, Preconditions, Effects, !},
    label={lst:type_model},
    basicstyle=\scriptsize\ttfamily
]
// property type model
Element  Beam {loc:Location}
Element  Plate {loc:Location}
Location FirstPos
Agent    Robot {loc:Location mType:String}
Tool     VacGripper {loc:Location isActive:Boolean}
// predicate type model
Predicate Holding {
  item : Element
  agent : Agent
}
Predicate AtPlace {
  item : Element
  loc : Location
} 
// action type model
Define Action PlaceHL {
    ActLevel: HighLevel
    Properties {
        objToPlace: Element
        objLoc: Location
        client: Robot
    }
    Preconditions {
        Holding(objToPlace, client)
    }    
    Effects {
        AtPlace(objToPlace, objLoc)
        !Holding(objToPlace, client) 
    }   
}
\end{lstlisting}
\end{minipage}
\hfill
\begin{minipage}[t]{0.48\textwidth}
\begin{lstlisting}[
    language=MontiCore, 
    caption={Example model of DomainTypesCon grammar describing the instances used in our behavior tree running example from Figure 2}, 
      morekeywords={Plate, Beam, Robot, FirstPos, VacGripper, Action, Root, AtPlace, BehaviorTree, Sequence, Holding, PickUpHL, PlaceHL},
    label={lst:concrete_model},
    basicstyle=\scriptsize\ttfamily
]
// property model
Plate      plate1 (fp1)
Beam       beam1 (fp2)
Beam       beam2 (fp3)
Robot      r1 (rp1 ur10)
FirstPos   fp1 ()
FirstPos   fp2 ()
FirstPos   rp1 ()
VacGripper vg1 (fp1 True)
// predicate model
AtPlace(beam2 fp1)
Holding (beam1 r1)
// actions in BT model
BehaviorTree MyRobotTree {
  Root Sequence MainSequence {    
    Sequence HandlePlate {
      Action PickUpHL PickUpPlate1 (plate1 fp1 r1)  @PickUpMLSubtree    
      Action PlaceHL PlacePlate1 (plate1 fp2 r1)    @PlaceMLSubtree
        }
    }
}
\end{lstlisting}
\end{minipage}

\end{figure*}

\begin{figure}[h]
    \centering
    \includegraphics[width=0.7\linewidth]{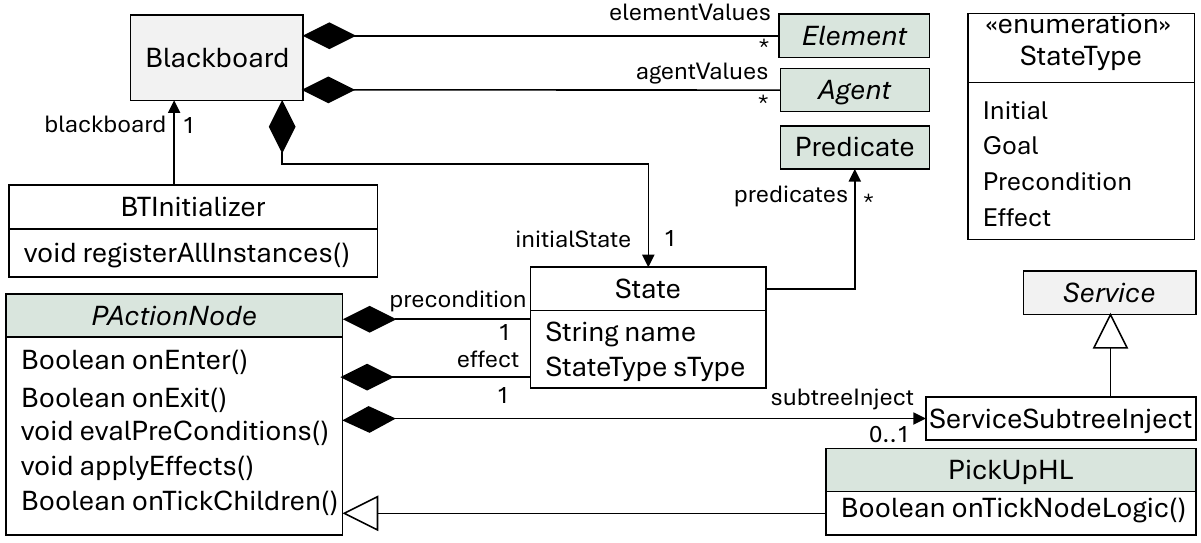}
    \caption{Behavior model extending the metamodels defined in this section to describe runtime behavior. It highlights concepts such as blackboard registration. Green boxes correspond to the \texttt{DomainTypesDef} or \texttt{DomainTypesCon} metamodel elements and gray boxes to the \texttt{BehaviorTree} metamodel elements. Custom types other than \texttt{Element} and \texttt{Agent} have similar blackboard relations but are omitted for clarity.}
    \label{fig:execrftypes}
\end{figure}
\end{nolinenumbers}

 \textbf{Behavior}: Figure~\ref{fig:execrftypes} illustrates the behavior model of the domain type metamodels. Classes highlighted in green correspond to classes of the \texttt{DomainTypesDef} and \texttt{DomainTypesCon} metamodels, while gray classes correspond to the \texttt{BehaviorTree} metamodel. 
A \texttt{State} has a list of predicates describing a world state and can be of four kinds: an initial state or goal state, which define the planning problem, or a precondition or effect, which define the conditions and outcomes of an action. Accordingly, each action type has two associated states: preconditions and effects. The blackboard maintains a global \texttt{initialState}, which holds a set of predicates representing the current world state, and acting as the single source of truth for planning modules. 

The execution semantics extend the standard BT behavior model (Figure~\ref{fig:exe-bt}) by implementing the \texttt{onEnter} and \texttt{onExit} operations for planning actions. During \texttt{onEnter}, the action's preconditions are checked (\texttt{evalPreConditions}) by querying the initial state from the blackboard. Predicate evaluation follows PDDL-style semantics: a predicate holds if it exists in the initial state, and does not hold otherwise.
After the node-specific logic is executed successfully with the \texttt{onTickNodeLogic} operation, the \texttt{onExit} operation applies the action’s effects by updating the blackboard through an \texttt{applyEffects} operation, which adds or removes predicates from the current state.

If an action references a refinement subtree, the subtree is ticked within the \texttt{onTickNodeLogic} operation because the higher-level action's effects must only be applied after the refinement subtree has been successfully executed. This enables hierarchical action decomposition across planning levels.
 It is not necessary for every task plan to include all levels of granularity. The subtree injection service (\texttt{ServiceSubtreeInject}) generates subtrees for higher-level actions. The effects of the higher-level action are used as planning goals to preserve the logical consistency of the task plan. The initial state is retrieved from the blackboard to automatically generate the problem file for the lower-level flow node.

The execution can be divided into two phases. The first phase performs input validation and instance generation using MontiCore and its generated ASTs. The second phase handles BT execution, which is implemented separately in a C\#-based execution engine. 
Algorithm~\ref{typeparsing} describes the first phase, which parses the user-defined type models. Property types are parsed first, followed by predicate and action types. After each successful parsing step, new types are written on the \texttt{DomainTypesCon} grammar and the grammar is recompiled (\texttt{generateMCGrammars} in Algorithm \ref{typeparsing}) to ensure that subsequent models are validated against the updated type definitions. In addition, C\# code is generated for the newly introduced types to support execution in the BT engine.

The same incremental process is applied when validating the initial BT model input, as shown in Algorithm~\ref{alg:typeparsing}. First, the property model is parsed to declare the scene objects (Listing~\ref{lst:concrete_model}, lines~1–8), and its symbols are added to MontiCore’s symbol table (\texttt{APTreeParser.parseAndGenSyms}). Next, predicates defining the initial state of the full planning problem are validated (Listing~\ref{lst:concrete_model}, lines~11–12). Finally, the initial input BT model is parsed using the fully constructed symbol table for properties and predicates (Listing~\ref{lst:concrete_model} lines~16–21). The initial BT model might or might not have any actions as planners can provide actions at runtime. This process ensures that all symbols referenced in the BT are validated at design-time. For example, if the BT references a symbol such as \texttt{beam3} that is not defined in the property model, a design-time error is raised. To ensure consistency between the types registered in the BT model and the execution engine, Algorithm~\ref{alg:typeparsing} also invokes the \texttt{writeInstancesAsJson} operation, which serializes all generated instances from the input models into JSON files for the BT execution engine.
Finally, the \texttt{BTInitializer} in Figure \ref{fig:execrftypes} reads these JSON files, instantiates the corresponding instances using factories, and registers them on the blackboard via the \texttt{registerAllInstances} operation. After the BT starts running, the BT model is validated each time a planner generates a new runtime plan for insertion into the tree.

\begin{algorithm}[h]
\DontPrintSemicolon
\footnotesize
\SetKwFunction{FParse}{parse}
\SetKwFunction{FWrite}{writeRulesToGrammar}
\SetKwFunction{FGen}{generateMCGrammars}
\SetKwFunction{FGenCS}{generateTypeInCSharp}

\caption{Domain model type parsing and grammar generation using MontiCore infrastructure}
\label{typeparsing}
\KwIn{Paths: $propPath$, $predPath$, $actPath$}
\KwOut{Boolean success or Error}

\tcp{1. Property Type Parsing}
\If{$\neg \text{DomainPropertyTypeParser}.\FParse{propPath}$}{
    \Return{\text{Error}}\;
}
\text{DomainPropertyTypeParser}.\FWrite{}\;
\text{DomainPropertyTypeParser}.\FGenCS{}\;
\FGen{}\;

\tcp{2. Predicate Type Parsing}
\If{$\neg \text{DomainPredicateTypeParser}.\FParse{predPath}$}{
    \Return{\text{Error}}\;
}
\text{DomainPredicateTypeParser}.\FWrite{}\;
\text{DomainPredicateTypeParser}.\FGenCS{}\;
\FGen{}\;

\tcp{3. Action Type Parsing}
\If{$\neg \text{DomainActionTypeParser}.\FParse{actPath}$}{
    \Return{\text{Error}}\;
}
\text{DomainActionTypeParser}.\FWrite{}\;
\text{DomainActionTypeParser}.\FGenCS{}\;
\FGen{}\;

\Return{\text{Success}}\;
\end{algorithm}

\begin{algorithm}[h]
\DontPrintSemicolon
\footnotesize
\SetKwFunction{FParseSym}{parseAndGenSyms}
\SetKwFunction{FSetBB}{writeInstancesAsJson}
\SetKwFunction{FSetAllBB}{writeInstancesAsJson}
\SetKwFunction{FWriteBT}{writeBTStructureAsJson}

\caption{Validation and symbol generation for scene items, initial state, and BT models}
\label{alg:typeparsing}
\KwIn{Paths: $propPath$, $predPath$, $btPath$}
\text{$propPath, predPath, btPath \neq \emptyset$ }
\KwOut{Boolean success or Error}

\tcp{1. Property Model Validation}
\If{$\neg \text{ConcretePropertyParser}.\FParseSym{propPath}$}{
    \Return{\text{Error}}\;
}
\FSetBB{properties}\;

\tcp{2. Predicate Model Validation}
\If{$\neg \text{ConcretePredicateParser}.\FParseSym{predPath}$}{
    \Return{\text{Error}}\;
}
\FSetAllBB{predicates}\;

\tcp{3. BT Structure Validation}
\If{$\neg \text{APTreeParser}.\FParseSym{btPath}$}{
    \Return{\text{Error}}\;
}

\tcp{4. Initial BT Structure Serialization}
\FWriteBT{}\;

\Return{\text{Success}}\;
\end{algorithm}

\subsection{Planning Service DSL Package} \label{plansrv}

\begin{figure}[h]   
\centering
    \includegraphics[width= 0.4\linewidth]{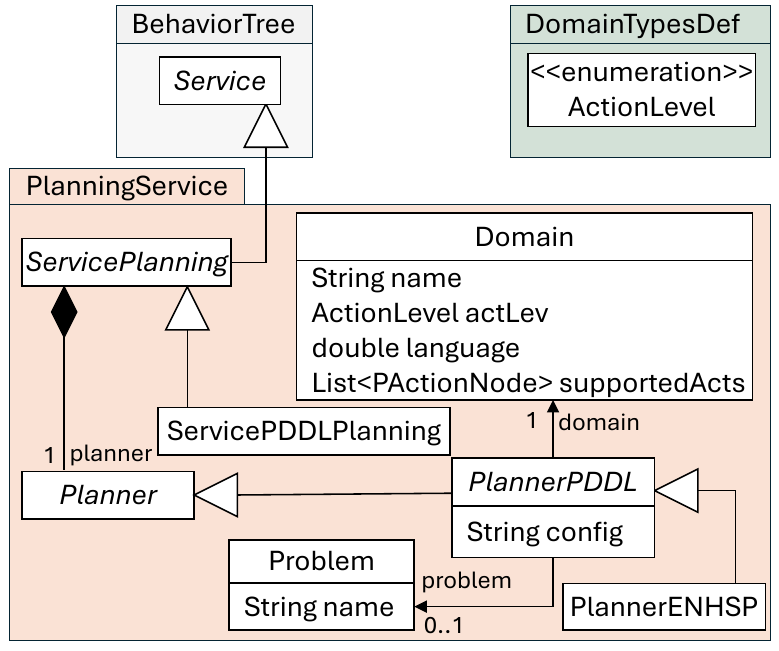}   
   \caption{Planning service metamodel defining automated planning as a BT service, highlighting concepts such as planning services, planners, and PDDL-based planners.}

    \label{fig:mm-plansrv}
\end{figure}

Figure~\ref{fig:mm-plansrv} presents the planning service metamodel (\texttt{ServicePlanning}), which provides a formal representation of symbolic planners and their configurations, enabling their specification, validation, and invocation by the execution model. Flow nodes reference these planner services to generate task plans as BT nodes. 

\textbf{Abstract syntax:} The planning service metamodel extends the BT service metamodel with the following concepts:
\begin{itemize}
   \item Planning services (\texttt{ServicePlanning}, \texttt{ServicePDDLPlanning}). Specify planning services for BT node generation by providing a planner name; \texttt{ServicePDDLPlanning} invokes a PDDL-based symbolic planner.
    
   \item Planners (\texttt{Planner}). Define planner tools by type and configuration, enabling multiple planners to be invoked through a common execution model.
    
   \item PDDL planners (\texttt{PlannerPDDL}). Model PDDL-based symbolic planners with a planning domain (\texttt{Domain}), problem file (\texttt{Problem}), and domain action level, ensuring compatibility between planning domains and action granularity (e.g., a domain containing mid-level actions such as \texttt{Travel} is not used for high-level planning).
\end{itemize}

  Structural well-formedness rules for this metamodel include the following: Only flow nodes may be connected to planning services; action nodes are not permitted. The action level of a planning domain must match the hierarchy level of the actions it generates. The planner used for the PDDL planning service must be of type \texttt{PlannerPDDL}.

\textbf{Concrete syntax}: The concrete syntax is shown in Listing \ref{lst:planner} with examples of valid models for planner definitions (lines 1-2 in Listing \ref{lst:model planner}) and for the use of planners in a part of our running example (lines 4-10 in Listing \ref{lst:model planner}). Planner tools other than ENHSP can similarly be added to the grammar.
\begin{nolinenumbers}
    
\needspace{12\baselineskip}

\begin{lstlisting}[language=MontiCore, 
caption={The planning service MontiCore grammar defining the textual concrete syntax for planning services and for planners.}, 
label={lst:planner}]
grammar PlanningService extends DomainTypesCon {
//planning service
abstract ServicePlanning  extends Service = name:Name planner:Planner;
ServicePDDLPlanning extends ServicePlanning = "ServicePlanning"  name:Name planner:Planner;
// Planners
abstract Planner ;
abstract PlannerPDDL extends Planner =  "Domain" ":" domain:Name@Domain   ("Problem" ":" problem:Name@Problem)? ("Config" ":" config:String)?;
PlannerENHSP extends PlannerPDDL = "Enhsp"  "Domain" ":" domain:Name@Domain   ("Problem" ":" problem:Name@Problem)? ("Config" ":" config:String)?;
symbol Domain = "Domain" name:Name  "LanguageVersion" ":" language:Double "Level" ":" actLev:ActionLevel "{" supportedActs:Name@PActionNode ("," supportedActs:Name@PActionNode)* "}";
symbol Problem = "Problem" name:Name;
}

\end{lstlisting}

\begin{lstlisting}[language=MontiCore,  morekeywords={BehaviorTree, Root, FlowNode, Nodegraph, Action, Service, ServicePlanning,Sequence, Domain, Problem, PddlPlanner, Enhsp}, 
caption={Example model showing the definition of planning service for the \texttt{HandlePlate} sequence node in our running example (lines 5 to 10) }, 
label={lst:model planner}]
ServicePlanning MyPlanner1 Enhsp Domain:PickAndPlaceHL Problem:PlacePlate1
Domain PickAndPlaceHL LanguageVersion:2.1  Level:HighLevel  {pickUphl, placehl}
Problem PlacePlate1

BehaviorTree MyRobotTree {
  Root Sequence MainSequence {    
    Sequence HandlePlate {
      ServicePlanning MyPlannerSrv1 Enhsp Domain:PickAndPlaceHL Problem:PlacePlate1
        Action PickUpHL PickUpPlate1 (plate1 fp1 r1) @PickUpMLSubtree    
        Action PlaceHL PlacePlate1 (plate1 pr1 r1) @PlaceMLSubtree
    }}}
\end{lstlisting}

\end{nolinenumbers}

\begin{figure}[h]   
\centering
    \includegraphics[width= 0.8\linewidth]{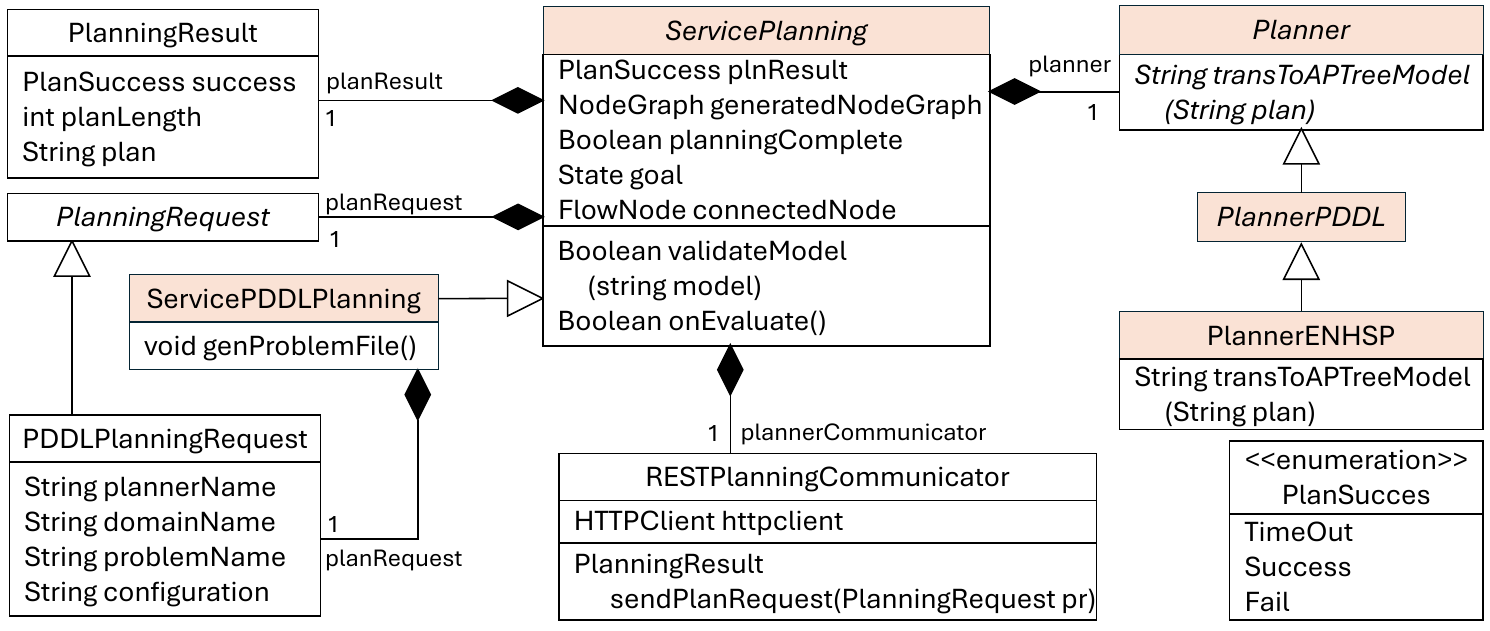}   
   \caption{Behavior model extending the planning service metamodel, highlighting components such as the \texttt{RESTPlanningCommunicator}, which sends planning requests and receives planning results. Classes highlighted in orange correspond to the same classes in the \texttt{PlanningService} metamodel.}
    \label{fig:exe-plansrv}
\end{figure}

\textbf{Behavior}:
Figure~\ref{fig:exe-plansrv} illustrates the behavior model for planning services. Boxes highlighted in orange correspond to the elements defined in the abstract syntax metamodel shown in Figure~\ref{fig:mm-plansrv}.
A PDDL planning service performs three main functions: invoking planners by providing the required information (e.g., domain and problem files for PDDL planners), receiving the resulting task plan, and transforming this plan into APTree DSL for validation.

Planning services interact with external planners through a REST-based communication architecture. As a result, each planning service maintains both a planning request (\texttt{PlanningRequest}) and a planning result (\texttt{PlanningResult}).

The execution logic of the \texttt{onEvaluate} operation of the planning service is as follows: 

\texttt{plannerCommunicator.sendPlanRequest} sends a planning request to the selected planner and, if a plan is found, receives it. If planning succeeds, the resulting plan model is transformed into an APTree model using \texttt{planner.transToAPTreeModel} operation, and then parsed and validated using APTree's full DSL. If successful, the resulting plan is assigned to the attached flow node, which then executes it. If any step fails, the \texttt{onEvaluate} operation returns \texttt{false}.
If a problem file is not provided, or replanning is required for mid-level and low-level actions, the \texttt{genProblemFile} operation automatically generates one. This operation queries the initial state from the blackboard, refines predicates for the planner domain, and sets the higher-level action effects as the goal in the generated problem file. Therefore, each planning service maintains a reference to its connected flow node (\texttt{connectedNode}).

\subsection{Dynamic Flownode DSL Package} \label{FN}

This section introduces the dynamic flow node metamodel (Figure~\ref{fig:mm-fn}) and the corresponding method. The metamodel represents complete task plans, including task nodes and their ordering, as directed graphs. It further separates task plan representation from task execution semantics (e.g., success criteria), allowing task plans to be modified independently of the execution model that orchestrates planning, replanning, and task execution. 
\begin{figure}[h]
    \centering
    \includegraphics[width=0.5\linewidth]{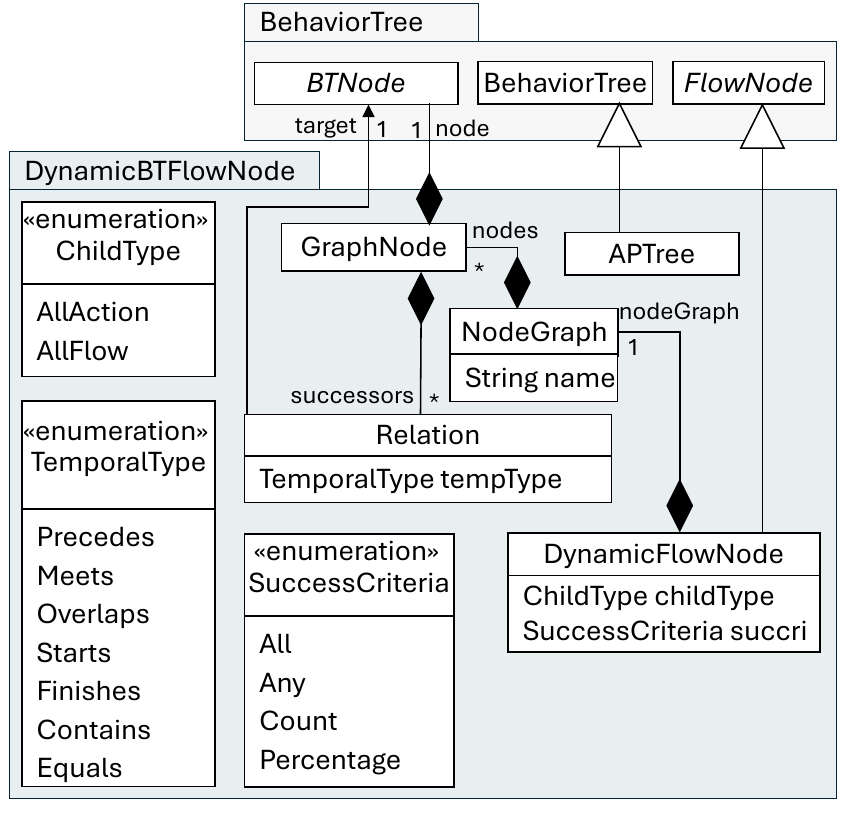}
  \caption{Dynamic flow node metamodel defining flexible and configurable flow nodes, highlighting core concepts such as node graphs, success criteria, and temporal types.}
    \label{fig:mm-fn}
\end{figure}

\textbf{Abstract syntax:}
The dynamic flow metamodel extends the previous metamodels with the following concepts:
\begin{itemize}
\item Dynamic flow node (\texttt{DynamicFlowNode}). Extends BT flow nodes by replacing hardcoded execution logic (e.g., sequence, parallel) with a single configurable logic. A dynamic flow node may have children of only one type (\texttt{AllAction} or \texttt{AllFlow}), thereby clarifying the hierarchy and separating planning from execution in the BT.

\item Success criteria (\texttt{SuccessCriteria}). Allow configuration of when a flow node returns \texttt{SUCCESS} after receiving a tick. Each dynamic flow node must define exactly one success criterion:
\begin{itemize}
    \item \texttt{All}: All child nodes must succeed.
    \item \texttt{Any}: At least one child node must succeed.
    \item \texttt{Count}: A specified number of child nodes must succeed.
    \item \texttt{Percentage}: A specified percentage of child nodes must succeed, enabling partial success thresholds.
\end{itemize}

\item Node graph (\texttt{NodeGraph}). Replaces the linear child list of traditional BTs with a graph of \texttt{GraphNode}s, each wrapping a BT node. Relations between graph nodes define execution order, while the absence of a relation indicates no ordering constraint, enabling partially ordered task plans. For example, in Listing~\ref{lst:bt_model_instance}, there is no relation between \texttt{PickUpBeam1} and \texttt{PickUpBeam2}, avoiding the need for a separate parallel flow node. This separation allows a complete subplan, generated by a planner to be represented and managed by a single \texttt{DynamicFlowNode}, rather than distributing it across multiple BT flow nodes. Consequently, planning, execution, and success management remain encapsulated within a single execution submodule.

\item Temporal relationships (\texttt{TemporalType}). Relations between graph nodes are annotated with a temporal type (\texttt{TemporalType}) defining execution constraints using Allen's interval algebra~\cite{Tomi19}, without replacing the execution logic (flow node). For example, the \texttt{Meets} relation allows a successor node to execute immediately after its predecessor finishes. 

\item APTree. Extends \texttt{BehaviorTree} to support all extended DSL concepts defined in this paper that are not supported by the base BT metamodel.

\end{itemize}
Structural well-formedness rules include the following: An \texttt{APTree} instance accepts only dynamic flow nodes; traditional flow nodes are not permitted. In an \texttt{APTree} model, all dynamic flow nodes that have action nodes must have at least one planning service except for the root node.

\textbf{Concrete syntax:} Listing~\ref{lst:dynamic_bt_grammar} presents the concrete syntax of the dynamic flow nodes, followed by a valid model of the running example BT in Listing~\ref{lst:bt_model_instance}. Since the \texttt{DynamicFlowNode} grammar is the final grammar in the hierarchy, the model in Listing~\ref{lst:bt_model_instance} represents the complete BT in our running example using the APTree DSL.

\begin{nolinenumbers}

\begin{lstlisting}[
    language=MontiCore, 
    caption={The DynamicBTFlowNode MontiCore grammar defining the textual concrete syntax for graph-based dynamic flow nodes.}, 
    label={lst:dynamic_bt_grammar}
]
grammar DynamicBTFlowNode extends PlanningService {
APTree extends BehaviorTree = "BehaviorTree" name:Name "{" "Root" root:FlowNode "}";
    // Dynamic BTflow node definition 
    DynamicFlowNode extends FlowNode = 
      "FlowNode" name:Name "{" 
          (Decorator | Service)*     
          succri:SuccessCriteria
          childType:ChildType
          nodeGraph:NodeGraph 
      "}";
    // enums
    enum SuccessCriteria = "All" | "Any" | "Count" | "Percentage";
    enum TemporalType = "Meets" | "Precedes" | "Overlaps" |"Starts" | "Finishes" | "Contains" | "Equals";
    enum ChildType = "AllAction" | "AllFlow"; 
    // NodeGraph definition 
    NodeGraph = "NodeGraph" "{"
    (nodes:GraphNode)*
    "}";
    GraphNode = node:BTNode ("{"
        (successors:Relation)+
        "}")?;
    Relation = "--[" temptype:TemporalType "]-->" target:Name@BTNode ";";   
}
\end{lstlisting}

\begin{lstlisting}[
    language=MontiCore, 
    basicstyle=\scriptsize\ttfamily, 
    basewidth=0.5em,
    lineskip=-1pt,
    keywordstyle=\color{blue}\bfseries,
    commentstyle=\color{darkgray}\bfseries\itshape,
    morekeywords={BehaviorTree, Root, FlowNode, Nodegraph, Action, Service, NodeGraph, ServicePlanning, PickUpHL, PlaceHL, Domain, Enhsp, Problem}, 
    caption={Example model defining the running example behavior tree using the full DSL.}, 
    label={lst:bt_model_instance}
]
BehaviorTree APTree {
    Root FlowNode Main {
        All
        AllFlow
        NodeGraph {    
            FlowNode  HandlePlate {
                ServicePlanning MyPlannerSrv1 Enhsp Domain:PickAndPlaceHL Problem:PlacePlate1
                All
                AllAction
                NodeGraph {
                    Action PickUpHL PickUpPlate1 (plate1 fp1 r1) @PickUpMLSubtree {
                       --[Meets]--> PlacePlate1;
                    }
                    Action PlaceHL PlacePlate1 (plate1 pr1 r1) 
                }
            } 
            FlowNode BeamPlacement {
                ServicePlanning MyPlannerSrv1 Enhsp Domain:PickAndPlaceHL Problem:handleBeams
                All
                AllAction
                NodeGraph {
                            Action PickUpHL PickUpBeam1 (beam1 fp2 r1) { 
                                --[Meets]--> PlaceBeam1;
                            }
                            Action PlaceHL PlaceBeam1 (beam1 pr2 r1) 
                            Action PickUpHL PickUpBeam2 (beam2 fp3 r1) { 
                                --[Meets]--> PlaceBeam2;
                            }
                            Action PlaceHL PlaceBeam2 (beam2 pr3 r1) 
                }   
            }
        }
    }
} 
\end{lstlisting}

\end{nolinenumbers}
\textbf{Behavior}:
The execution of flow nodes follows the same logic introduced in the BT behavior model shown in Figure~\ref{fig:exe-bt}. It is defined by two main operations: \texttt{onTickNodeLogic} and \texttt{onTickChildren}. The \texttt{onTickNodeLogic} operation evaluates the node’s status using the configured success criteria through the \texttt{checkSuccessCriteria} operation (Figure~\ref{fig:exe-fn}). This evaluation is based on the current \texttt{BTNodeResult} of the node’s children (Figure~\ref{fig:exe-bt}). For example, in the model shown in Listing~\ref{lst:bt_model_instance}, the \texttt{All} success criterion ensures that the flow node \texttt{HandlePlate} returns \texttt{Success} only if both \texttt{PickUpPlate1} and \texttt{PlacePlate1} complete successfully.

\begin{figure}[h]
    \centering
    \includegraphics[width=0.6\linewidth]{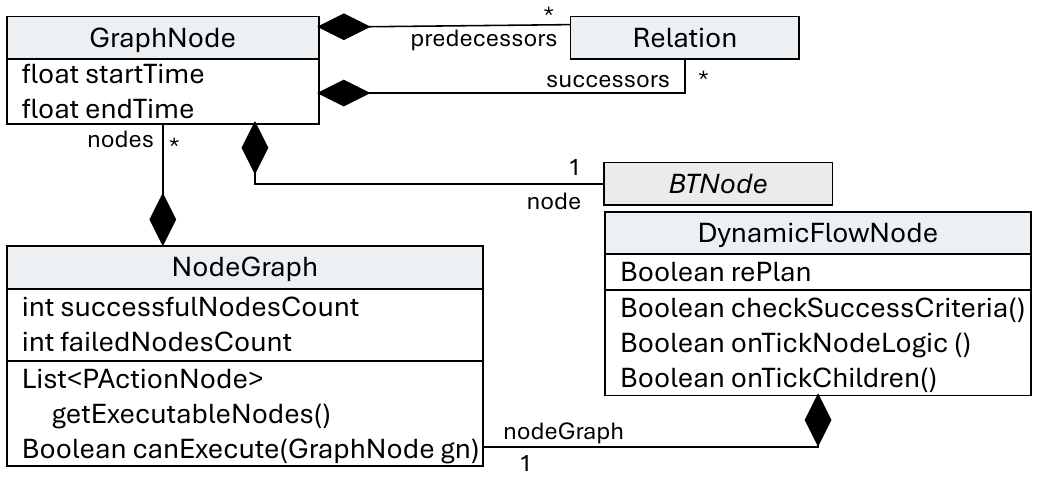}
  \caption{Behavior model extending the \texttt{DynamicFlowNode} metamodel, highlighting operations such as \texttt{getExecutableNodes} and \texttt{canExecute} that determine executable graph nodes. Blue classes correspond to the metamodel (Figure \ref{fig:mm-fn}).}
    \label{fig:exe-fn}
\end{figure}

The \texttt{onTickChildren} operation determines which nodes may execute and when they execute. It uses the \texttt{getExecutableNodes} operation, which iterates over the node graph and evaluates each node using the \texttt{canExecute} operation to identify executable nodes. The \texttt{canExecute} operation excludes nodes that have completed execution, retains nodes that are currently executing, and marks nodes without predecessors as immediately executable. For nodes with predecessors, \texttt{canExecute} evaluates temporal relations and returns \texttt{true} only if all required predecessor conditions are satisfied. For example, under a \texttt{Meets} constraint between \texttt{PickUpPlate1} and \texttt{PlacePlate1} in Listing~\ref{lst:bt_model_instance}, \texttt{PlacePlate1} may execute only after \texttt{PickUpPlate1} has completed.

While success criteria determine when a flow node returns \texttt{Success}, a child failure initially causes the flow node to return \texttt{Failure}. However, each dynamic flow node has a post-processing decorator in its list of decorators that intercepts this failure in the \texttt{onExit} operation (Section~\ref{bt}) and resets the node’s status to \texttt{InProgress}, thereby triggering replanning in the next round of tick signals. A dynamic flow node returns \texttt{Failure} only if the maximum number of replanning attempts is reached and the node remains unsuccessful.

Upon replanning, nodes that have already executed remain in the graph but become inactive, while nodes that have not yet received a tick may be removed. New nodes generated during replanning can be appended to the right side of the node graph as active nodes. Inactive nodes are excluded from the evaluation of flow node success criteria.

\section{Case Study: livMatS Biomimetic Shell Pavilion Prefabrication Process} \label{sec:casestudy}

The APTree method was evaluated through a real-world CRF task planning case study of the livMatS Biomimetic Shell pavilion~\cite{Wagner20}, where we describe the domain in APTree DSL and plan tasks for it. 
The pavilion consists of modular cassettes assembled by an industrial robotic arm. The robotic fabrication process involves picking up cassette elements (plates and beams) from designated staging areas, placing them on a work table, and manipulating them with interchangeable end-effectors~\cite{Wagner20} (Figure \ref{fig:livMatS Biomimetic Shell2}). 

\begin{figure*}[t]
\centering
\includegraphics[width= 0.8\linewidth]{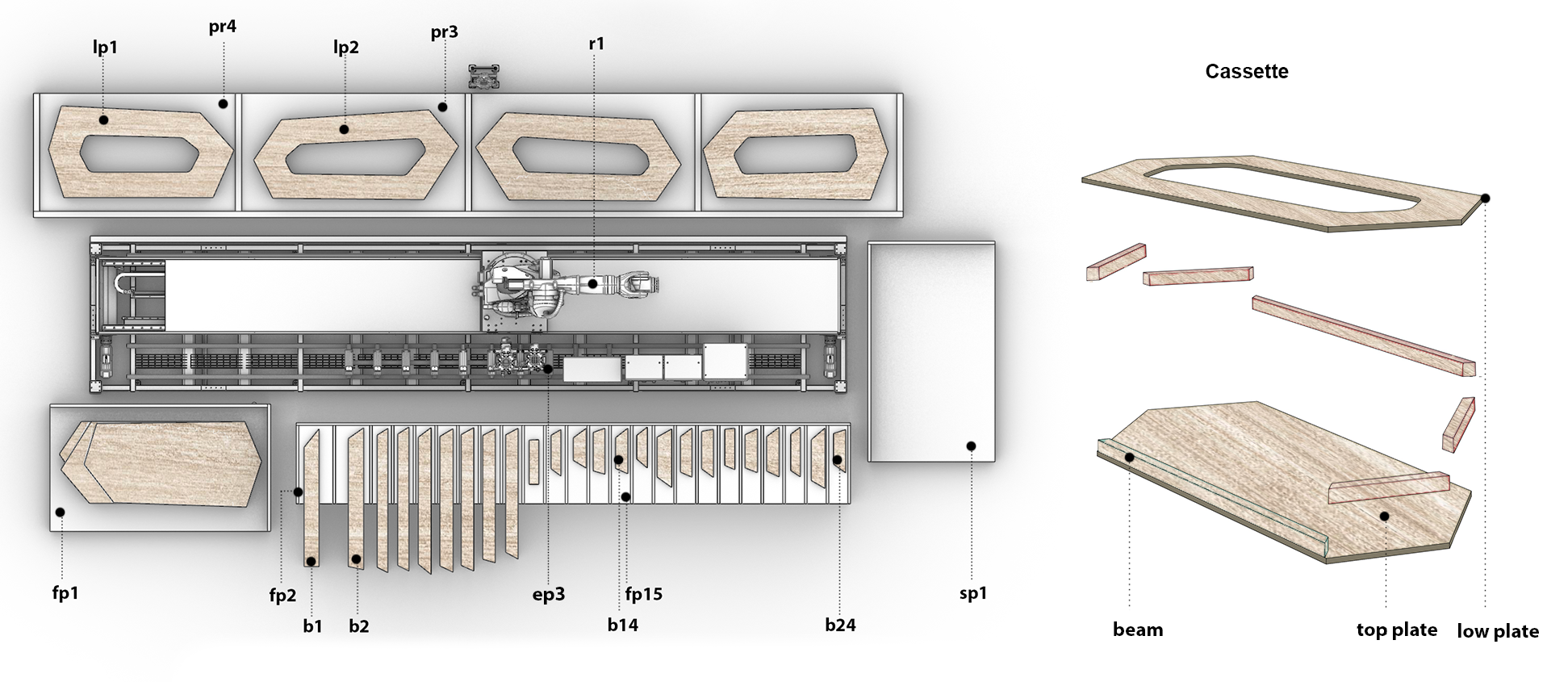}

\centering
\includegraphics[width=\linewidth]{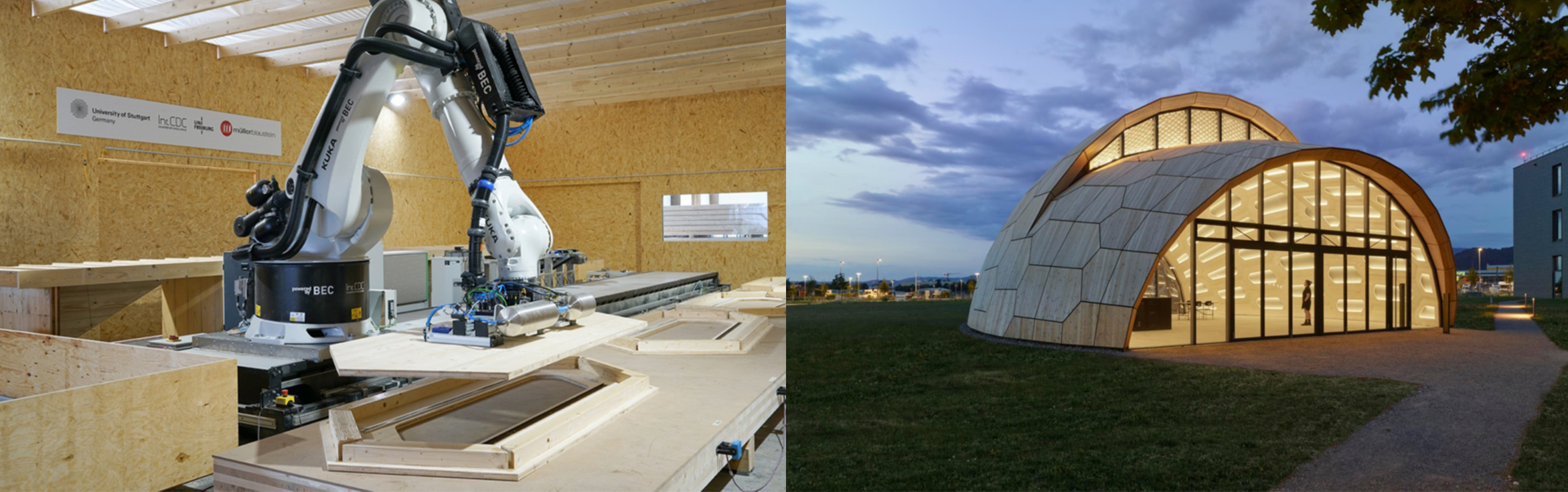}
\caption{Robotic prefabrication setup for the livMatS Biomimetic Shell project~\cite{Wagner20}. Shown are the main types and instances: raw element storage \texttt{(FirstPos, fpn)}, rail positions \texttt{(PositionOnRail, prn)}, assembled cassette storage \texttt{(StackPosition, sp1)}, industrial robot \texttt{(Robot, r1)}, and equipment positions \texttt{(EquipPosition, epn)}. Each cassette comprises a low plate \texttt{(lpn)}, a top plate \texttt{(tpn)}, and six beams \texttt{(bn)}, where n denotes the instance number.}
\label{fig:livMatS Biomimetic Shell2}
\end{figure*}

The industrial robotic arm is mounted on a linear axis, enabling access to four work table locations: raw element storage, rail positions, assembled cassette storage, and end-effector equipment positions. Raw element storage (\texttt{FirstPos}) serves as a staging area for wooden plates and beams, rail positions (\texttt{PositionOnRail}) support robotic manipulation and assembly, assembled cassette storage (\texttt{StackPosition}) is used for accumulating completed cassettes, and equipment positions (\texttt{EquipPosition}) store end-effectors.

The robotic arm is equipped with three distinct end-effectors: a vacuum gripper for picking up and placing elements, a nail gripper for mechanical fastening, and a glue gun for adhesive bonding. Within this workspace, the robotic arm assembles three to four cassettes per operational cycle. For this study, the focus is on the assembly of three batches of four cassettes each, with each cassette consisting of six beams and two plates, as illustrated in Figure~\ref{fig:livMatS Biomimetic Shell2}.

\subsection{APTree Usage for the livMatS Pavilion} \label{aptreeusage}
Users must provide the input models required to use APTree. These input models include\footnote{All input models can be found in our Git repository, see \ref{APPB}.}:

\begin{itemize}
    \item \textbf{Property types}, such as location types, element types, and end-effectors (partially shown in Listing~\ref{lst:type_model}).
    \item \textbf{Predicate types} and the \textbf{initial state} describing the fabrication setup (partially shown in Listings~\ref{lst:type_model} and~\ref{lst:concrete_model}).
  \item \textbf{Action types} required for the case study task planning. The defined action types are summarized in Table~\ref{tab:planner_actions}, based on the symbolic planner domains for HL and ML planning. An example of an action type definition is shown in Listing~\ref{lst:type_model}.
    \item \textbf{Symbolic planner initial domain and problem files} used by APTree. For this case study, two PDDL domain files (for ML and HL planning) and 12 initial problem files (one for each cassette) are manually developed.
   \item \textbf{Initial BT and planner configurations}, shown in Listing~\ref{lst:LiveMat_model_instance}. This configuration specifies three batch branches, each containing four cassette planning branches. ENHSP is assigned as the HL planner for each cassette branch.
\end{itemize}

\begin{nolinenumbers}

\begin{lstlisting}[
    language=MontiCore, 
    basicstyle=\scriptsize\ttfamily, 
    basewidth=0.5em,
    lineskip=-1pt,
    keywordstyle=\color{blue}\bfseries,
    commentstyle=\color{darkgray}\bfseries\itshape,
    morekeywords={BehaviorTree, Root, FlowNode, Nodegraph, Action, Service, NodeGraph, ServicePlanning, Domain, Problem, Enhsp, Config}, 
    caption={Excerpt of the initial APTree model for task planning across 12 cassettes in 3 production batches}, 
    label={lst:LiveMat_model_instance}
]
BehaviorTree livMat {
  Root FlowNode Main {
    All
    AllFlow
    NodeGraph {
      FlowNode Batch1_C1toC4 {
        All
        AllFlow
        NodeGraph {
          FlowNode Cassette1 {
            ServicePlanning mypddlplanner1 Enhsp Domain:DomainHL Problem:ProblemC1 Config: "sat-hadd"
            All
            AllAction
            NodeGraph { }
          }
          FlowNode Cassette2 { ... }
          FlowNode Cassette3 { ... }
          FlowNode Cassette4 { ... }
        }
      } {
        --[Meets]--> Batch2_C5toC8; // sequential execution between batches
      }
      FlowNode Batch2_C5toC8 { ... } {
        --[Meets]--> Batch3_C9toC12;
      }
      FlowNode Batch3_C9toC12 { ... }
    }
  }
}
\end{lstlisting}

 \end{nolinenumbers}
 
\begin{table}[h]
    \centering  
    \footnotesize
    \begin{tabular}{>{\raggedright\arraybackslash}p{4cm} >{\raggedright\arraybackslash}p{11.5cm}}
        \toprule
        \textbf{Action Level and Planner} & \textbf{Action Types} \\
        \midrule
        HL-ENHSP & PickUpHL, PlaceHL, StackHL, StackOnMultipleHL, GluingPlateHL, GluingBeamHL, NailingHL, NailingBeamHL, MoveCassetteToStackHL \\
        ML-LAMA-First and Metric-FF & TravelML, EquipeML, DeequipML, InitializeML, CloseToolML, PickUpML, PlaceML, StackML, StackOnMultipleML, GluingML, NailingML, PickUpCassetteML, PlaceCassetteML \\
        \bottomrule
    \end{tabular}
    \caption{Action node types representing tasks defined in each planner domain.}
    \label{tab:planner_actions}
\end{table}

Four planners are used in our case study evaluation: ENHSP~\cite{scala_subgoaling_2020}, Metric-FF~\cite{hoffmann_metric-ff_2003}, LAMA-First~\cite{richter_lama_2010}, and temporal fast downward (TFD)~\cite{eyerich_using_2009}. As shown in Table~\ref{tab:planner_actions}, ENHSP is used for HL planning, while Metric-FF and LAMA-First are used for ML planning\footnote{We used the Planutils Docker image to invoke the planners~\cite{muise-et-al}.}. The TFD planner is used only for performance comparison against PlanSys2, as PlanSys2 only supports temporal planners~\cite{martin_plansys2_2021}.

The HL planning domain uses numeric fluents, universal quantifiers, and conditional effects to model assembly constraints, including the variable attachment capacity of components. ENHSP was selected because it supports these language features and can generate plans for our case study. The ML planning domains correspond to classical planning, as the HL planner has already determined the assembly decisions; the ML
planners therefore only need to determine the robotic traveling and tool-handling tasks required to realize those decisions. Metric-FF and LAMA-First were selected as representative classical planners: Metric-FF prioritizes planning speed, whereas LAMA-First employs landmark-based heuristics that typically yield higher-quality plans at the expense of
increased planning time~\cite{richter_lama_2010}. Unless stated otherwise, reported results use ENHSP at the HL and Metric-FF at the ML, as this combination is considerably faster; results with LAMA-First are reported where explicitly indicated.

The usage of APTree is also illustrated in Figure~\ref{fig:AD1} using a UML activity diagram. The model loader accepts user-defined input models and validates them according to the defined DSL. If the input is invalid, error messages are returned to the user. If the input is valid, the model loader registers all types and instances on the blackboard and signals the BT execution engine to initiate ticking the BT.

The execution engine starts the BT by sending a tick signal to the first dynamic flow node (the root). APTree operates in two modes: planning mode, in which only planner services associated with flow nodes are permitted to generate nodes (the \texttt{ServicePlanning} instances in Listing \ref{lst:LiveMat_model_instance}), and action execution mode, in which action nodes are executed (actions in Listing \ref{lst:LiveMat_model_instance-final}). On the initial tick, planning mode is activated. After all planner services that receive a tick complete HL plan generation (i.e., plan nodes are generated for each of the four initially empty flow nodes shown in Listing~\ref{lst:LiveMat_model_instance}), the action execution mode is triggered.

If a planner service fails to generate a plan or if the generated model violates the DSL constraints, an error message is reported to the user, and BT execution is paused (Figure \ref{fig:AD1}). Upon successful plan generation, the resulting action nodes either execute directly or are further expanded into a lower level of granularity. In the presented case study, mid-level actions are derived from high-level actions; no low-level actions are produced because no low-level motion planner is introduced.

\begin{figure}[h]
    \centering
    \includegraphics[width= 0.8\linewidth]{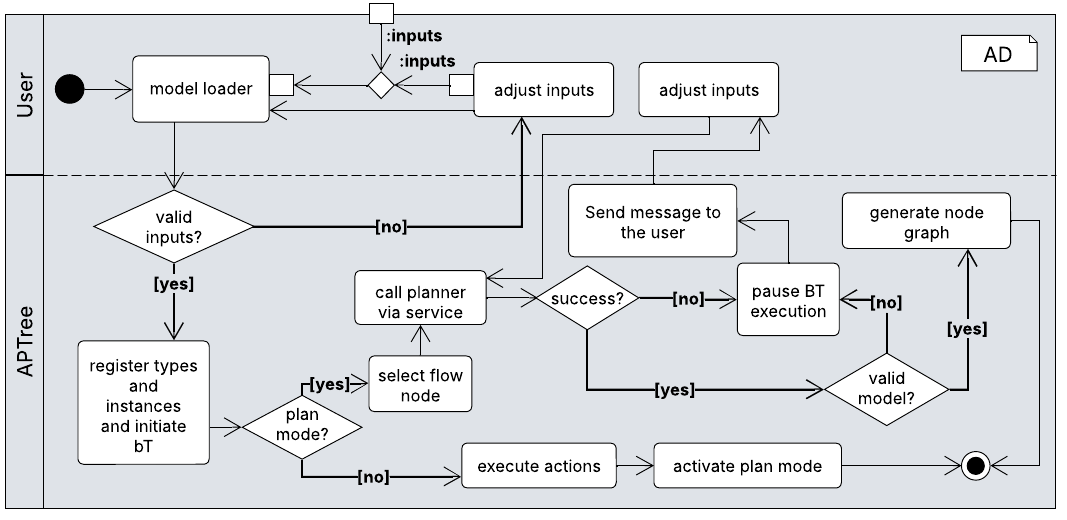}
\caption{Activity diagram illustrating the process from input provision to node generation and planning modes. }
    \label{fig:AD1}
\end{figure}

In action execution mode, BT decorators coordinate execution across the four branches of each batch. The \texttt{Lock} decorator is applied automatically to all ML subtrees, guaranteeing their atomic execution and preventing conflicting interleavings between partially ordered plans that share resources (e.g., the robot). The remaining decorators define the execution policy for coordinating the HL independent cassette assembly branches. We evaluate two configurations: execution with replanning and execution without replanning. Both produce the final BT shown in Listing~\ref{lst:LiveMat_model_instance-final}.

The four cassette assembly branches are independent and unordered. Their execution policy therefore affects only execution order, not plan correctness, and can reduce superfluous actions (e.g., allowing the robot to complete all tasks requiring the same end-effector across unordered branches before switching tools). Two policies are considered: \texttt{FairProgress}, which coordinates balanced branch execution without forcing replanning, and \texttt{CostEvaluation}, which deliberately maximizes replanning required for the evaluation of local replanning:

\textbf{Configuration with replanning}

In this configuration, planning is first performed for all four flow nodes in a batch, generating the corresponding high-level actions. The executable action in each branch is then selected (e.g., lines 15, 25 in Listing~\ref{lst:LiveMat_model_instance-final}), and the corresponding ML subtrees are generated, with their ML actions planned (e.g., lines 50–53 in Listing~\ref{lst:LiveMat_model_instance-final}). Among all executable branches, the cost evaluation decorator selects the ML subtree with the lowest estimated cost (number of actions in the branch).
The \texttt{Lock} decorator then ensures exclusive execution of the selected subtree until all its ML actions are completed. After execution, action effects are applied to the blackboard, and a \texttt{Reset} decorator removes the remaining obsolete planned subtrees in other branches and triggers replanning based on the updated world state. This cycle continues until all tasks are completed.

\textbf{Configuration without replanning (fair progress)}

In the second configuration, execution is coordinated using the \texttt{FairProgress} decorator, which enforces balanced progression across branches instead of selecting the globally cheapest subtree.
After executing a high-level action in one branch, priority shifts to available actions in other branches. If preconditions are not satisfied (e.g., the robot is already holding an object), execution continues within the same branch. As in the first approach, the \texttt{Lock} decorator guarantees atomic execution of the selected ML subtree. However, replanning is not triggered after each subtree execution; coordination is achieved solely through controlled branch progression.

These two decorator configurations are representative execution policies used to evaluate APTree's planning and replanning capabilities rather than the only possible coordination mechanisms, and for simplicity, are omitted from the models in Listings \ref{lst:LiveMat_model_instance} and \ref{lst:LiveMat_model_instance-final}. Since APTree extends standard BTs, it preserves their execution flexibility and other decorators or execution policies can be used as required.

\section{Evaluation Using Case Study Results} \label{eval}

This section evaluates APTree using the CRF case study described in Section~\ref{sec:casestudy}. The evaluation follows the GQM framework defined in Section~\ref{reqs}, addressing research questions in sequence.
All experiments were performed on a Vivobook ASUS K6604JI laptop equipped with a 13th Gen Intel(R) Core(TM) i9-13980HX CPU and an NVIDIA GeForce RTX 4070 GPU. We compare APTree with PlanSys2, as both use symbolic planning to generate BT-based execution models. PlanSys2 follows a two-stage architecture in which a symbolic planner first generates a task plan that is then transformed into a BT execution model using BT.CPP~\cite{martin_plansys2_2021}. In contrast, APTree integrates planning directly into the execution model, allowing task plans to be generated during execution rather than in a separate planning stage.

\textbf{Answering Q1--Q4.}
Representative structural, typing, syntactic, and logical errors are injected into the plan models generated for a planning problem from the CRF case study (\ref{APPA}). For APTree, errors are injected into the integrated planning and execution model. For PlanSys2, errors are injected for the symbolic planner and, when applicable, the generated execution model. Due to PlanSys2's two-stage architecture, an error is counted as detected only if it is identified before runtime. Accordingly, M1 (TPR) measures the proportion of injected errors detected at design time, while M2 (PR) measures the proportion prevented by design.

For M1, both APTree and PlanSys2 achieve a TPR of 100\% for structural errors, though at different stages. APTree detects all structural errors at design time through its grammar-based validation. PlanSys2 detects structural errors using the symbolic planner parser, at planner invocation time, while structural errors in the BT execution model are prevented by design, as the BT is generated automatically from the plan.
For typing errors, APTree achieves a TPR of 100\% by enforcing type consistency between the planning model and the execution model at design time. In PlanSys2, action types and their properties must be independently redefined in the BT execution layer (as BT.CPP action nodes), with no mechanism to verify that these definitions remain consistent with those in the PDDL planning domain. This separation is the source of potential type mismatches, resulting in a TPR of 0\%. Typing errors are largely inapplicable to the symbolic planner itself, since generated plans inherently conform to the action signatures defined in the domain.

Syntactic errors are detected trivially, as both APTree and PlanSys2 use DSLs that enforce syntactic correctness. However, APTree requires an initial BT to be authored before planning and therefore validates it at design time, while PlanSys2 automatically generates the execution BT, preventing syntax errors by design.
Regarding logical errors, PlanSys2 generates a single plan for each planning problem. Therefore, logical errors involving interactions between multiple execution models cannot be detected during planning and manifest as silent failures during execution. APTree likewise cannot detect these errors
at design time; however, it detects all injected logical errors at runtime through integrated reasoning over its planning and execution models. Of these,
50\% are resolved automatically (error~15 by replanning and error~18 by shifting execution to other branches until the preconditions are met again), while the remaining 50\% (errors~16 and~19) halt BT execution due to an infeasible plan. Overall, across all error categories, APTree achieves a TPR of 76\%, with the remaining errors detected at runtime, while PlanSys2 achieves a TPR of 25\%.

Regarding M2, APTree prevents 10\% of the injected errors by construction, as the initial empty BT is not automatically generated. In contrast, PlanSys2 prevents 36.8\% of injected errors, mainly syntactic errors, because both the symbolic plan and the BT execution model are generated automatically.

\begin{table}[h]
\centering
\footnotesize
\renewcommand{\arraystretch}{1.15}
\begin{tabularx}{\textwidth}{c p{1.1cm} X c c c}
\toprule
\multirow{2}{*}{\textbf{ID}} &
\multirow{2}{*}{\textbf{Error Type}} &
\multirow{2}{*}{\textbf{Error Description}} &
\multirow{2}{*}{\textbf{APTree}} &
\multicolumn{2}{c}{\textbf{PlanSys2}} \\
\cmidrule(lr){5-6}
& & & & \textbf{BT.CPP} & \textbf{Symbolic Planner} \\
\midrule

1  & \multirow{5}{*}{Structural}
   & Mismatch of planning model artifacts (e.g., planner and its configuration) & DT & NA & PT \\
2  & & Missing planning model artifact (e.g., lacking problem or domain) & DT & NA & PT \\
3  & & Mismatch of planner and planning models (e.g., unsupported language) & DT & NA & PT \\
4  & & Execution model involves invalid component hierarchy & DT & PD & NA \\
5  & & Execution model inconsistent with the symbolic plan & PD & PD & NA \\
\midrule
\multicolumn{3}{r}{Structural errors (\%)}
& \textbf{100\%} & \multicolumn{2}{c}{\textbf{100\%}} \\

\midrule
6  & \multirow{5}{*}{Typing}
   & Undefined parameter instance in task execution model & DT & RTC & NA \\
7  & & Undefined action type  & DT & RTC & NA \\
8  & & Undefined parameter type in actions & DT & RTC & NA \\
9  & & Planned action parameter count does not match the execution action signature  & DT & RTC & NA \\
10 & & Planned action parameter type does not match the execution action signature & DT & RTC & NA \\
\midrule
\multicolumn{3}{r}{Typing errors (\%)}
& \textbf{100\%} & \multicolumn{2}{c}{\textbf{0\%}} \\

\midrule
11 & \multirow{5}{*}{Syntax}
   & Delimiter mismatch in the model syntax & DT & PD & PD \\
12 & & Invalid ordering of execution model constructs & DT & PD & PD \\
13 & & Unexpected token in the model specification & DT & PD & PD \\
14 & & Missing required syntactic element or keyword & DT & PD & PD \\
\midrule
\multicolumn{3}{r}{Syntax errors (\%)}
& \textbf{100\%} & \multicolumn{2}{c}{\textbf{NA}} \\

\midrule
15 & \multirow{4}{*}{Logical}
   & World-state changes introduced by a preceding execution model invalidate the current plan. & RT & SF & NA \\
16 & & Execution model regresses a goal achieved by another execution model. & RT & SF & NA \\
17 & & Execution model cannot achieve its goal due to broken causal links. & PD & PD & NA \\
18 & & Action preconditions are invalidated by another execution model. & RT  & SF & NA \\
19 & & Two execution models pursue logically incompatible goals. & RT & SF & NA \\
\midrule
\multicolumn{3}{r}{Logical errors  (\%)}
& \textbf{0\%} & \multicolumn{2}{c}{\textbf{0\%}} \\

\bottomrule
\end{tabularx}
\caption{Comparison of error detection and prevention between APTree and PlanSys2. DT = detected at design-time, PT = detected at planning time, RT = detected at runtime, RTC = detection at runtime leads to crash, SF = silent failure, NA = not applicable, PD = prevented by design}
\label{tab:error_comparison}
\end{table}

\textbf{Answering questions Q5--Q6.}
We next evaluate the locality and scalability of APTree replanning, using the configuration with replanning, based on metrics measuring average structural change (M3), replanning frequency (M4), \hl{and node retention rate (M5).}
When replanning is required, PlanSys2 invokes the symbolic planner on the updated planning problem and regenerates the entire BT, replacing the previous execution model ($g_n = 100\%$). In contrast, over 1{,}098 replanning events, APTree's $g_n$ averaged 0.80\%. The $g_n$ values ranged from 0.00\% to 6.50\%, indicating that individual replanning operations consistently affect only a small fraction of the overall tree structure (see Figure \ref{prr}). The low average confirms that APTree supports runtime local replanning, replacing only the specific subtree requiring adaptation without disrupting the broader execution flow.

As APTree flow nodes can represent alternative execution orders within a single graph, replanning of a subtree does not alter the BT hierarchy regardless of the number of planned tasks. Consequently, the execution hierarchy remains unchanged throughout replanning. As shown in Table~\ref{tab:typesandinstances}, the final execution model contains 1{,}149 action nodes while maintaining a constant depth of six levels (Listing~\ref{lst:LiveMat_model_instance-final}), yielding $D_0=D_1=6$ and $g_d=0$. In contrast, the depth of a PlanSys2-generated BT increases with the number of sequential tasks in the plan.

The $g_w$ metric (Figure~\ref{prr}) averaged 0.38\% (range
0.00--4.84\%), indicating that tree width remains largely stable
despite continuous replanning. Across 1{,}098 replanning events, the
structural change metric (M3) is $\overline{G}_{\text{tree}} = 0.39\%$. Notably, the structural change per replanning event is bounded by the size of the replanned ML subtree, so its relative magnitude decreases as the execution model grows; replanning thus remains localized for increasingly large task plans.

Table~\ref{tab:typesandinstances} further shows that repeated replanning generated 6{,}507 action nodes across 1{,}098 replanning events, while only 1{,}149 nodes remained in the final BT and were executed to completion, identical to the configuration without replanning. This corresponds to a replanning frequency (M4) of 81.4 replans/min over 13 minutes and 29 seconds execution and a \hl{node retention rate (M5)} of 17.7\%. \hl{Because replanning in this configuration is deliberately induced by the \texttt{CostEvaluation} policy, these values reflect the replanning load APTree sustains while preserving the running execution model, rather than fault-driven recovery; they indicate that APTree supports frequent, localized replanning as the execution model grows. Complementing these plan-change metrics with task outcomes: in both configurations, all 372 high-level tasks achieved their goals, all 1{,}098 replanning invocations returned valid plans with no planner failures, and no action failures occurred. As actions in our setup always return success once they are ticked, these outcomes reflect valid plan generation and composition rather than physical execution performance.}

\begin{figure}[h]
    \centering
    \includegraphics[width=\linewidth]{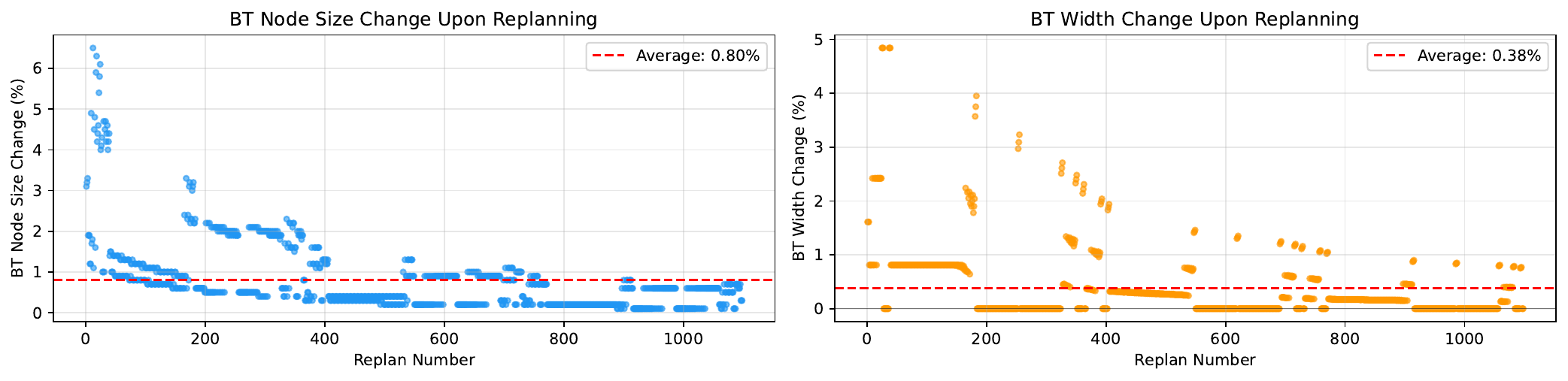}
    \caption{BT Node Change ($g_n$,left) and BT Width Change ($g_w$,right) across 1{,}098 replanning events. Each dot represents a single replan; the dashed red line indicates the average (0.80\% and 0.38\%, respectively).}
    \label{prr}
\end{figure}

\begin{table}[h]
\centering
\footnotesize
\begin{tabularx}{\textwidth}{l ccc ccc}
\hline
 & \multicolumn{3}{c}{With Replanning \_ 13 min 29 sec} 
 & \multicolumn{3}{c}{Without Replanning \_ 5 min 6 sec} \\
\cline{2-7}
Category &
Final &
Before Ticking &
After Ticking &
Final &
Before Ticking &
After Ticking \\ \hline

Action Instances (ML \& HL)  & 1{,}149 & 0 & 6{,}507 & 1{,}149 & 0 & 1{,}149 \\
Action Instances (HL)        & 372 & 0 & 372  & 372 & 0 & 372 \\
Replanning triggered        & 1{,}098 & 0 & 1{,}098  & 0 & 0 & 0 \\
PDDL Problem files           & 384 & 12 & 1{,}470  & 384 & 12 & 372 \\
FlowNode Instances           & 388 & 16 & 372  & 388 & 16 & 372 \\ \hline
\end{tabularx}
\caption{Comparison of BT structure metrics between execution with dynamic replanning (1{,}098 replans over 13 min 29 s) and without replanning (5 min 6 s), showing node counts at final state, before ticking, and after ticking. }
\label{tab:typesandinstances}
\end{table}

\textbf{Answering questions Q7--Q10.}
This evaluation examines the ability of APTree to apply and coordinate symbolic sub-plans generated by multiple automated planners and to decompose and recompose large-scale planning problems into a single executable and step-efficient task plan, as defined by Metrics M6--M12. Experiments were conducted under both execution configurations (Section~\ref{aptreeusage}): with and without replanning. Unless otherwise stated, all experiments use ENHSP for high-level planning and Metric-FF for mid-level planning to provide a consistent evaluation baseline. LAMA-First is additionally included in the evaluation of planner diversity (M6) and planner application latency (M8) to demonstrate heterogeneous planner integration and quantify the impact of planner choice on execution performance.

Metrics M6--M8 evaluate the framework's ability to coordinate heterogeneous symbolic planners and repeatedly apply them during execution. As shown in Table~\ref{tab:plannercomparison}, APTree can employ three distinct symbolic planners (ENHSP, Metric-FF, and LAMA-First), yielding a planner diversity of three and demonstrating heterogeneous planner integration. During execution, Metric-FF and LAMA-First are each invoked 735 times
(each handling half of the ML planning modules), whereas ENHSP is
invoked 12 times, as high-level actions do not require replanning in our deterministic planning setup. In total, the framework automatically generates 1{,}470 and 372
planning problem instances during execution for the configurations with
and without replanning, respectively (Table~\ref{tab:typesandinstances}),
reducing manual problem specification to the 12 initial problem files (Listing~\ref{lst:LiveMat_model_instance}). In contrast, PlanSys2 invokes a single planner once for the entire execution. These results demonstrate that APTree supports distributed, repeated, and heterogeneous symbolic planning within a single execution model.

The M8 metric evaluates the cumulative planning latency (PL). Without replanning, using ENHSP + Metric-FF, the system accumulates 2\,min\,20\,s of planning latency over a total execution time of 5\,min\,6\,s, yielding a normalized PL of 45.8\%. With replanning enabled, the same planner combination accumulates 8\,min\,53\,s over 13\,min\,29\,s, corresponding to a normalized PL of 65.9\%. Replacing half of the ML subtree planners with LAMA-First increases the cumulative planning latency to 43\,min\,8\,s over 48\,min of total execution, yielding a normalized PL of 89.9\%. Table~\ref{tab:plannercomparison} summarizes the per-planner latency breakdown for this three-planner replanning configuration.

Figure~\ref{plannerlatency} shows that the average planning latency per invocation remains stable throughout 1{,}482 planner calls in the replanning configuration --- the configuration with the highest number of planning invocations. The average latency is approximately 360\,ms with ENHSP + Metric-FF, and 1{,}746\,ms with ENHSP + Metric-FF + LAMA-First. These results indicate that cumulative planning latency is determined primarily by the choice of symbolic planner, while repeated planning and replanning do not introduce additional per-invocation overhead.

\begin{nolinenumbers}
\begin{lstlisting}[
    language=MontiCore, 
    basicstyle=\scriptsize\ttfamily, 
    basewidth=0.5em,
    lineskip=-1pt,
    keywordstyle=\color{blue}\bfseries,
    commentstyle=\color{darkgray}\bfseries\itshape,
    morekeywords={BehaviorTree, Root, FlowNode, Nodegraph, Action, Service, NodeGraph, ServicePlanning, PickUpHL, PlaceHL, Domain, Problem, Enhsp, FF}, 
    escapechar=§,
    caption={Excerpt from the final generated APTree model. Ellipses (...) denote omitted actions or additional behavior tree fragments. The highlighted subtree reference is defined below.},
    label={lst:LiveMat_model_instance-final}
]
BehaviorTree livMat {
  Root FlowNode Main {
    All
    AllFlow
    NodeGraph {
      FlowNode Batch1_C1toC4 {
        All
        AllFlow
        NodeGraph {
          FlowNode Cassette1 {
            ServicePlanning mypddlplanner1 Enhsp Domain:DomainHL Problem:ProblemC1
            All
            AllAction
            NodeGraph {
              Action PickUpHL PickUpHL_lp1_fp1_r1 (lp1 fp1 r1) §\textcolor{purple}{@PickUpHL\_lp1\_fp1\_r1Subtree}§ {
                --[Meets]--> PlaceHL_lp1_pr1_r1;
              }  ...             
            }
          }
          FlowNode Cassette2 {
            ServicePlanning mypddlplanner2 Enhsp Domain:DomainHL Problem:ProblemC2
            All
            AllAction
            NodeGraph {
              Action PickUpHL PickUpHL_lp2_fp9_r1 (lp2 fp9 r1) §\textcolor{mygreen}{@PickUpHL\_lp2\_fp9\_r1Subtree}§ {
                --[Meets]--> PlaceHL_lp2_pr2_r1;
              }  ...            
            }
          }
          FlowNode Cassette3 { ... }
          FlowNode Cassette4 { ... }
        }
      } {
        --[Meets]--> Batch2_C5toC8;
      }
      FlowNode Batch2_C5toC8 { ... } {
        --[Meets]--> Batch3_C9toC12;
      }
      FlowNode Batch3_C9toC12 { ... }
    }
  }
}

BehaviorTree §\textcolor{mygreen}{PickUpHL\_lp2\_fp9\_r1Subtree}§ {
  Root FlowNode FF_Default_DynamicFlow_PickUpHL_lp2_fp9_r1 {
    ServicePlanning subtreeSrv_PickUpHL_lp2_fp9_r1 FF Domain:DomainML Problem:problemPickUpHL_lp2_fp9_r1
    All
    AllAction
    NodeGraph {
      Action TravelML TravelML_r1_pr1_fp9 (r1 pr1 fp9) {
        --[Meets]--> PickUpML_lp2_fp9_r1_vg1;
      }
      Action PickUpML PickUpML_lp2_fp9_r1_vg1 (lp2 fp9 r1 vg1)
    }
  }
} ...
\end{lstlisting}
\end{nolinenumbers}

\begin{table}[h]
\resizebox{\textwidth}{!}{%
\begin{tabular}{lcccccc c}
\hline
PlannerType &
  Call Count &
  Success Rate &
  \begin{tabular}[c]{@{}c@{}}Average Planner Time (ms)\end{tabular} &
  \begin{tabular}[c]{@{}c@{}}Average Service Time (ms)\end{tabular} &
  Average Plan Size &
  \begin{tabular}[c]{@{}c@{}}Cumulative Planner\\Time (s)\end{tabular} &
  \begin{tabular}[c]{@{}c@{}}Cumulative Service\\Time (s)\end{tabular} \\ \hline
ENHSP      & 12   & 100.00\% & 1364.30 & 1556.77 & 31   & 16.37   & 18.68   \\ \hline
Metric-FF  & 735  & 100.00\% & 270.10  & 349.96  & 3.9  & 198.52  & 257.22  \\ \hline
LAMA-First & 735  & 100.00\% & 3082.30 & 3145.78 & 4.5  & 2265.49 & 2312.15 \\ \hline
TOTAL      & 1482 & 100.00\% & 1673.67 & 1746.32 & 4.4  & 2480.38 & 2588.05 \\ \hline
\end{tabular}%
}
\caption{Average planning time, planning service time (including parsing and action instance generation), and average plan size by planner type. The TOTAL row reports call-weighted averages. These values were obtained using the active replanning configuration.}
\label{tab:plannercomparison}
\end{table}

\begin{figure}[h]
    \centering
    \includegraphics[width=\linewidth]{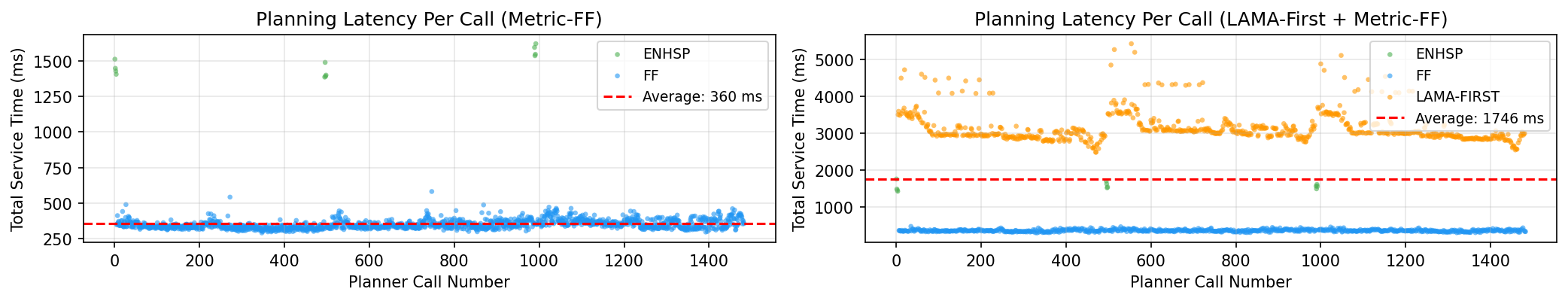}
   \caption{Per-call planning service latency for Metric-FF only (left) and LAMA-First + Metric-FF (right) as ML planners, ENHSP as HL planner, with replanning. Per-call latency remains stable across all invocations, with no degradation from repeated calls.}
    \label{plannerlatency}
\end{figure}

To compare the scalability and performance of APTree against PlanSys2, we defined five planning problems of increasing complexity derived from the livMatS case study: P0 (single pick-up), P1 (pick-and-place), P2 (pick-place-glue), P3 (P2 plus one stacking beam task), and P4 (P3 plus a second stacking beam task). To ensure a fair comparison, both approaches use the same temporal planner (TFD): APTree applies TFD to each decomposed ML planning problem, whereas PlanSys2 applies TFD once to the complete planning problem. \hl{Unlike the livMatS case study, each comparison problem (P) is specified to APTree as a single HL planning problem without manual decomposition; the ML planning problems are generated automatically by APTree, which refines each planned HL action into a separate ML problem whose goal is the effects of that action.}  Figure~\ref{plansys} compares APTree's modular planning approach with PlanSys2's monolithic planning. For the simplest problems (P0--P1), PlanSys2 is faster (168--850\,ms vs. 1.4--2.1\,s) because it performs a single planning call without decomposition overhead. As problem complexity increases, however, APTree's modular planning scales more effectively. At P2, APTree's cumulative planning time is 3.0\,s versus 5.7\,s for PlanSys2. At P3 and P4, PlanSys2 fails with out-of-memory errors because the entire planning problem must be grounded simultaneously (full problems available in our public repository,~\ref{APPB}), whereas APTree solves both by decomposing them into bounded ML-level subproblems that are planned and executed incrementally. \hl{This automatic hierarchical decomposition enables APTree to solve problems beyond the reach of monolithic planning while reducing time to first action. These experiments therefore establish the benefit of APTree-enabled decomposition, rather than directly demonstrating that APTree itself is intrinsically more scalable than PlanSys2 under matched decompositions.} 

\begin{figure}[h]
    \centering
    \includegraphics[width= \linewidth]{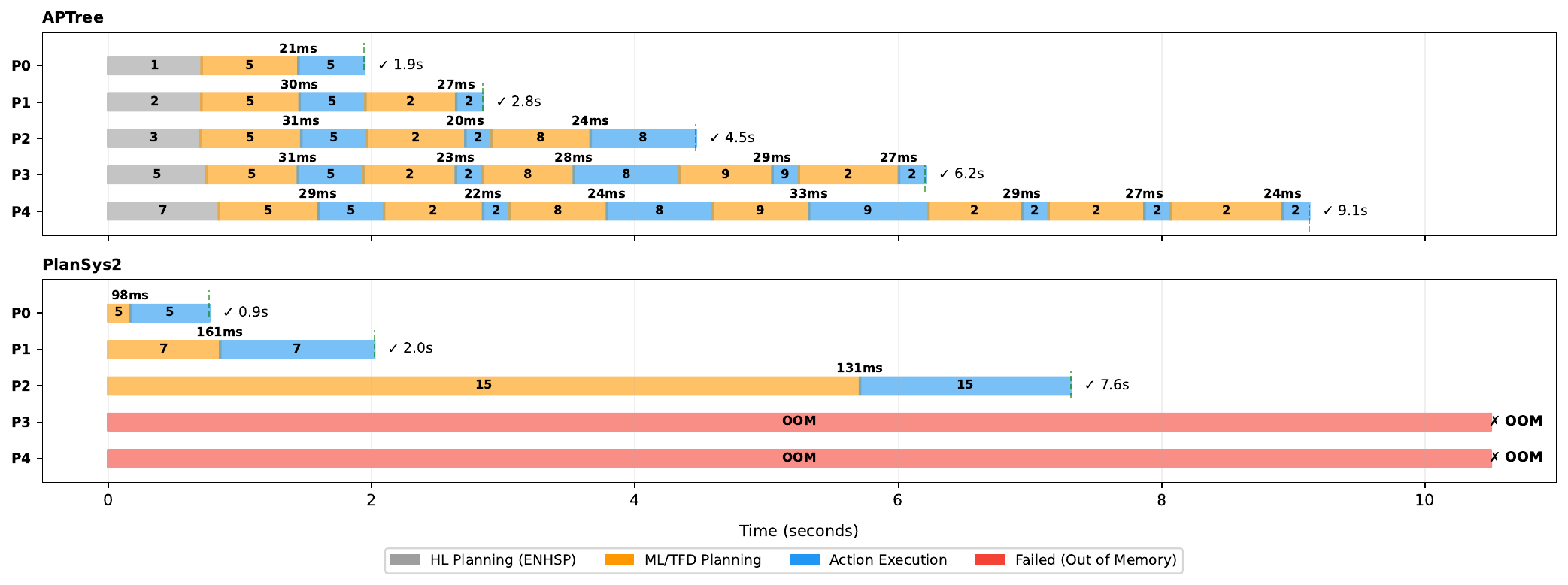}
  \caption{Comparison of APTree and PlanSys2 on progressively larger subsets of the livMatS case study. The values above the planning-to-execution transitions indicate the BT generation overhead. The values inside ribbons are the number of actions per planning instance.}
    \label{plansys}
\end{figure}

Metric M9 evaluates the framework's ability to decompose large symbolic planning problems into smaller subproblems and recompose them into a single executable task plan. In the livMatS case study, the final executed plan comprises 1{,}149 action nodes, contributed by 372 planning submodules for both configurations (Table~\ref{tab:typesandinstances}), yielding a decomposition factor of 3.1. In contrast, a monolithic plan generated by PlanSys2 would yield a decomposition factor equal to the total number of executed task steps. These results demonstrate effective decomposition and recomposition, supporting modular and scalable task planning.

Metrics M10 and M11 evaluate plan correctness and step efficiency; a formal analysis establishing the conditions under which APTree preserves plan correctness is given in~\ref{APPformal}. We compared the generated plan against the manually authored livMatS task plan~\cite{Skoury_Wortmann_2024} as reference. As shown in Table~\ref{tab:plan_comparison}, all logically required steps and their ordering were preserved, giving a PPR of 100\% (M10). By coordinating balanced progression across the four independent branches, the proposed execution policies also minimize superfluous actions: whereas purely sequential execution of the branches would incur redundant tool changes and travel, yielding an SSR (M11) of 42\%, the proposed policies reduce it to 0\% with both of our execution policies. Together, these results show that APTree preserves the reference plan's logical structure while producing highly step-efficient execution models.

\begin{nolinenumbers}

\begin{table}[t]
\centering
\footnotesize
\begin{tabularx}{\textwidth}{>{\raggedright\arraybackslash}X | >{\raggedright\arraybackslash}X}
\hline
Manual Plan (JSON, Excerpt) & Generated Plan (APTree, Excerpt) \\ 
\hline
\texttt{"names": "Move", "actors": "Robot"} 
& \texttt{TravelML\_r1\_pr2\_ep1} \\

\texttt{"names": "Take Tool", "tools": "VG1"} 
& \texttt{EquipeML\_r1\_vg1\_ep1} \\

\texttt{"names": "Initiate Tool"} 
& \texttt{InitializeML\_r1\_vg1} \\

\hline
\end{tabularx}
\caption{Comparison and mapping between the manual expert plan (JSON) and the automatically generated task plan from APTree logs.}
\label{tab:plan_comparison}
\end{table}
\end{nolinenumbers}

To evaluate runtime performance (Q10), we compute the consecutive action latency (CAL), defined by M12. Figure~\ref{consecutiveactions}
shows the interval between consecutive action completions throughout execution alongside the cumulative net number of action nodes in the execution model. Although
replanning repeatedly adds and replaces nodes, the net node count remains bounded and finishes at the same level with and without replanning (1{,}149 nodes), showing that obsolete plan branches are effectively removed rather than accumulating. Average CAL is 383\,ms without replanning and 1{,}028\,ms with replanning; despite more planner invocations and execution model nodes, latency stays
stable throughout, indicating no cumulative execution overhead from repeated planning. Higher latencies occur only at inter-branch
transitions, where replanning is required: 706\,ms without replanning versus 2{,}207\,ms with replanning. Thus, the replanning configuration's longer execution time stems from repeated symbolic planning triggered by the execution policy, not from execution model
overhead. These results indicate that APTree adds negligible overhead beyond the intrinsic cost of symbolic planning, while maintaining low intra-branch latency suitable for online robotic execution.

\begin{figure}[h]
    \centering
    \includegraphics[width=\linewidth]{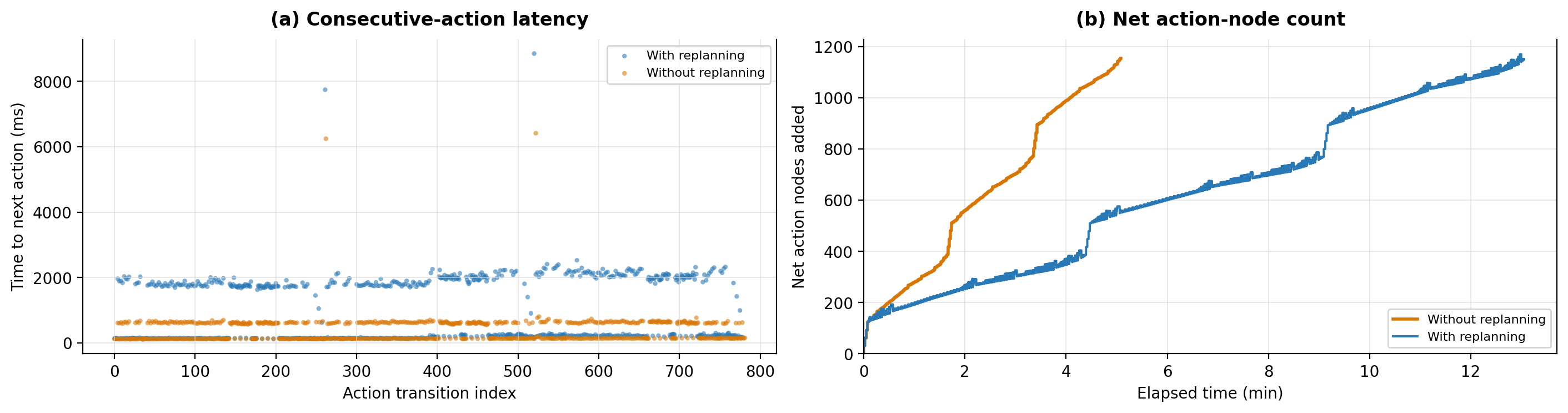}
    \caption{Execution sensitivity under the two branch-selection policies.
(a) Time between consecutive ML actions, with (blue) and without
(orange) replanning. (b) Cumulative net action nodes over execution
time; upward steps are generated nodes, downward steps are nodes
removed during plan replacement (replanning only).}
    \label{consecutiveactions}
\end{figure}

\section{Discussion} \label{disc}

Using a real-world, full-scale CRF task planning case study, we show as a proof of concept that APTree detects structural, syntactic, and typing errors at design time and logical errors at runtime (C1); supports localized runtime replanning via modular execution model reconfiguration (C2); and repeatedly applies and coordinates symbolic planners to compose scalable task execution models (C3). The DSL enables automatic validation of task plans and execution models, integrating error detection, planning, replanning, and execution within a single framework.

Working with BTs as the target execution model revealed that, while they enable significant modularity, they conflate planning and execution semantics. A sequence node, for instance, encodes both task ordering (a planning concept) and task success criteria (an execution concept) in one construct. This became a problem for modular replanning: replacing one module's plan often required replacing its entire
control-flow node, disrupting the execution flow, since ordering and
success criteria could not be changed independently. Likewise, even BT.CPP's explicit input/output ports do not distinguish a precondition from a plain runtime task property, and since BTs have no native representation for preconditions, effects, or world state, planners must encode these indirectly, e.g., via decorator nodes, which
can be modified or removed unintentionally. APTree addresses this by making planning semantics first-class grammar elements, separate from execution semantics, so replanning can target only the task plan rather than restructuring the execution model.

\subsection{Result Analysis}

Symbolic planners reason over planning semantics to generate a plan, but the resulting task plan typically discards this semantic information, breaking the link between a task's planning-level and execution-level representation. This is particularly problematic when task plans are developed independently and later composed within an execution model, as typing mismatches and logical interactions between separately defined task plans remain undetected. As shown in
Table~\ref{tab:error_comparison}, PlanSys2's two-stage architecture exhibits this limitation, achieving a TPR of 0\% for both typing and
logical errors, which instead manifest as runtime failures. In contrast, APTree preserves planning semantics through a typed DSL
shared by the planning and execution models, enabling structural, syntactic, and typing errors to be detected at design time, and logical errors at runtime.

The runtime detection of logical errors operates through two
complementary mechanisms. Half of the injected logical errors are
detected and resolved automatically through continuous evaluation of
HL-action preconditions and effects against the current world state:
the execution policies move execution between execution models,
releasing only those actions whose preconditions currently hold, so
that multiple task modules progress across branches without waiting
for one another. The remaining half are detected through a plan
invariant: a completed HL goal must not be undone by another execution
model (Invariants~\ref{inv:initial-noninterference}
and~\ref{inv:ml-noninterference}, \ref{APPformal}). Invariant~\ref{inv:initial-noninterference} is a static property of the problem specifications: when it is violated (for example, by initial problem files specifying
contradictory goals) no valid plan exists, and APTree deliberately
terminates execution. Invariant~\ref{inv:ml-noninterference}, in contrast, concerns the plans generated at runtime, in particular the ML plans produced by refinement and replanning, and cannot be verified from the initial
specifications alone; its violations manifest at execution time as violations of HL goals and lead to replanning or a full stop of BT execution (\ref{APPformal}). Invariant~\ref{inv:ml-noninterference} concerns interactions among execution models; goal regressions caused
by external disturbances are a separate concern of deployment in dynamic environments. In addition, contradictions between ML-level goals cannot corrupt execution, as each ML plan executes
atomically (the \texttt{Lock} decorator).

Our findings further demonstrate that separating execution semantics (e.g., success criteria) from the task plan itself enables localized planning and replanning at runtime while preserving the flexibility and scalability of BTs. This separation also allows practitioners to select any PDDL-based symbolic planner without concern for its compatibility with the execution model or the combinatorial explosion associated with monolithic planning, as seen with PlanSys2.

Finally, the case study demonstrates that integrating and autonomously invoking multiple symbolic planners within a single execution model improves modularity and scalability for planning. We attribute this scalability to APTree's horizontal and hierarchical problem decomposition: each ML sub-plan corresponds to a comparatively small and constrained planning problem, on which satisficing planners such as Metric-FF consistently produce near-optimal solutions. This
also puts the observed cumulative planning latencies (up to 89.9\% of execution time with LAMA-First) into perspective: planning latency is
dominated by planner choice rather than by the framework itself, and since per-invocation latency remains stable across repeated
invocations, it can be reduced substantially by selecting faster planners --- as our results show, without a loss in plan
quality. This suggests that, for hierarchically decomposed task execution models, planner optimality is less critical than planning
efficiency.

\subsection{Application Scope, Limitations and Future Work}

APTree is designed for highly flexible, scalable, and multi-faceted robotic pre-assembly task planning. Construction prefabrication was selected as the target domain because it is particularly affected by this challenge, requiring modular, scalable task planning to handle its inherent task heterogeneity and variability~\cite{davila_delgado_robotics_2019}. Construction prefabrication is commonly categorized into component sub-assembly, non-volumetric pre-assembly, volumetric pre-assembly, modular construction, hybrid structures, and whole-building prefabrication~\cite{braganca_modularity_2025}. The presented case study belongs to the component sub-assembly category, as it considers the robotic fabrication of timber cassettes for on-site assembly.

Although only this category is evaluated, APTree's modular architecture, planner-agnostic execution model, and extensible DSL are
designed to facilitate adaptation to other prefabrication categories with similar modular, long-horizon planning requirements. The DSL
separates the framework from domain-specific knowledge through an extensible type system: the timber-specific actions and object types
used in this case study are not intrinsic to APTree but represent one domain instantiation of the DSL. Applying APTree to another domain
therefore requires defining the corresponding action, object, and predicate types together with their planning semantics via
APTree's \texttt{DomainTypesDef}, after which the planning and execution framework can be reused. In addition, depending on the
structure of the target planning problems, new execution policies (e.g., decorators such as fair progress) may need to be developed to
coordinate branch execution effectively. Empirically validating such domain transfers, and categorizing planning problems by type to
systematically assign execution policies that ensure step-efficient planning, remain future research directions.

While APTree enabled large-scale automated task planning, identifying the appropriate PDDL-based planner and developing its domain were substantial tasks in our case study. Hence, a potential counterargument to our method is the manual effort required to specify input domains for multiple task planning tools, particularly symbolic task planners, which could undermine the framework's overall benefits. We acknowledge this challenge; however, domain files created for specific types of problems can be saved as libraries and reused for similar problems. Furthermore, current research, such as~\cite{vallati_mypddl_2020}, aims to automate the generation of PDDL domains using AI techniques, including LLMs. 

In addition, the evaluation focuses on the novel planning and validation features
APTree provides rather than on a thorough runtime performance
evaluation. Execution-level metrics such as task completion rate,
recovery success, and multi-agent concurrent planning and execution
require deployment in a real or simulated robotic environment with
real-time task planning and continuous environmental feedback,
including external disturbances that may invalidate achieved goals, and
are therefore beyond the scope of this study. For the same reason,
mid-level tasks are configured to always return \texttt{Success} and
all temporal relations are set to \texttt{Meets}; while the method
supports action nodes with complex execution logic and temporal
planning, evaluating their coordination requires real-time robotic
execution. Similarly, the injected errors used to evaluate design-time
detection are derived from problems encountered in our case study
rather than an exhaustive error taxonomy, and the usability,
learnability, and modeling effort of the DSL are best assessed through
controlled user studies with construction engineers, which were not
conducted here. Future work will extend the evaluation to
execution-based scenarios and human-subject studies, and investigate
how multiple planners can generate plans concurrently and how this
affects the correctness and optimality of the overall plan.

Several aspects of the framework currently require manual configuration. The initial planning problem is manually subdivided into 12 sub-problems, the decorators and services coordinating overall plan execution are hard-coded, and the planner type for each dynamic flow node is assigned by hand. Automating these steps---estimating planning complexity to decompose problems automatically, introducing low-code methods for rapidly defining and systematically comparing alternative execution policies, and selecting planning techniques automatically---constitutes a direction for future research toward a fully self-configuring framework.

Finally, while~\ref{APPformal} establishes the safety of
interleaved partial-order execution, progress guarantees under
adversarially constructed overlapping partial orders have not been
formally established or systematically evaluated. Investigating such
scenarios, together with failure-handling strategies when planning
services themselves fail, represents an important direction for
strengthening the robustness of execution control.

\section{Conclusion} \label{conc} 
APTree enables scalable automated task planning in CRF through a method that applies symbolic planners repeatedly, with a DSL acting as a shared interface between planning and execution, and uses execution models as part of the planning process. In a full-scale prefabrication case study, APTree detected all injected typing and logical errors missed by the baseline, scaled beyond monolithic planning through automatic decomposition, and produced plans matching the expert reference plan. \hl{The accompanying formal analysis guarantees the safety of composed execution and local replanning only under explicitly stated assumptions, including planner soundness, non-interference of generated plans, and atomic mid-level execution.} APTree thus bridges the gap between AI task planning and CRF domain application.

\section*{Acknowledgment}
This research was funded by the Ministry of Science, Research and the Arts Baden-Wuerttemberg through the Artificial Intelligence Software Academy (AISA) and \hl{the InnovationCampus Future Mobility}, and supported by the Deutsche Forschungsgemeinschaft (DFG, German Research Foundation) under Germany’s Excellence Strategy -- EXC 2120/2 -- 390831618.

\hl{We also thank Dr. Daniel Hernández for his guidance in preparing a precise formal analysis, and the anonymous reviewers, whose constructive feedback significantly improved both this paper and the APTree method.}

\section*{Declaration of Generative AI and AI-assisted Technologies in the Writing Process}
  
During the preparation of this work, the authors used Copilot, Gemini, and ChatGPT to proofread and edit the text. After using these tools/services, the authors reviewed and edited the content as needed and take full responsibility for the content of the publication. Additionally, we have used the Cursor AI coding assistant for the implementation of the APTree method.

\bibliographystyle{elsarticle-num} 
  \bibliography{references}

\appendix

\section{Additional Material} \label{APPA}

This appendix provides the models referenced in Section~\ref{eval},
including the intentionally invalid models for APTree
(Listing~\ref{lst:bt_modelerrors}) and BT.CPP
(Listing~\ref{lst:errorxml}), with the injected errors identified by
their error IDs from Table~\ref{tab:error_comparison}. These models
serve to evaluate design-time error detection where applicable to each
framework; errors classified as NA or PD for a framework, as well as logical
errors that manifest only at runtime, are not represented in
the corresponding models.

\begin{itemize}
    \item \textbf{EID1}: The planning service  references a non-existing planner configuration.

    \item \textbf{EID2}: The planning domain file is not specified.

    \item \textbf{EID3}: The FF planner is used with a planning domain containing unsupported language (PDDL~2.1).

    \item \textbf{EID4}: An action node in the BT has child nodes. Only control flow nodes may have children.

    \item \textbf{EID5}: The action sequence in the execution model does not match the symbolic plan.

    \item \textbf{EID6}: The object \texttt{lp24} (\texttt{lowerplate24}) does not exist because only 12 lower plates are defined in the planning setup model.

        \item \textbf{EID7}: The action type \texttt{Initialize} is used but not defined in the execution model.

    \item \textbf{EID8}: The parameter \texttt{rp} is assigned the type \texttt{Rail}, which is not defined for the execution model.

    \item \textbf{EID9}: The \texttt{Travel} action in the execution model expects three parameters, whereas the generated symbolic plan provided four.

    \item \textbf{EID10}: An action in the execution model expects a parameter of type \texttt{robot}, but an argument of type \texttt{human} is provided.

    \item \textbf{EID11}: A required curly bracket is missing in the APTree model, and a required closing tag is missing in the BT.CPP model.

    \item \textbf{EID12}: The order of keywords for the root and behavior tree is swapped for both models.

    \item \textbf{EID13}: Swapping keywords in EID12 also causes an unexpected token error.

    \item \textbf{EID14}: The required \texttt{All} keyword in APTree and \texttt{main\_tree\_to\_execute} in BT.CPP are missing.

\end{itemize}

\begin{nolinenumbers}

\begin{lstlisting}[
    language=MontiCore, 
    basicstyle=\scriptsize\ttfamily, 
    basewidth=0.5em,
    lineskip=-1pt,
    keywordstyle=\color{blue}\bfseries,
    commentstyle=\color{darkgray}\bfseries\itshape,
    morekeywords={BehaviorTree, Root, FlowNode, Nodegraph, Action, Service, NodeGraph, ServicePlanning, Decorator, Problem, Domain, Config}, 
    caption={APTree model containing the injected errors (EIDs) used for evaluation Section~\ref{eval}.}, 
    label={lst:bt_modelerrors},
    moredelim={[is][\color{red}]{//\ EID}{EID}},
    moredelim={[l][\color{red}]{//\ EID}}
]

Root APTree3 { // EID12  // EID13 
    BehaviorTree FlowNode Main {
        // EID14 
        AllFlow
        NodeGraph {
            Action PickupSequence {   // EID4
                ServicePlanning MyPlannerSrv1 Enhsp  Problem:problempickupHL Config: "pt-blind"   // EID1 // EID2
                All
                AllAction
                NodeGraph {

                    Action EquipeML EquipVg1 (r1 vg1 ep1) {   // EID5
                        --[Meets]--> InitVg1;
                    }
                    Action TravelML TravelToEp1 (r1 pr2 ep1 fp1) { // EID9
                        --[Meets]--> EquipVg1;
                    }   
                    Action Initialize InitVg1 (r1 vg1) {   // EID7
                        --[Meets]--> TravelToFp9;
                    }
                    Action TravelML TravelToFp9 (human ep1 fp9) {  // EID10
                        --[Meets]--> PickupLp2; 
                    }
                    Action PickUpML PickupLp2 (lp24 rp r1 vg1)  // EID8 //EID6
                }
            }
        
        FlowNode PlaceSequence {
            ServicePlanning MyPlannerSrv1 FF Domain:PickAndPlaceHL Problem:problemPlace // EID3
                All
                AllAction
                NodeGraph {
                }
        }
        } // EID11 
        }
    }
}

  
\end{lstlisting}

\begin{lstlisting}[
    language=XML,
    basicstyle=\scriptsize\ttfamily,
    basewidth=0.5em,
    lineskip=-1pt,
    escapeinside={|}{|},
    caption={BT.CPP XML model containing the injected errors (EIDs) used for evaluation Section \ref{eval}.},
    label={lst:errorxml}
]

<BehaviorTree ID="MainTree"> <!-- |{\color{red}EID12}| -->
  <root BTCPP_format="4"> <!-- |{\color{red}EID13}| --> <!-- |{\color{red}EID14}| -->
    <Parallel success_count="2" failure_count="1" name="(travelml r1 pr2 ep1):1"> <!-- |{\color{red}EID5}| -->
      <Sequence name="(travelml r1 pr2 ep1 fp9):1"> <!-- |{\color{red}EID9}| -->
        <WaitAtStartReq action="(travelml r1 pr2 ep1):1"/>
        <ApplyAtStartEffect action="(travelml r1 pr2 ep1):1"/>
        <ReactiveSequence name="(travelml r1 pr2 ep1):1">
          <CheckOverAllReq action="(travelml r1 pr2 ep1):1"/>
          <ExecuteAction action="(travelml r1 pr2 ep1):1"/>
        </ReactiveSequence>
        <CheckAtEndReq action="(travelml r1 pr2 ep1):1"/>
        <ApplyAtEndEffect action="(travelml r1 pr2 ep1):1"/>
      </Sequence>
      <Sequence name="(equipeml r1 vg1 ep1):5011">
        <Sequence name="(equipeml r1 vg1 ep1):5011">
          <WaitAction action="(travelml r1 pr2 ep1):1"/>
          <WaitAtStartReq action="(equipeml r1 vg1 ep1):5011"/>
          <ApplyAtStartEffect action="(equipeml r1 vg1 ep1):5011"/>
          <ReactiveSequence name="(equipeml r1 vg1 ep1):5011">
            <CheckOverAllReq action="(equipeml r1 vg1 ep1):5011"/>
            <ExecuteAction action="(equipeml r1 vg1 ep1):5011"/>
          </ReactiveSequence>
          <CheckAtEndReq action="(equipeml r1 vg1 ep1):5011"/>
          <ApplyAtEndEffect action="(equipeml r1 vg1 ep1):5011"/>
        </Sequence>
        <Parallel success_count="2" failure_count="1">
          <Sequence name="(initialize r1 vg1):7021"> <!-- |{\color{red}EID7}| -->
            <Sequence name="(initializeml r1 vg1):7021">
              <WaitAction action="(equipeml r1 vg1 ep1):5011"/>
              <WaitAtStartReq action="(initializeml r1 vg1):7021"/>
              <ApplyAtStartEffect action="(initializeml r1 vg1):7021"/>
              <ReactiveSequence name="(initializeml r1 vg1):7021">
                <CheckOverAllReq action="(initializeml r1 vg1):7021"/>
                <ExecuteAction action="(initializeml r1 vg1):7021"/>
              </ReactiveSequence>
              <CheckAtEndReq action="(initializeml r1 vg1):7021"/>
              <ApplyAtEndEffect action="(initializeml r1 vg1):7021"/>
            </Sequence>
            <Sequence name="(pickupml lp24 rp r1 vg1):12032"> <!-- |{\color{red}EID6}| --> <!-- |{\color{red}EID8}| -->
              <WaitAction action="(initializeml r1 vg1):7021"/>
              <WaitAction action="(travelml r1 ep1 fp9):7022"/>
              <WaitAtStartReq action="(pickupml lp24 rp r1 vg1):12032"/>
              <ApplyAtStartEffect action="(pickupml lp24 rp r1 vg1):12032"/>
              <ReactiveSequence name="(pickupml lp24 rp r1 vg1):12032">
                <CheckOverAllReq action="(pickupml lp24 rp r1 vg1):12032"/>
                <ExecuteAction action="(pickupml lp24 rp r1 vg1):12032"/>
              </ReactiveSequence>
              <CheckAtEndReq action="(pickupml lp24 rp r1 vg1):12032"/>
              <ApplyAtEndEffect action="(pickupml lp24 rp r1 vg1):12032"/>
            </Sequence>
          </Sequence>
          <Sequence name="(travelml human ep1 fp9):7022"> <!-- |{\color{red}EID10}| -->
            <WaitAction action="(equipeml r1 vg1 ep1):5011"/>
            <WaitAtStartReq action="(travelml human ep1 fp9):7022"/>
            <ApplyAtStartEffect action="(travelml human ep1 fp9):7022"/>
            <ExecuteAction name="(travelml human ep1 fp9):7022"> <!-- |{\color{red}EID4}| -->
              <CheckOverAllReq action="(travelml human ep1 fp9):7022"/>
              <ExecuteAction action="(travelml human ep1 fp9):7022"/>
            </ExecuteAction>
            <CheckAtEndReq action="(travelml human ep1 fp9):7022"/>
            <ApplyAtEndEffect action="(travelml human ep1 fp9):7022"/>
          </Sequence>
        </Parallel>
      </Sequence>
    {/Parallel> <!-- |{\color{red}EID11}| -->
  </root>
</BehaviorTree>


\end{lstlisting}

\end{nolinenumbers}

\section{Formal Correctness Analysis} \label{APPformal}

This section provides formal proofs and invariants establishing plan
correctness and composability of APTree under (i) dynamic node
insertion, (ii) replanning, and (iii) partially ordered interleaved execution.
 
\medskip
\textbf{Preliminaries.} A planning domain is a pair
$\mathcal{D} = \langle \mathcal{P}, \mathcal{A}\rangle$
($\mathcal{P}$: fluents, $\mathcal{A}$: actions with $\mathrm{pre}(a)$,
$\mathrm{eff}(a)$)~\cite{Ghallab04}. Each dynamic flow node (\texttt{DFN}) is associated
with a domain of one of two kinds: a \emph{high-level domain}
$\mathcal{D}_{HL}$, used for high-level planning, or a
\emph{mid-level domain} $\mathcal{D}_{ML}$, used for subtrees generated
by decomposition or replanning, where each $\mathcal{D}_{ML}$ refines
its corresponding $\mathcal{D}_{HL}$ into finer-grained actions and
fluents. For any \texttt{DFN}, its associated planning problem is
$\Pi = \langle \mathcal{D}, s_0, g\rangle$, pairing its domain with an
initial state $s_0 \subseteq \mathcal{P}$ and a goal $g$.
High-level planning problems are
instantiated from the user-provided initial problem specifications;
mid-level planning problems are instantiated lazily, taking as their
initial state the world state at the time of refinement or replanning.
Node results at a given tick $\mathrm{result}(n,\mathcal{T}_k)$ follow
our definitions (\texttt{BTNodeResult}) in Section~\ref{bt}.
We write $\mathrm{eff}(a)$ as a pair of an add-list and a delete-list:
an effect $p$ makes fluent $p$ true and an effect $\neg p$ makes it
false. Where an effect set is used as a goal, it denotes the conjunction of its add-list literals and the negations of its delete-list literals. Throughout this appendix, DFNs are assumed to use the
\texttt{All} success criterion (Section~\ref{FN}); under partial-success
criteria (\texttt{Any}, \texttt{Count}, \texttt{Percentage}) a DFN may
report \texttt{Success} without its planning goal holding, and the
results below do not apply.
Let $\mathcal{S}$ denote a symbolic planner, i.e.\ a procedure that,
given a planning problem $\Pi = \langle \mathcal{D}, s_0, g\rangle$,
either returns a plan $\pi$ (a finite sequence of actions from
$\mathcal{A}$) or reports failure if no plan is found.

\begin{definition}[Goal interference]
\label{def:interference}
Two goals $g_i, g_j$ \emph{interfere}, written $g_i\bowtie g_j$, iff
$\exists p\in\mathcal{P}: p\in g_i \land \neg p\in g_j$, or vice versa.
A plan $\pi_j$ \emph{interferes} with a goal $g_i$ iff some action
$a \in \pi_j$ has an effect $\neg p$ with $p \in g_i$, or an effect $p$
with $\neg p \in g_i$.
\end{definition}

\subsection{Correctness under dynamic node insertion}

APTree does not insert nodes at arbitrary positions. Every generated
node is introduced either as part of the initial decomposition of a
planning problem into independent high-level planning subproblems or as
part of the vertical refinement of a high-level action into an ML
subtree. We show that both preserve correctness.

\begin{assumption}[Planner soundness]
\label{assum:planner-sound}
All symbolic planners employed by APTree are assumed to be sound. That
is, for every planning problem $\Pi=\langle\mathcal{D},s_0,g\rangle$,
if $\mathcal{S}(\Pi)$ returns a plan $\pi$, then $\pi$ is applicable
from $s_0$ and reaches a state satisfying $g$.
\end{assumption}

\begin{invariant}[Initial non-interference]
\label{inv:initial-noninterference}
Let $n_1,\ldots,n_k$ denote the initial DFNs that decompose the overall
planning task into high-level planning subproblems with goals
$g_1,\ldots,g_k$. Then, for all $i\neq j$: (1) $\neg(g_i \bowtie g_j)$,
and (2) no plan generated for subproblem $j$ interferes with goal
$g_i$ (Definition~\ref{def:interference}). That is, neither the goals
of the initial high-level subproblems nor the effects of their plans
undo one another's goals.
\end{invariant}

\begin{invariant}[Goal inheritance]
\label{inv:goal-inheritance}
Let $F_i$ be the DFN generated to refine high-level action $a_i$,
rooting ML subtree $T_i$. Then $\mathrm{goal}(F_i)=\mathrm{eff}(a_i)$,
i.e., $F_i$ inherits the effects of the high-level action it
refines.
\end{invariant}

Invariant~\ref{inv:initial-noninterference} constrains only the
high-level plans. The ML subtrees, however,
contain finer-grained actions of $\mathcal{D}_{ML}$ that no
high-level plan mentions, and whose effects
Invariant~\ref{inv:initial-noninterference} therefore does not
constrain. Goal persistence of the high-level subgoals and of the
effects already realized by completed ML subtrees requires the
analogous condition at the mid-level:

\begin{invariant}[ML non-interference]
\label{inv:ml-noninterference}
For every ML subtree $T$ generated during execution, the plan of $T$
does not interfere (Definition~\ref{def:interference}) with (1)~any
subgoal $g_i$ of the initial decomposition, nor (2)~the inherited goal of any distinct ML subtree whose execution completes before that of $T$; i.e., no action of $T$ has an effect that falsifies a literal of a high-level subgoal or of an already-achieved inherited goal. This invariant is a condition on the plans generated at runtime. 
\end{invariant}

\begin{theorem}[Refinement soundness]
\label{thm:refinement}
Let the initial DFNs define high-level planning subproblems with goals
$g_1,\ldots,g_k$ satisfying
Invariant~\ref{inv:initial-noninterference}, each solved by a sound HL
planner (Assumption~\ref{assum:planner-sound}). Let every high-level
action $a_i$ produced by these planners be refined into an ML subtree
$T_i$ according to Invariant~\ref{inv:goal-inheritance}, where each
$T_i$ is generated by a sound ML planner from the world state current
at its refinement, and let the resulting ML subtrees satisfy
Invariant~\ref{inv:ml-noninterference}. Then any execution of the ML
subtrees, each executed atomically (Invariant~\ref{inv:ml-nonpreempt})
and in an order consistent with the high-level plans, is applicable
and reaches a state satisfying the overall high-level goal
$g_1 \land \ldots \land g_k$.
\end{theorem}

\begin{proof}
By Invariant~\ref{inv:goal-inheritance}, each ML subtree $T_i$ inherits
$\mathrm{eff}(a_i)$, the effects of the high-level action it refines,
as its goal. Since each $T_i$ is generated by a sound ML planner
(Assumption~\ref{assum:planner-sound}) from the world state current at
its refinement, its plan is applicable from that state and reaches a
state satisfying $\mathrm{eff}(a_i)$; applied inductively along the
execution order, this makes the entire schedule applicable and each
subtree achieve its inherited goal at the point where it completes.
An inherited goal, once achieved, persists: it could only be falsified
by a subsequently executed action; every such action belongs to an ML
subtree that completes after $T_i$, and by clause~(2) of
Invariant~\ref{inv:ml-noninterference} none of its effects falsifies a
literal of $\mathrm{eff}(a_i)$. Hence the final state satisfies
$\mathrm{eff}(a_i)$ for every refined high-level action $a_i$.
It remains to derive the high-level subgoals. Let $l$ be a literal of
subgoal $g_j$ and $\pi_j$ the high-level plan solving subproblem $j$.
If some action $a_i \in \pi_j$ establishes $l$, then
$l \in \mathrm{eff}(a_i)$, and by the above $l$ holds in the final
state. Otherwise the sound plan $\pi_j$ reaches $g_j$ although none of
its actions establishes $l$, so $l$ already holds in the initial
state; since every change to the world state is effected by an action
of some ML subtree and, by clause~(1) of
Invariant~\ref{inv:ml-noninterference}, no such action falsifies the
literal $l$ of $g_j$, $l$ holds throughout the execution. In both
cases every literal of every $g_j$ holds in the final state, which
therefore satisfies $g_1 \land \ldots \land g_k$, the overall
high-level goal.
\end{proof}

\subsection{Correctness under local replanning}
\label{sec:replan}

APTree preserves correctness under local replanning by ensuring that
replanning is performed only from a quiescent execution state while
preserving the original planning goal. The following definitions and
invariants formalize these conditions and show that replacing an
obsolete subtree with a newly planned subtree preserves correctness.

\begin{invariant}[Goal immutability]
\label{inv:goal-immutable}
Let $F_i$ be the DFN rooting mid-level subtree $T_i$. Once assigned at
$F_i$'s creation (Invariant~\ref{inv:goal-inheritance}),
$\mathrm{goal}(F_i)$ remains unchanged for the lifetime of $T_i$. The initial state of the associated planning problem may
change as execution progresses, requiring replanning from the current
state while preserving the original goal.
\end{invariant}

\begin{definition}[Quiescent execution]
\label{def:quiescent}
The execution is \emph{quiescent} at tick $\mathcal{T}_k$ iff no
action node is executing, i.e.,
\[
\mathrm{Quiescent}(\mathcal{T}_k)
\iff
\forall T_i,\;\forall n\in \mathrm{actions}(T_i):\;
\mathrm{result}(n,\mathcal{T}_k)\neq\texttt{InProgress},
\]
where $\mathrm{actions}(T_i)$ denotes the action nodes of $T_i$. Flow
nodes are not constrained: a DFN whose status has been
reset to \texttt{InProgress} to trigger replanning does not violate
quiescence, since flow-nodes do not change
the world state.
\end{definition}

\begin{definition}[Replan operation]
\label{def:replan}
Let $F_i$ be a DFN with active children $C_i^{\mathrm{active}}$, and
let execution be quiescent (Definition~\ref{def:quiescent}). The
operation $\mathrm{Replan}(F_i)$:
\begin{enumerate}
    \item archives all completed children;
    \item removes all unexecuted active children (and any descendant
          subtrees);
    \item generates a new plan from the current world state while
          preserving $\mathrm{goal}(F_i)$
          (Invariant~\ref{inv:goal-immutable}), and attaches the
          resulting nodes as the new active children.
\end{enumerate}
Archived children are retained for execution history and are never
re-ticked.
\end{definition}

\begin{invariant}[Status excludes archived nodes]
\label{inv:status-exclusion}
At every tick $\mathcal{T}_k$, the node result of a DFN $F_i$ is
determined solely by its active children $C_i^{\mathrm{active}}$.
Archived children $C_i^{\mathrm{archived}}$ never contribute to status
evaluation, although their effects remain reflected in the world state.
\end{invariant}

\begin{theorem}[Correctness under local replanning]
\label{thm:replanning}
Given planner soundness (Assumption~\ref{assum:planner-sound}),
quiescent execution (Definition~\ref{def:quiescent}), goal immutability
(Invariant~\ref{inv:goal-immutable}), and the replan operation
(Definition~\ref{def:replan}), any local replanning event preserves
correctness of the affected DFN. Moreover, the event is \emph{local}:
provided the newly generated plan satisfies
Invariant~\ref{inv:ml-noninterference}, every already-achieved subgoal
of a distinct subproblem remains satisfied.
\end{theorem}

\begin{proof}
Let $\mathrm{Replan}(F_i)$ occur at tick $\mathcal{T}_k$. Since
execution is quiescent, no action node is executing; hence the
current world state $s_k$ contains no partially applied effects and can
serve as the initial state of the new planning problem. By
Invariant~\ref{inv:goal-immutable}, replanning preserves the original
goal $\mathrm{goal}(F_i)$. The replan operation replaces only the
unexecuted active nodes of $F_i$ and does not re-tick archived nodes,
whose effects remain reflected in the world state
(Invariant~\ref{inv:status-exclusion}).
By planner soundness, the newly generated subtree is applicable from
$s_k$ and reaches $\mathrm{goal}(F_i)$. Therefore, replacing the
obsolete unexecuted subtree by the newly planned subtree preserves the
correctness of $F_i$.
For locality, let $g_j$ be a subgoal of a distinct subproblem already
satisfied in $s_k$. The new plan is generated during execution, so
Invariant~\ref{inv:ml-noninterference} applies to it: none of its
actions has an effect that falsifies a literal of a HL $g_j$. Hence $g_j$
remains satisfied throughout the execution of the new subtree, and the
replanning event affects no subgoal outside $F_i$. Since each
replanning event satisfies the same conditions from the then-current
state, any sequence of such events preserves correctness.
\end{proof}

\subsection{Logical composability under partial-order execution}

APTree supports the concurrent ticking of sibling DFNs corresponding to
independent high-level planning subproblems. Although multiple branches
may be ticked concurrently, execution is performed by a single agent,
yielding an interleaved execution. We therefore show that the
scheduling policy preserves the logical correctness of the sound
partial-order plan.

\begin{invariant}[HL plan immutability]
\label{inv:hl-immutable}
Let $\pi_{HL}$ (or equivalently the partial-order plan
$\langle\mathcal{A}',\prec\rangle$) denote the high-level plan
generated at system initialization. Throughout the execution considered
in this work, this plan remains unchanged. Consequently,
$\mathrm{Replan}$ (Definition~\ref{def:replan}) is applied only to DFNs
refining ML subtrees and is never invoked on the HL
plan.
\end{invariant}

\begin{invariant}[ML non-preemption]
\label{inv:ml-nonpreempt}
Once a DFN $F_i$ begins executing the active children of its mid-level
subtree $T_i$, execution continues until $F_i$ reaches either
\texttt{Success} or \texttt{Failure}. Although sibling DFNs may be
ticked concurrently, no sibling DFN executes during this interval;
ML subtrees therefore execute atomically.
\end{invariant}

\begin{invariant}[Precondition-guarded execution]
\label{inv:precondition-guarded}
A HL action in a partial-order plan becomes eligible for
execution only when all of its preconditions are satisfied in the
current world state, regardless of whether it is unordered with respect
to sibling actions.
\end{invariant}

\begin{invariant}[Predecessor completion]
\label{inv:predecessor-completion}
A node with predecessors becomes eligible for execution only after the
ordering relations connecting it to those predecessors (e.g., \texttt{Meets}, \texttt{Precedes} annotating
the edges of the \texttt{NodeGraph}) are satisfied, i.e., only after
all of its $\prec$-predecessors have completed. This is the behavior
implemented by the \texttt{canExecute} operation (Section~\ref{FN}),
which releases a node only if it has not already completed and all
temporal relations to its predecessors are satisfied. A node that has
completed successfully is re-released only by its own decorator (e.g.,
\texttt{Retry}); the executions considered in this analysis contain no
decorators that re-release succeeded HL action nodes.
\end{invariant}

\begin{theorem}[Logical correctness under partial-order execution]
\label{thm:general-composability}
Given planner soundness (Assumption~\ref{assum:planner-sound}),
Invariants~\ref{inv:hl-immutable}, \ref{inv:ml-nonpreempt},
\ref{inv:precondition-guarded},
and~\ref{inv:predecessor-completion}, and ML non-interference
(Invariant~\ref{inv:ml-noninterference}), any execution of sibling
DFNs that completes achieves the goals of the high-level plan.
\end{theorem}

\begin{proof}
The high-level plan consists of one sequential plan per independent
subproblem, each sound by
Assumption~\ref{assum:planner-sound}; the partial order $\prec$ chains
the actions within each subproblem, while actions of distinct
subproblems are mutually unordered. By
Invariant~\ref{inv:predecessor-completion}, an action can execute only
after all of its $\prec$-predecessors have completed and no completed
action is re-released; hence every executed schedule is a valid
linearization of $\langle\mathcal{A}',\prec\rangle$, executing each
subproblem's actions in plan order.
Invariant~\ref{inv:precondition-guarded} additionally ensures that an
action, even when unordered with respect to sibling actions, is
released only when its preconditions hold in the current world state.
By Invariant~\ref{inv:ml-nonpreempt}, the ML subtrees refining the
released actions execute atomically, so the overall execution is an
interleaving of the branches at subtree granularity; by
Invariant~\ref{inv:hl-immutable}, replanning is confined to mid-level
subtrees and never modifies the high-level plan. Correctness of this
interleaved execution then follows as in
Theorem~\ref{thm:refinement}: each ML subtree is generated lazily by a
sound planner from the current world state and achieves its inherited
goal, Theorem~\ref{thm:replanning} covers local replanning within a
branch, and ML non-interference
(Invariant~\ref{inv:ml-noninterference}) guarantees that each subgoal,
once achieved, persists across subsequently executed subtrees of other
branches. Hence the final state of any completed execution satisfies
$g_1 \land \ldots \land g_k$, the overall high-level goal. This is a
safety guarantee; we do not claim that execution always makes
progress, i.e., that some executable action is always available until
all goals are reached, and establishing such progress guarantees is
future work.
\end{proof}

\subsection{Enforcement of the assumptions and invariants} 
\label{sec:enforcement}

\hl{The assumption and invariants above differ in how they are ensured. \textbf{Framework-enforced:} goal inheritance and goal immutability (Invariants}~\ref{inv:goal-inheritance} \hl{and}~\ref{inv:goal-immutable}\hl{) hold by construction, since the planning service always generates the ML problem file with the effects of the refined HL action as its goal; status exclusion (Invariant}~\ref{inv:status-exclusion}\hl{) is implemented by the success-criteria evaluation of DFNs; ML non-preemption (Invariant}~\ref{inv:ml-nonpreempt}\hl{) by the \texttt{Lock} decorator, which is attached automatically to every ML subtree; and precondition-guarded execution and predecessor completion (Invariants}~\ref{inv:precondition-guarded} \hl{and}~\ref{inv:predecessor-completion}\hl{) by the \texttt{onEnter} and \texttt{canExecute} operations. Quiescent replanning (Definition}~\ref{def:quiescent}\hl{) follows from ML non-preemption, as replanning is triggered only after the locked ML subtree has completed or failed. 

\textbf{User configuration, checkable before execution:} the \texttt{All} success criterion, HL plan immutability (Invariant}~\ref{inv:hl-immutable}\hl{), and the absence of decorators that re-release completed HL actions are properties of the input model and the selected execution policies. Initial non-interference (Invariant}~\ref{inv:initial-noninterference}\hl{) is the user's responsibility when decomposing the planning problem; it can be checked statically from the initial problem files (clause~1) and from the HL plans generated before execution starts (clause~2). \textbf{External assumption:} planner soundness (Assumption}~\ref{assum:planner-sound}\hl{) is a property of the selected planners; APTree validates generated plans structurally and by type and checks action preconditions at runtime, but does not verify planner soundness. \textbf{Runtime invariant:} ML non-interference (Invariant}~\ref{inv:ml-noninterference}\hl{) concerns plans generated during execution and therefore cannot be checked statically. APTree detects its violations at runtime, when an achieved HL goal is falsified, and halts execution (error~16 in Table}~\ref{tab:error_comparison}\hl{).}

\section{Replicability and Data Availability}
\label{APPB}

All artifacts required to reproduce the listings, tables, and
results are permanently archived in the University of Stuttgart's DaRUS
repository at \url{https://doi.org/10.18419/DARUS-6522}, where the two
evaluated configurations are archived as \texttt{v1.0} (without replanning)
and \texttt{v2.0} (with replanning).
A development version is maintained at
\url{https://github.com/shshermin/APTree-automated-planning-tree},
where the corresponding branches are:

\texttt{Version-1-withoutreplanning-scaled-up} and 

\texttt{version-2-withreplanning-scaled-up}.

The artifact comprises three main components:
\repofile{APTreeDSL/}, \repofile{APTreeExecutionEngine/}, and
\repofile{Lean/}. Reproduction instructions are provided in their README files.

\textbf{DSL and models.}
The DSL implementation, grammars from Section~\ref{packages}, DSL functions,
and example behavior trees are provided in \repofile{APTreeDSL/}.
The livMatS input and generated models are located at
\repofile{APTreeDSL/src/test/resources/valid/behavior_trees/APTreeLiveMat.bt}
and
\repofile{APTreeDSL/src/test/resources/valid/behavior_trees/APTreeLiveMatFinal.bt},
respectively. Valid and intentionally invalid models used in
Section~\ref{eval} are available under
\repofile{APTreeDSL/src/test/resources/valid/} and
\repofile{APTreeDSL/src/test/resources/invalid/}.
The PlanSys2 tests used for the comparison models are available at
\url{https://github.com/shshermin/Plansys2_Comparison}.

\textbf{Execution and evaluation.}
The execution and logging infrastructure is provided in
\repofile{APTreeExecutionEngine/}. Its source code and entry point are located
at \repofile{APTreeExecutionEngine/src/} and
\repofile{APTreeExecutionEngine/src/Program.cs}. Generated logs and summaries
are written to \repofile{APTreeExecutionEngine/WrittenLogs/}, while the logs
used in the paper are retained in
\repofile{APTreeExecutionEngine/kept logs/}.

\textbf{Formal verification.}
The machine-checked Lean 4 formalization and instructions for reproducing the
proofs are provided in \repofile{Lean/}.

\end{document}